\documentclass[acmtog]{acmart}
\acmSubmissionID{626}

\usepackage{booktabs} 
\usepackage{todonotes}
\setcitestyle{nosort,square} 
\usepackage{footmisc}

 \renewcommand\footnotetextcopyrightpermission[1]{} 

\makeatletter
\let\@authorsaddresses\@empty
\makeatother

\author{Noam Aigerman$^*$}
\affiliation{%
  \institution{University of Montreal}
  \country{Canada}}

\author{Thibault Groueix$^*$}
\affiliation{%
  \institution{Adobe Research}
  \country{USA}}

\usepackage[ruled]{algorithm2e} 

\usepackage{graphicx}
\usepackage{subcaption}
\usepackage{caption}
\usepackage{enumitem}
\usepackage{wrapfig}
\usepackage{float}

\usepackage[font={small}]{caption}
\SetAlFnt{\small}
\SetAlCapFnt{\small}
\SetAlCapNameFnt{\small}
\SetAlCapHSkip{0pt}

\usepackage{amsthm}
\usepackage{xspace}
\usepackage{amsmath}

\begin{document}

\newcommand{\theprompt}[1]{{''#1``}}
\newcommand{\R}{\mathbb{R}}
\newcommand{\Q}{\mathbb{Q}}
\newcommand{\vc}[1]{\mathbf{#1}}
\newcommand{\Cres}{\mathcal{X}}
\newcommand{\transpose}{{\mathsf T}}
\newcommand{\sfT}{{\mathsf T}}
\newcommand{\dd}[2]{\frac{\partial#1}{\partial#2}}
\newcommand{\dn}[1]{\nabla #1 \cdot\n}
\newcommand{\dt}[1]{\nabla #1 \cdot\t}
\newcommand{\gT}[1]{\nabla #1^\transpose}
\newcommand{\norm}[1]{\left\| #1\right\|}
\newcommand{\abs}[1]{\left\vert#1\right\vert}
\newcommand{\babs}[1]{\Big \vert#1 \Big \vert}

\newcommand{\set}[1]{\left\{#1\right\}}
\newcommand{\parr}[1]{\left (#1\right )}
\newcommand{\brac}[1]{\left [#1\right ]}
\newcommand{\ip}[1]{\left \langle #1 \right \rangle }

\newcommand{\wallpaper}[1]{{\tiny #1}}
\newcommand{\orbI}{\wallpaper{442}}
\newcommand{\orbII}{\wallpaper{333}}
\newcommand{\orbIII}{\wallpaper{632}}
\newcommand{\orbIV}{\wallpaper{2222}}

\newcommand{\torus}{\wallpaper{O}}
\newcommand{\mob}{\wallpaper{*x}}
\newcommand{\cylinder}{\wallpaper{**}}
\newcommand{\klein}{\wallpaper{xx}}
\newcommand{\projective}{\wallpaper{22x}}

\newcommand{\orbIhyb}{\wallpaper{4*2}}
\newcommand{\orbIIhyb}{\wallpaper{3*3}}
\newcommand{\orbRhyb}{\wallpaper{2*22}}
\newcommand{\orbIVhyb}{\wallpaper{22*}}

\newcommand{\reflectI}{\wallpaper{*442}}
\newcommand{\reflectII}{\wallpaper{*333}}
\newcommand{\reflectIII}{\wallpaper{*632}}
\newcommand{\reflectIV}{\wallpaper{*2222}}


\newcommand{\change}[1]{#1}
\newtheorem{thm}{Theorem}
\newcommand{\noam}[1]{}
\newcommand{\thibault}[1]{}
\newcommand{\OTE}{OTE}
\newcommand{\grp}{\mathcal{G}}
\newcommand{\action}{g}
\newcommand{\tile}{\mathcal{T}}
\newcommand{\V}{\mathbf{V}}
\newcommand{\T}{\mathbf{T}}
\newcommand{\vertex}{\mathbf{v}}
\newcommand{\mesh}{\textbf{M}}
\newcommand{\bpairs}{\mathcal{P}^\partial}
\newcommand{\image}{\mathcal{I}}
\newcommand{\render}{\mathcal{R}}
\newcommand{\x}[0]{\mathbf{x}}
\newcommand*{\eg}{e.g.\@\xspace}
\newcommand*{\ie}{i.e.\@\xspace}

\newenvironment{absolutelynopagebreak}
  {\par\nobreak\vfil\penalty0\vfilneg
   \vtop\bgroup}
  {\par\xdef\tpd{\the\prevdepth}\egroup
   \prevdepth=\tpd}

\title{Generative Escher Meshes}



\begin{abstract}
This paper proposes a fully-automatic, text-guided generative method for producing perfectly-repeating, periodic, tile-able 2D imagery, such as the one seen on floors, mosaics, ceramics, and the work of M.C. Escher. In contrast to square texture images that are seamless when tiled, our method generates non-square tilings which comprise solely of repeating copies of the same object. It achieves this by optimizing both geometry and texture of a 2D mesh, yielding a non-square tile in the shape and appearance of the desired object, with close to no additional background details, that can tile the plane without gaps nor overlaps. We enable optimization of the tile's shape by an unconstrained, differentiable parameterization of the space of all valid tileable meshes for given \change{boundary conditions} stemming from a symmetry group. Namely, we construct a differentiable family of linear systems derived from a 2D mesh-mapping technique - Orbifold Tutte Embedding - by considering the mesh's Laplacian matrix as differentiable parameters. We prove that the solution space of these linear systems is exactly all possible valid tiling configurations, thereby providing an end-to-end differentiable representation for the entire space of valid tiles. We render the textured mesh via a differentiable renderer,  and leverage a pre-trained image diffusion model to induce a loss on the resulting image, updating the mesh's parameters so as to make its appearance match the text prompt. We show our method is able to produce plausible, appealing results, with non-trivial tiles, for a variety of different periodic tiling patterns. 
\vspace{-20pt}
\end{abstract}

%
%




\begin{teaserfigure}
    \centering
    \captionsetup{justification=centering}
    \begin{subfigure}[b]{0.245\linewidth}
        \includegraphics[width=\linewidth]{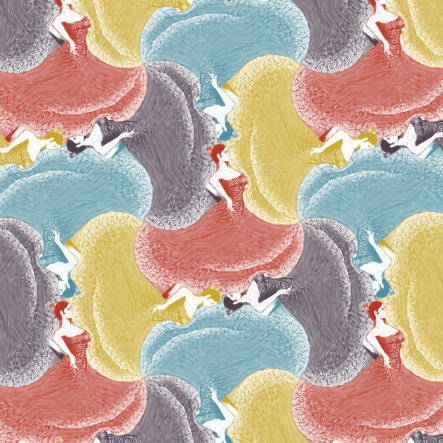}
        \caption{\theprompt{A flamenco dancer} \orbI}
        \end{subfigure}
    \begin{subfigure}[b]{0.245\linewidth}
        \includegraphics[width=\linewidth]{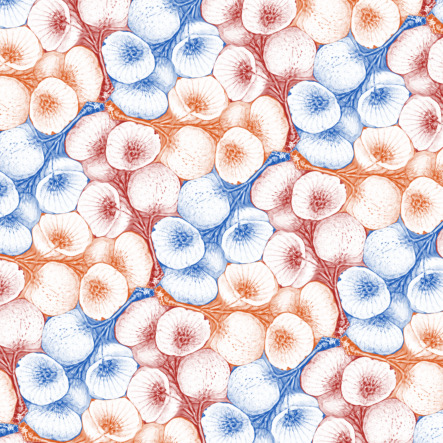}
        \caption{\theprompt{Poppies} \orbII}
        \label{fig:overleaf1}
    \end{subfigure}
    \hfill
    \begin{subfigure}[b]{0.245\linewidth}
        \includegraphics[width=\linewidth]{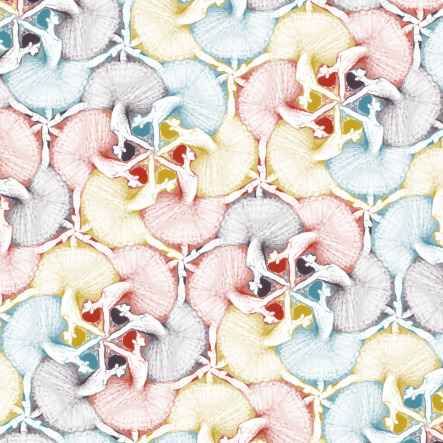}
        \caption{\theprompt{A ballet dancer} \orbIII}
        \label{fig:overleaf2}
    \end{subfigure}
        \hfill
    \begin{subfigure}[b]{0.245\linewidth}
        \includegraphics[width=\linewidth]{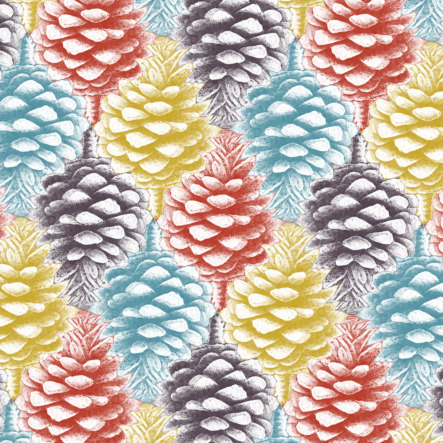}
        \caption{\theprompt{A pinecone} \orbIV}
        \end{subfigure}
    \\
    \begin{subfigure}[b]{0.245\linewidth}
        \includegraphics[width=\linewidth]{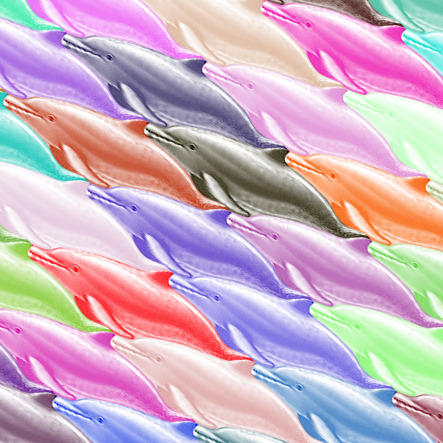}
        \caption{\theprompt{A dolphin} \torus}
        \end{subfigure}
    \hfill
    \begin{subfigure}[b]{0.245\linewidth}
        \includegraphics[width=\linewidth]{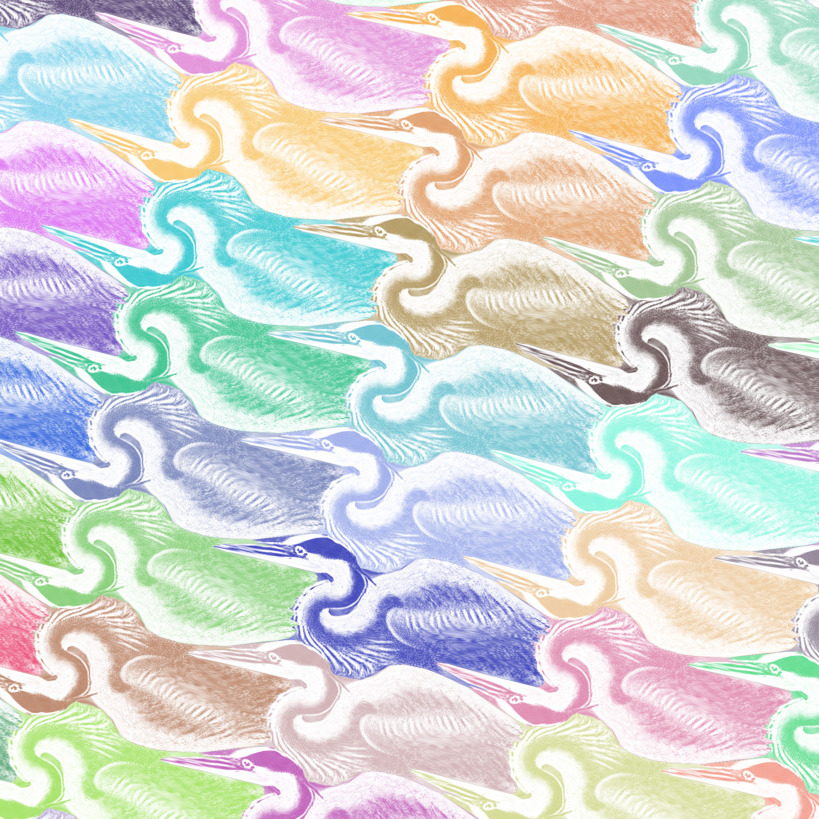}
        \caption{\theprompt{A heron} \klein}
        \label{fig:overleaf2}
    \end{subfigure}
        \hfill
    \begin{subfigure}[b]{0.245\linewidth}
        \includegraphics[width=\linewidth]{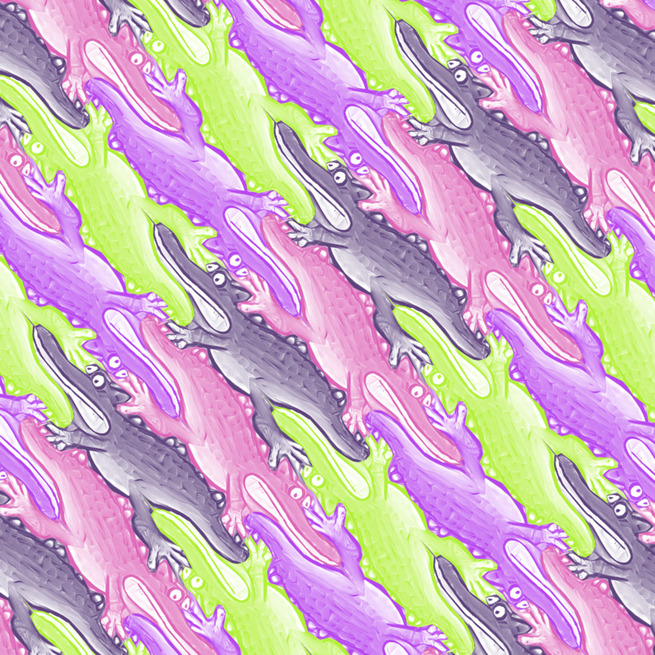}
        \caption{\theprompt{An alligator} \projective}
    \end{subfigure}
    \begin{subfigure}[b]{0.245\linewidth}
        \includegraphics[width=\linewidth]{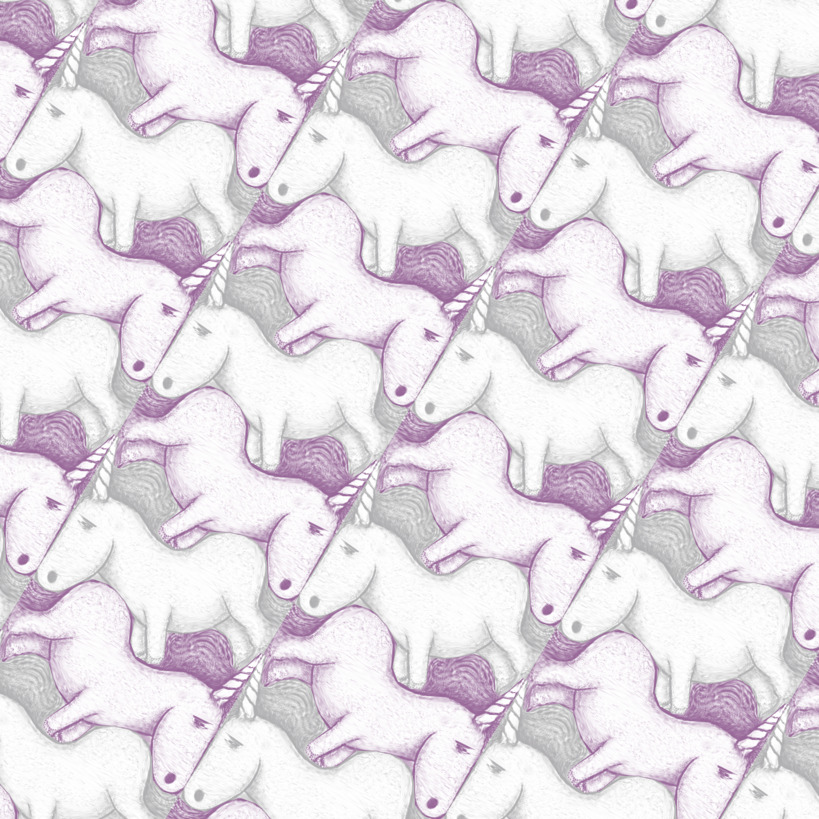}
        \caption{\theprompt{A unicorn} \mob}
        \label{fig:overleaf2}
    \end{subfigure}
    \captionsetup{justification=justified}
            \vspace*{-2mm}
   \caption{\textbf{Tilings automatically
    generated by our method.} Our method automatically generates a textured non-square 2D mesh that tiles the plane without any overlaps nor gaps between the tiles, matching a given a text prompt (written beneath each result) and a chosen symmetry group (denoted by a string to the right of the prompt, using \textit{orbifold notation}~\cite{conway2008symmetries}).  } 
    \label{fig:teaser}
\end{teaserfigure}

\maketitle
\def\thefootnote{*}\footnotetext{Equal contribution}\def\thefootnote{\arabic{footnote}}

\section{Introduction}

This paper proposes a data-driven method for generating textured 2D shapes which can tile the 2D plane. Symmetry and repetition are one of the key characteristics of humans' aesthetic behavior, serving as a cornerstone of design in most cultures around the world since the dawn of humanity. Namely, the ability to automatically produce such patterns has applications in, e.g., fabric design, architecture, interior design, and illustration.

Producing aesthetic, perfectly-repeating periodic imagery is a demanding artistic task which lies at the intersection of human perception and euclidean geometry: geometrically, the pattern must adhere to the stringent designated layout, connect with itself seamlessly, and repeat continuously ad-infinitum; perceptually, the imagery must be a plausible representation of the desired object, as well as hold  visual appeal. Indeed, the beauty of \change{artist M.C. Escher's tilings~\cite{schattschneider2004escher,escher2000mc}} lies in his ability to work within these stringent geometric constraints and produce very complex, realistic objects which still align with themselves perfectly, to yield an infinite tiling.

\begin{figure*}[t]
    \centering
    
        \begin{subfigure}[t]{0.196\linewidth}
        \includegraphics[width=\linewidth]{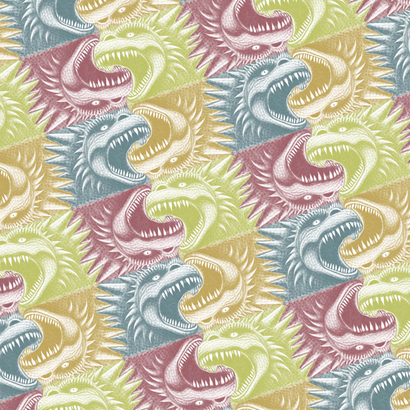}
        \caption{\theprompt{Godzilla} \orbIVhyb}
        \end{subfigure}
        \begin{subfigure}[t]{0.196\linewidth}
        \includegraphics[width=\linewidth]{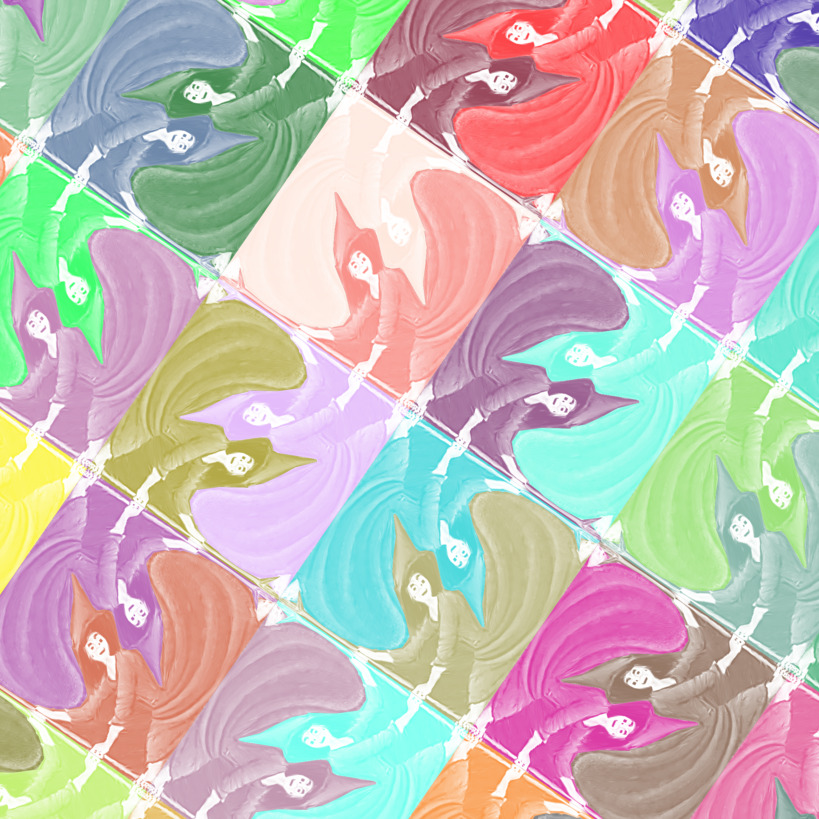}
        \caption{\theprompt{A witch} \orbRhyb}
        \end{subfigure}
        \begin{subfigure}[t]{0.196\linewidth}
        \includegraphics[width=\linewidth]{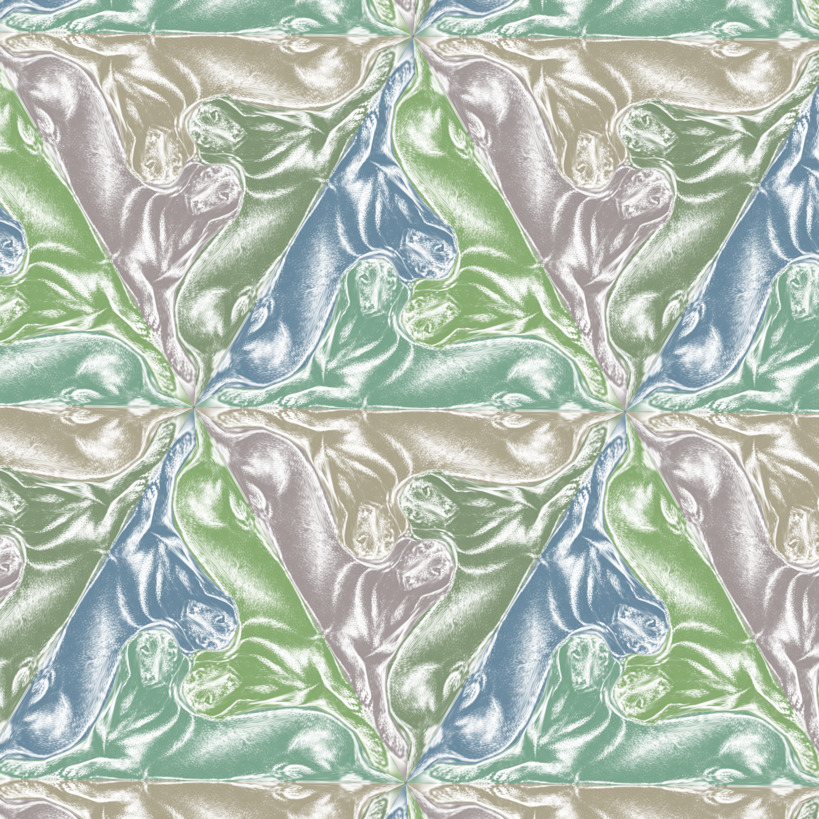}
        \caption{\theprompt{A Dachshund} \orbIIhyb}
        \end{subfigure}
        \begin{subfigure}[t]{0.196\linewidth}
        \includegraphics[width=\linewidth]{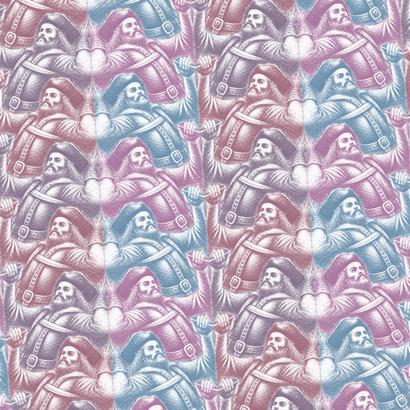}
        \caption{\theprompt{A pirate} \cylinder}
        \end{subfigure}
        \begin{subfigure}[t]{0.196\linewidth}
        \includegraphics[width=\linewidth]{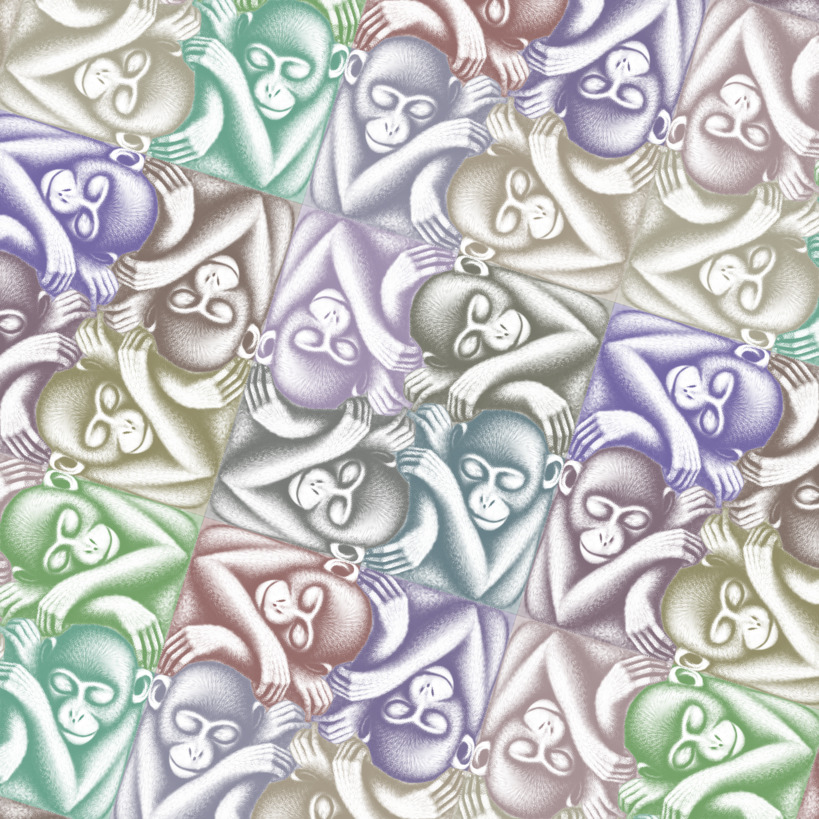}
        \caption{\theprompt{Sleeping monkey} \orbIhyb}
        \end{subfigure}
        \caption{The other  ``interesting'' wallpaper groups left out of Figure~\ref{fig:teaser}.  These five groups have reflectional symmetry which forces the corresponding part of the boundary to remain a straight line, leading to a more-restricted domain and less elaborate tile shapes. The remaining four have a completely fixed boundary.}
        \label{fig:other_five}
\end{figure*}
The goal of this work is thus to provide a text-guided, fully automatic generative approach for aesthetic and plausible periodic imagery. Namely, given a desired textual description of the object and the desired symmetry pattern of the 2D plane, our method produces a textured 2D mesh which: visually resembles the textual description; adheres to the geometric rules that ensure it can be used to tile the plane; and, contains only the foreground object, so as to produce a tiling effect similar to Escher's works. Figure~\ref{fig:teaser} and Figure~\ref{fig:other_five} show several results produced by our method for different prompts. Throughout the paper, we use a unique string to identify each symmetry group via \emph{orbifold notation}~\cite{conway2008symmetries}, which we write next to the text prompt. Refer to Figure 1 of the supplementary material for the illustrated wallpaper groups.


 This goal faces a fundamental geometric challenge: as the mesh's shape is modified during the text-guided optimization, it must maintain its tile-ability. No two copies of the mesh can overlap one another, and the mesh itself must be overlap-free (i.e., the boundary doesn't self-overlap and no triangles of the mesh are inverted), otherwise it does not well-define a tileable bounded shape that can be rendered or fabricated. This requirement is difficult to satisfy, since the space of valid tiles is highly non-convex and does not directly lend itself to optimization using gradient-based techniques involving neural networks, such as the one we aim to use.

We overcome this challenge through our key technical contribution: a \emph{differentiable tile representation}, which casts the highly non-convex, constrained space of valid mesh configurations that are tileable, into an unconstrained, differentiable representation. 
We achieve this by building on top of the theory of Orbifold Tutte Embeddings~\cite{aigerman2015orbifold}. That work provides a method to UV-map a triangular mesh into a 2D tile, however it does not provide a way to control the specific resulting shape of the tile, which is crucial for our task of modifying the tile towards a specific visual appearance. Our main observation is that we can construct a\textit{ differentiable family} of linear systems,  by considering the mesh's Laplacian's edge weights as differentiable, optimizeable parameters. This yields a differentiable representation which enables us to control the resulting shape of the tile - we provide a formal proof this representation enables us to achieve all possible valid tiles, and only them.


With this differentiable tile representation at hand, we can optimize the tile's shape via standard  unconstrained gradient-based optimization. We texture the mesh and render it through a differentiable renderer. The resulting render is fed to Score Distillation Sampling~\cite{dreamfusion}, which uses a trained image diffusion model to provide a loss for optimizing an image w.r.t. a given textual description. Both texture and shape of the mesh are optimized simultaneously, crucial in order to attain tileable shapes containing little to no background. 

Our experiments show that our method is able to generate tileable shapes for all the different wallpapers groups (which include all possible euclidean tilings exhibited in M.C. Escher's work), producing visually appealing and plausible tiles which contain almost solely the desired object, and yield significantly more impressive tiling results compared with a pure image-based generative approach.

To summarize, our contributions are:
\begin{itemize}[leftmargin=15pt]
    \item We provide the first method for fully-automatic text-guided generation of non-square tilings containing only the foreground object, applicable to any of the euclidean symmetry groups.
    \item As a technical contribution, we propose a new differentiable formulation built on top of orbifold Tutte embeddings~\cite{aigerman2015orbifold}, by treating mesh Laplacians as optimizeable parameters, which enables controlling the mesh's 2D embedding while ensuring it remains tileable and overlap-free. 
    \item We provide a simple proof that this differentiable formulation covers exactly the space of all possible overlap-free 2D mesh embeddings that tile the plane w.r.t. \change{boundary conditions corresponding to the} given symmetry group.
\end{itemize}
\textit{See \url{ https://github.com/ThibaultGROUEIX/GenerativeEscherMeshes} for our code and implementation.}



\begin{figure*}
    \centering
    \includegraphics[width=\textwidth]{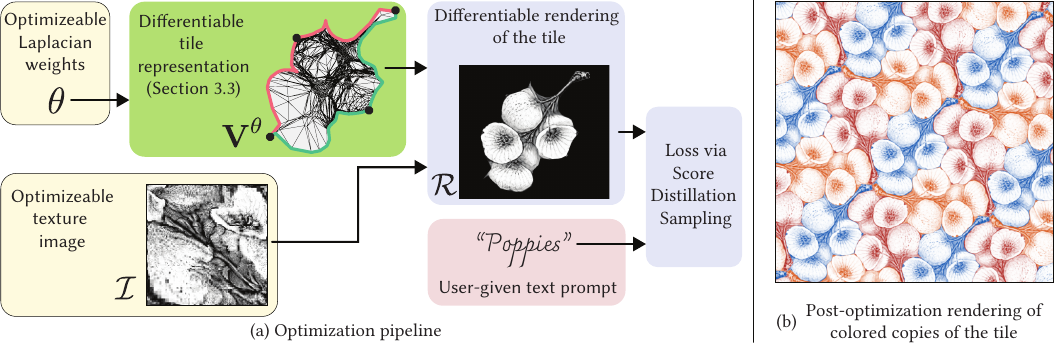}
    \caption{\textbf{Overview of our method.} (a) \textit{the optimization pipeline}: the parameters $\theta$ control our novel differentiable tile representation (Section~\ref{ss:diff_tile}), which outputs the vertices $\V^\theta$ of the mesh, guaranteed to represent a valid tile, with no self-overlaps, that can tile the plane perfectly. The mesh is textured using the texture image $\image$ and rendered via a differentiable renderer, producing the final render $\render$ of the textured tile. The render is fed into Score Distillation Sampling~\cite{dreamfusion} along with the text prompt. From there, gradients are back-propagated in reverse to the optimizable parameters - $\theta$ and $\image$. \\(b) \textit{production of the final image of the tiling}: After optimization is finished, we take copies of the resulting tile, color them with different colors, and tile the plane with them, to produce the final images shown in this paper. }
    \label{fig:overview}
\end{figure*}
\section{Related work}

\paragraph{2D Tilings.}
Planar symmetries give rise to a rich set of rules and possible tilings, see~\cite{conway2008symmetries} for a comprehensive discussion. Tilings come in different variants, e.g., alternating between two different tile shapes, or aperiodic tilings which use a fixed set of tiles but connect them in different ways~\cite{penrose1974role,smith2023aperiodic}. 
We propose a generative method for the simple case of tilings formed by quotients of symmetry groups, i.e., produced by repeatedly applying a set of rigid transformations to one tile - this case is extensive enough to include all of Escher's euclidean tilings.

\paragraph{Tilings in graphics.} As a computational problem combining pleasing aesthetics and geometry, tilings have long been of interest to the graphics community. Escherization~\cite{kaplan2000escherization} is a major source of inspiration for our work. It fits a tileable polygon to a given shape through an optimization process. Similarly to us, the authors aim to formulate an unconstrained parameterization of different tile shapes, via angles and lengths. However, these parameters only define a ``tiling polygon'', which is comprised of a fixed, small set of ``tiling vertices'' (e.g., a square for most examples in Figure 1 in supplementary). The edges between these can be subdivided with ``shape vertices'', which are required to satisfy a similar type of boundary conditions to the ones used in our work, and serve to add details to the tile's shape. However, the rules devised by the authors do not directly ensure these additional ``shape vertices'' (necessary for achieving more elaborate shapes such as the ballerina in Figure~\ref{fig:teaser})  form a boundary which is free of self-overlaps. As discussed before, ensuring an overlap-free boundary is critical to our application, as the image-based guidance of SDS~\cite{dreamfusion} yields strong deformations during optimization which immediately lead to self-overlaps if not reined. Additionally, the parameterization of~\cite{kaplan2000escherization} is solely for a tileable \emph{boundary} polygonal curve, and does not consider the parameterization of a tileable triangular mesh. Followups present better fitting methods~\cite{escherization_eigenvalue} and extend this method with a deformable mesh~\cite{escherization_deformation}, however as they apply a similar parameterization, they cannot prohibit self-overlaps as well. In summary, these methods are geared towards a different goal of turning an existing shape into a tile, and not towards optimization via perceptual losses.

Other works reconstruct shapes from existing tiles such as mosaics~\cite{hausner2001simulating,smith2005animosaics} and fabricable pieces~\cite{chen2017fabricable},  focus on designing tiles for tiling a surface~\cite{eigensatz2010paneling,fu2010k,singh2010triangle}, or deforming templates to match layouts~\cite{peng2014computing} and creating checkerboard patterns~\cite{peng2019checkerboard}. Recently, TilinGNN~\cite{TilinGNN} employ a GNN to design patterns for connecting given tiles to form a specific shape. None of these works tackles our task: optimizing a tile for perceptual appearance, while ensuring it tiles the plane.

\paragraph{Image generation of tileable texture.}
Texture synthesis has been a major focus of computer graphics and vision in the past twenty years~\cite{efros2001image,efros1999texture,kwatra2003graphcut,dong2007optimized,kwatra2005texture, portilla2000parametric}, see~\cite{akl2018survey, raad2018survey} for surveys. \cite{moritz2017texture} use a PatchMatch variant to preserve the stationarity of a synthesized image. 
 ~\cite{Rodriguez-Pardo2019AutomaticPatterns} optimally crop an image for tiling. 
\cite{li2020inverse} use GraphCuts to improve the tileability of the image's borders. GANs can generate tileable square images by adding periodic constraints~\cite{bergmann2017learning}, tiling the latent space~\cite{rodriguezpardo2022SeamlessGAN} or modifying the padding pattern~\cite{tilegen, photomat}. All these works focus on producing seamless textures over a standard rectangular image such that copies of the image connect seamlessly. This cannot ensure only a single repeating object without any background nor account for specific symmetries such as the 6-way symmetry of the ballet dancers in Figure~\ref{fig:teaser}. In contrast, we aim to find \emph{one} non-square geometric shape that can be repeated infinitely without any gaps nor overlaps to map the plane.

\paragraph{Text-guided generation of images and graphics.}
CLIP~\cite{clip} provides a coupling between images and text in the same latent space  which enables text-guided optimization of appearance of pixels~\cite{clipstyler}, Bezier curves~\cite{clipasso,clipdraw}, NeRFs~\cite{dreamfields}, and 3D triangle meshes~\cite{clipmesh,text2mesh,textDeformer}.  This was followed by diffusion models' breakthrough in text-guided generative approaches for images~\cite{origdiffusion2, origdiffusion1,imagen,dalle,stablediffusion}. 
Central to this work, DreamFusion~\cite{dreamfusion}   introduces Score Distillation Sampling (SDS), a technique to turn a trained diffusion model into a prior which can be used as a loss in a traditional optimization setting, which they show outperforms  CLIP guidance. \cite{sjc} introduce a concurrent variant of SDS, while other works report improvements by applying SDS to generative vector art~\cite{vectorart} and 3D generation~\cite{magic3d,fantasia3d, latentnerf, realfusion}. Similarly, we also use SDS, for the completely-novel task of generating tilings.

\paragraph{Overlap-free mesh embeddings}
Our method hinges on our ability to ensure the mesh is overlap-free while enforcing periodic boundary constraints. \cite{aigerman2015orbifold} provide an extension of Tutte's embedding~\cite{tutte1963draw,floater2003one,gortler2006discrete} for tilings,  which we turn modifiable and differentiable in order to incorporate it into the gradient-based optimization of a perceptual loss. Many other methods exist for ensuring overlap-free mesh embeddings, however they cannot be used in our setting: energy-based methods~\cite{schuller2013locally,smith2015bijective,rabinovich2017SLIM} provide a barrier term which prevents triangles from inverting and the boundary from self-overlapping.  While a barrier term can be incorporated as a regularizer, it exhibits extremely-large amplitudes and produces large gradients which  compete with SDS's own gradients, and additionally requires performing a line search as well as collision detection on the boundary, significantly increasing the run time. \textit{Convexification} methods such as~\cite{lipman2012bounded,kovalsky2014controlling} restrict the space of possible embeddings and require slow, more advanced optimization techniques. Other methods such as using auxiliary structures~\cite{muller2015air,jiang2017simplicial}, or iterative constructive processes~\cite{progressiveembeddings,simplexassembly} are not applicable in a general gradient-based optimization.  

\section{Method}
\label{sec:method}
\begin{figure}[t]
    \centering
    \captionsetup{justification=centering}
        
    \begin{subfigure}[t]{0.49\linewidth}
        \includegraphics[width=\linewidth]{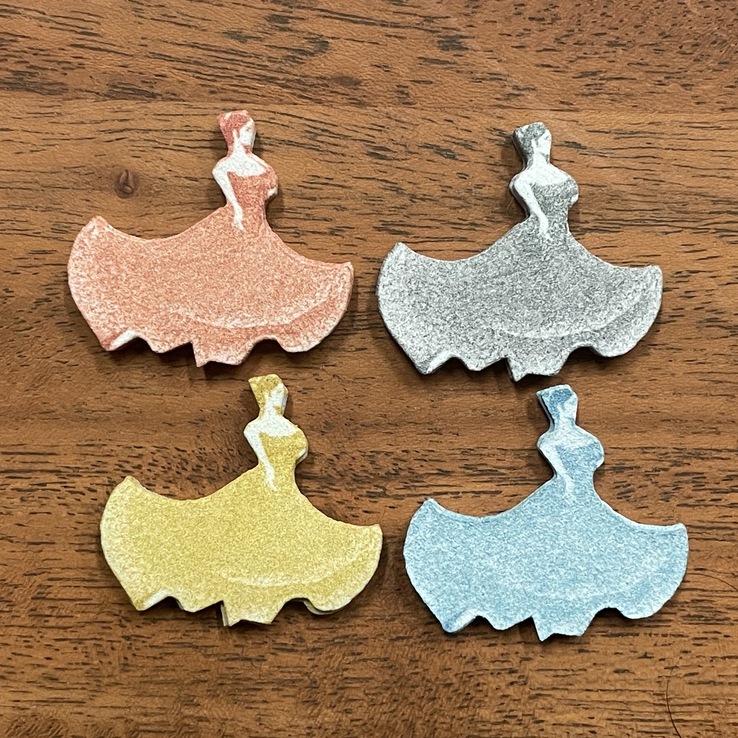}
        \caption{Fabricated tiles}
        \label{fig:print2}
    \end{subfigure}
    \hfill
    \begin{subfigure}[t]{0.49\linewidth}
        \includegraphics[width=\linewidth]{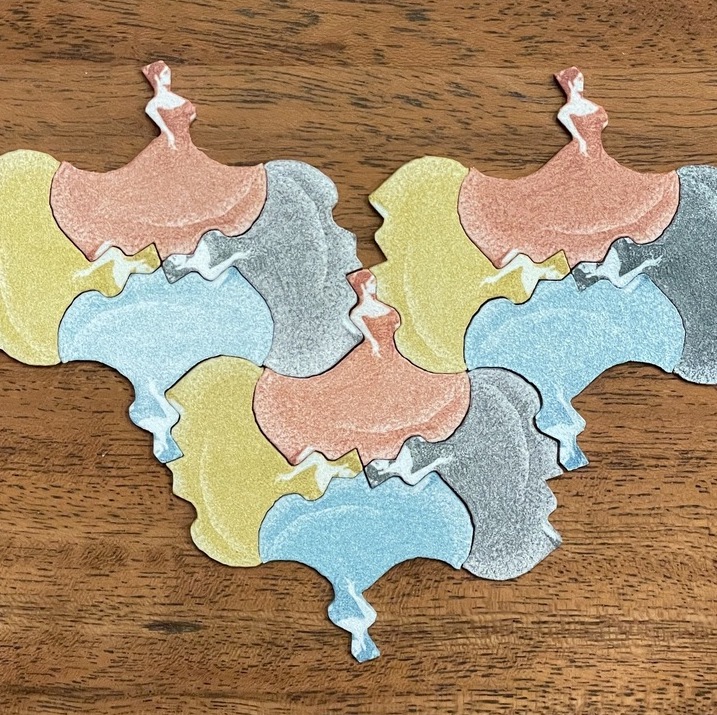}
        \caption{Plane Tesselation}
        \label{fig:print3}
    \end{subfigure}
    \vspace{-10pt}
    \captionsetup{justification=justified}
    \caption{\textbf{Fabrication.}  The output of our method can be readily used to 3D print physical tiles.}
    \label{fig:farbication}
\end{figure}

We assume to be given a user-input of a text prompt describing the object that should make up the tiling. We are also given the user's choice of \change{symmetry pattern} (Subsection \ref{ss:mesh_cond}) which describes how the generated object will repeat itself in order to tile the plane. Our \emph{goal} is to produce an infinite, perfectly-repeating tiling that (i) visually matches the text;  (ii) comprises solely of the desired object with little to no background; and (iii) respects the desired symmetry pattern. 

A naive approach of directly using a generative technique (e.g., a GAN or a diffusion process) to produce a square image that matches the text-prompt will trivially satisfy requirement (i); however, it cannot ensure the tiling will comprise solely of the desired object (thus does not satisfy (ii)), and cannot accommodate for all possible types of symmetry groups, e.g., ones with 6-way  symmetry such as the ballet dancers in Figure~\ref{fig:teaser} (thus does not satisfy (iii)). 

 Instead, to satisfy all three requirements, our core idea is to define the tile $\tile$ as a 2D triangular mesh with vertices $\V$ and triangles $\T$. Each vertex of the mesh has fixed UV coordinates which enable us to texture the mesh with respect to a texture-image $\image$. 
 We can then jointly optimize both the tile's geometry via its vertices, and its colors via the texture (see Figure~\ref{fig:overview}) - this joint optimization encourages the resulting tile to contain only the desired foreground imagery. We optimize these parameters with respect to an off-the-shelf energy defined in Score Distillation Loss (SDS)~\cite{dreamfusion}, quantifying how much the rendering of the textured, deformed mesh visually-matches the input text prompt.

 The core \textit{challenge} lies in the geometric part of the task - optimizing the shape of the mesh while ensuring it is still a valid tile, i.e., can cover the entire plane without gaps nor overlaps. To describe our novel differentiable layer that produces tileable meshes, we first review basic theory of planar symmetries (Subsection \ref{ssec:theory}) and devise necessary and sufficient conditions for tileability (Subsection~\ref{ss:mesh_cond}).
\subsection{Preliminaries: planar symmetries} 
\label{ssec:theory}
The following subsection discusses the basic theoretical underpinning of this work. We refer the reader to \textit{The Symmetry of Things}~\cite{conway2008symmetries}  \change{and \textit{Tilings and Patterns}~\cite{grunbaum1987tilings}} for a comprehensive discussion. 

We are concerned with infinite tilings of the plane, composed of the basic tile $\tile\subset\R^2$ and infinitely-many copies of it. Each copy is shifted by some euclidean isometry $g$, placing the copy at $g\parr{\tile}$.  The copies of the tile must completely cover $\R^2$, without any gaps nor overlaps between different copies. 

In mathematical terms, the different patterns - possible ways to generate tilings by placing copies of a tile - \change{stem from} the symmetry groups of the plane, aptly called the \emph{wallpaper groups}. There are exactly $17$ such groups (see supplementary material), each one comprising of an infinite discrete set of isometric  affine transformations - rotation, translation, reflection, and a reflection composed with a translation (called a ``glide reflection'').  \change{Thus, a tiling always corresponds to one of the 17 wallpaper groups} $\grp$ \change{(although different tiling patterns can be created from the same group)}. 

Any point  $p\in\tile$ in the interior of the tile appears exactly once in each copy $g\parr{\tile}$ of the tile, as $g\parr{p}$.
All copies of the point $p$ are an \emph{equivalence class}, $\brac{p}\triangleq\set{g\parr{p}|g\in\grp}$, and the tile itself is thus a representation of the \emph{quotient} of $\R^2$ under $\grp$, $\tile\equiv\R^2/\grp$, where we select one representative point $p\in\R^2$ for each equivalence class $\brac{p}$. The different choices of representative points for the equivalence classes lead to all the possible tile shapes. We mainly restrict our attention to simply-connected tiles that are topological disks, however our method can support other topologies, see Figure~\ref{fig:dual}.

\subsection{Conditions for valid tiles} 
\label{ss:mesh_cond}
Given the symmetry group $\grp$, what are the necessary and sufficient conditions on the placement of $\V$ that ensure that the mesh $\tile=\parr{\V,\T}$ can tile the plane using $\grp$?

\begin{wrapfigure}[20]{r}{0.24\textwidth}
\vspace{-10pt}
\includegraphics[width=0.24\textwidth]{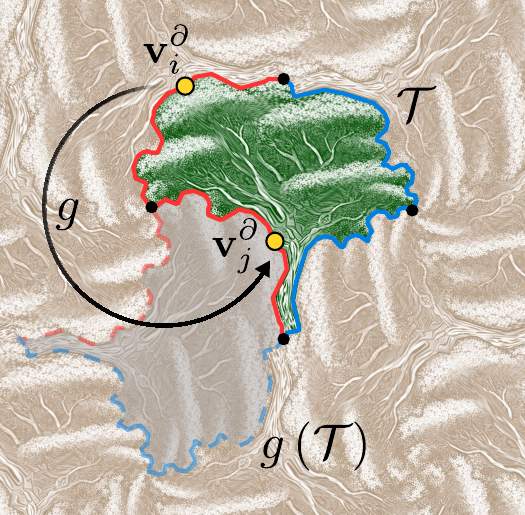}
    \caption{The way in which copies of the tile connect to one another defines  correspondences between different parts of the boundary, and as a result boundary conditions.}
    \label{fig:tile_structure}
\end{wrapfigure}
First, since we require the tiling to be overlap-free, the mesh itself should be overlap-free as well. Now, consider a tileable mesh $\tile$ (green tree in inset), and one of its copies $\action\parr{\tile}$ which is directly adjacent to it (dark-brown tree), \ie, their boundaries align. This alignment implies that when a boundary vertex $\vertex^\partial_i$  of $\tile$ is copied to a vertex $\action\parr{\vertex^\partial_i}$ of the tile $\action\parr{\tile}$, the copy aligns with another boundary vertex $\vertex^\partial_j$ of the original tile:

\begin{equation}
    \label{eq:boundary_constraints}
    \action\parr{\vertex_i^\partial} = \vertex_j^\partial.
\end{equation}
Thus every boundary vertex $\vertex^\partial_i$ has a unique corresponding boundary vertex $\vertex^\partial_j$ which aligns with it in some copy, through an isometry $\action$ (a specific one for each vertex of the mesh).
These correspondences are assigned for a given mesh by choosing a known, canonical vertex-placement which is tileable (e.g., a square), and deducing correspondence between vertices for that placement.

This yields the necessary and sufficient conditions for a \emph{valid} tile: 
\begin{enumerate}[noitemsep,partopsep=0pt,topsep=0pt,parsep=0pt]
    \item All boundary vertices satisfy the boundary conditions 
    \label{item:boundary}
    Eq.~\eqref{eq:boundary_constraints}.
    \item the mesh does not self-overlap. \label{item:overlap_cond}
\end{enumerate}
%
\begingroup

Note that avoiding overlaps (Condition 2) is crucial, as simply \setlength{\columnsep}{5pt}%
\begin{wrapfigure}[4]{r}{.5in}
  \centering
  \vspace{-1.5em}\includegraphics[width=\linewidth]{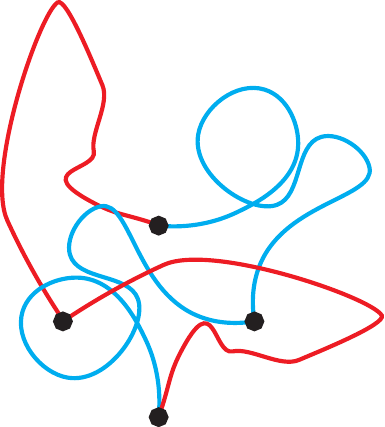}
\end{wrapfigure}
  satisfying the boundary conditions, Equation~\eqref{eq:boundary_constraints}, easily produces meaningless meshes: the inset shows an example satisfying the same periodic boundary conditions as the above illustration with the trees, however self-overlaps render it useless as it does not well-define an inside/outside partition of space. Additionally, self-overlaps may happen not only due to the boundary, but also due to inverted triangles in the interior of the mesh, overlapping other triangles, in turn leading to entangled configurations that  cause severe rendering issues, prohibiting the use of visual guidance via neural networks.

\change{Lastly, note that different choices of the basic tile lead to different boundary correspondences, and to a different family of possible tiles, for the same wallpaper group. These families differ from one another by how the tile copies align with one another (based on the boundary correspondences). Figure~\ref{fig:base_shape_orbII} shows two possible tilings generated for the same wallpaper group: i) using boundary conditions derived from a rhombus-shaped basic tile (same one used for this wallpaper group everywhere else in this paper), leading to 2 pairs of corresponding boundary segments; ii) using boundary conditions stemming from a hexagon-shaped basic tile, with 3 pairs of corresponding boundary segments. We consider the basic tile as a user input, as it relies on a per-case constructive process. We illustrate the basic tile we chose to use for each wallpaper group in the supplementary material.}
\begin{figure}
    \centering
    \captionsetup{justification=centering}
        \begin{subfigure}[b]{0.49\linewidth}
        \includegraphics[width=\linewidth]{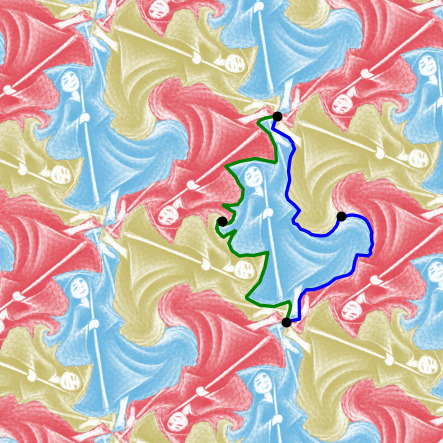}
        \caption{\theprompt{A witch} \orbII}
        \end{subfigure}
        \begin{subfigure}[b]{0.49\linewidth}
        \includegraphics[width=\linewidth]{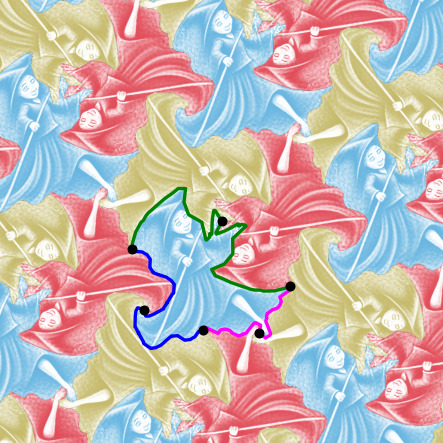}
        \caption{\theprompt{A witch} \orbII}
        \end{subfigure}
\\
        

\captionsetup{justification=justified}
\vspace*{-6pt}
    \caption{\change{\textbf{Different boundary conditions for the same symmetry group.} Different choices of the boundary conditions, Equation~\eqref{eq:boundary_constraints}, lead to different families of tiles for the same wallpaper group, each family having different correspondences on its boundary  (visualized with matching colors): left has two parts of its boundary in correspondence to one another, right has three parts.}}
    \label{fig:base_shape_orbII}
\vspace*{-10pt}
\end{figure}

\endgroup



\subsection{An optimizeable representation of tileable meshes}
\label{ss:diff_tile}
We now focus on the main challenge standing between us and generative tilings: given a chosen wallpaper group $\grp$, devise an \emph{unconstrained} optimization scheme which can modify the vertex positions $\V$ of $\tile$, while ensuring that it remains a \emph{valid tile}, i.e., satisfies the tiling conditions defined at the end of Subsection~\ref{ss:mesh_cond}, which ensure it can cover the entirety of $\R^2$ without any overlaps. 

More specifically, we will represent the vertex positions as a function $\V^\theta$ of another set of parameters $\theta\in\R^k$, and then use $\theta$ in an unconstrained optimization setting. For that, we require that our representation  satisfies the following properties: 
\begin{itemize}[leftmargin=10pt,noitemsep,partopsep=0pt,topsep=0pt,parsep=0pt]

    \item\textbf{Validity}:  $\tile=\parr{\V^\theta,\T}$ is a valid tile for \emph{any} $\theta$. This entails the representation yields an \emph{unconstrained} optimization problem.
    \label{item:valid}
    \item \textbf{Completeness}: Any valid tile can be achieved for some $\theta$. This ensures our representation is expressive and not restricted to a limited class of tile shapes.
    \label{item:all}
    \item\textbf{Differentiability}:  $\frac{\partial V^\theta}{\theta}$ exists and is well-defined, except for a measure-0 set. This enables gradient-based optimization.
\end{itemize}
The validity property requires satisfying the two necessary and sufficient tiling conditions from Subsection~\ref{ss:mesh_cond}: Condition~\ref{item:boundary} enforces the boundary conditions, Equation~\eqref{eq:boundary_constraints}, which are merely linear equalities that can be easily incorporated within unconstrained optimization. On the other hand, Condition~\ref{item:overlap_cond} mandates that the mesh has no self-overlaps, leading to a significant challenge as the space of overlap-free configurations of a given triangulation is highly non-convex and cannot be directly enforced as a constraint.

To construct our differentiable representation, we first consider another work, Orbifold Tutte Emebedding~\cite{aigerman2015orbifold}, which focused on a different task related to the same type of tilings as ours, namely computing overlap-free \emph{UV maps} of 3D meshes into those tilings. Their algorithm can be summarized in our context as follows. Given a mesh $\mesh$ with boundary conditions as in Equation~\eqref{eq:boundary_constraints}: {i. Compute the mesh's laplacian $L$} (they use either the cotan~\cite{pinkall1993computing} or MVC~\cite{floater2003mvc} laplacians);
{ii. Solve a linear system}  which requires that the boundary conditions of Equation \eqref{eq:boundary_constraints} are satisfied, and additionally that the embedding is discrete-harmonic with respect to $L$, \ie, each $\vertex_i$ satisfies: 

\begin{equation}
    \label{eq:harmonic}
    \sum_{\vertex_j\in\mathcal{N}^i}L_{ij}\parr{\vertex_i-\vertex_j}=0, 
\end{equation}

where $\mathcal{N}^i$ are all the neighbors of $\vertex_i$ in the tiled mesh. This algorithm is guaranteed to produce a mesh that is a valid tile without any self-overlaps. However, it does not provide any way to control the shape of the resulting tile. As such, the embedding cannot be interacted with, necessary for modifying the mesh w.r.t. the text-guided objective function of SDS~\cite{dreamfusion} (see comparison in Figure~\ref{fig:comparison}). Thus, we next propose a method to extend~\cite{aigerman2015orbifold} to  a differentiable representation.



Our core observation is that instead of using a \emph{fixed} linear system as is done in~\cite{aigerman2015orbifold}, we can construct a differentiable parameterization of a \emph{space} of linear systems, whose solutions parameterize exactly the \textit{entire} space of valid tile shapes. This is achieved by parameterizing the space of mesh \emph{Laplacians}: Assign a scalar parameter $\theta_{ij}\in\R$ to each edge $\parr{i,j}$. Let $f:\R\leftrightarrow\R^+$ be a bijective function from the real numbers to the positive real numbers. Define the mesh Laplacian $L^\theta$ whose entries are parameterized by $\theta$:

\begin{equation}
    L^\theta_{ij}= \begin{cases}
  -f\parr{\theta_{ij}}  &  \text{edge $\parr{i,j}$ exists in $\tile$} \\
 \sum_{k}{f\parr{\theta_{ik}}} & \text{$i=j$}\\
 0 & \text{else}
\end{cases}
\end{equation}

Finally, define the vertex positions $\V^\theta$ as the solution to the linear system \eqref{eq:boundary_constraints}, \eqref{eq:harmonic}, w.r.t the laplacian $L^\theta$, to obtain the parameterization $\theta\to\V^\theta$. This parameterization is trivially differentiable, as a composition of differentiable functions. The next theorem proves that this representation can yield all valid tilings, and only them.

\begin{thm}\label{thm:OTE} A placement of the vertices $\V^*$  is a valid tile iff there exist parameters $\theta$ s.t. $\V^*$ satisfies Equations~\eqref{eq:boundary_constraints},\eqref{eq:harmonic} w.r.t. $L^\theta$.
\end{thm}
\begin{proof} $(\Rightarrow)$ 
Assume the mesh is tileable. Then by the necessary and sufficient tiling conditions (Section~\ref{ss:mesh_cond}) we have that $\V^*$ satisfies Equation~\eqref{eq:boundary_constraints}, and additionally, the configuration has no overlaps. For a mesh with no overlaps, each vertex lies in the convex hull of its neighbors in the tiling,  and hence can be written in positive
barycentric coordinates of its neighbors, $v_i = \sum_{j\in\mathcal{N}^i}  b_{ij}v_j$. Since $b_{ij}$ sum to $1$ for each $i$, we have $\sum_{j\in\mathcal{N}^i} b_{ij} (v_j-v_i)=0$. Choosing $\theta$ s.t. $f\parr{\theta_{ij}} =  b_{ij}$ entails $\V^*$ also satisfies Equation~\eqref{eq:harmonic}$. (\Leftarrow)$ this direction is proved in~\cite{aigerman2015orbifold}. 
\end{proof}

 We do not pertain to significant mathematical novelty by presenting this proof, as it states a fact that is immediate to deduce to many people working on graph embeddings and graph laplacians in broader contexts than ours. Our contribution is in the observation that this specific fact yields an unconstrained, differentiable parameterization of valid, overlap-free tilings.  As far as we know, we are the first to provide such a parameterization and to propose to use the laplacian as a differentiable parameter to control tilings. Note: we did not assume the laplacian is symmetric ($L^\theta_{ij}\neq L^\theta_{ji}$), and indeed our method fully supports non-symmetric laplacians. 

Lastly, every wallpaper group is well-defined up to a rigid motion, and since we optimize for a perceptive loss which is not rotation-invariant, we apply a global rotation $\mathbf{R}_\phi$ by an angle $\phi$, \ie, $\V^\theta\cdot\mathbf{R}^T_\phi$, and add $\phi$ as a variable of our optimization.

\subsection{Text-guided optimization  via the differentiable pipeline}
\label{ssec:tile_optim}
The other components of our pipeline follow a standard approach:
\paragraph{Rendering the tile.} The  mesh we use in all experiments is  a regular triangulation of a $40\times40$ square (however initialization with any valid tile is possible). We use the initial vertex coordinates as UV coordinates, which are kept static throughout optimization. We texture the mesh using an image $\image$, whose pixel values we optimize.  We render the textured mesh using a differentiable renderer (Nvdiffrast~\cite{nvdiffrast}), thereby obtaining a raster image of the tile $\render$, from which we  back-propagate gradients to our optimizable parameters, see Figure~\ref{fig:overview}.

\paragraph{Loss via Score Distillation Sampling.} To optimize the rendered tile's appearance to match the given text prompt,  we use Score Distillation Sampling (SDS)~\cite{dreamfusion}, which  defines a loss  $\mathcal{L}_\text{SDS}$ via a pre-trained image diffusion model, enabling text-guided optimization of the tile's appearance. We emphasize SDS serves as a blackbox. Any other differentiable loss can be used, either from perceptive modules, or a geometric one. \change{For instance, it is possible to additionally added a regularization term to reduce the distortion of the deformation, though for our task which requires extreme deformation, large amounts of distortion may be desired.}

\paragraph{Avoiding trivial square tiles.}
We employ three simple strategies to prevent generation of square images with background color: first, we randomize the background color of the renderings to prevent SDS from using a specific color as a background color the mesh (see Figure~\ref{fig:ablation} for an ablation); second, we render the shape without texture in half of the optimization steps; third, we set a higher learning rate  for the geometry optimization than for the texture, such that textures emerge slower, giving the shape time to form. We found that these strategies mitigate the trivial solution systematically.
\paragraph{Greyscale vs. colored textures.} 
For this paper, we choose to focus mainly on greyscale textures and not RGB colors, although our method supports color as well (see Figure~\ref{fig:ablation} and supplementary). This is mainly due to the following reason: we follow~\cite{magic3d,dreamfusion, vectorart}, who observe high \textit{guidance} significantly improves the quality of the results. The guidance parameter $\alpha$ defines the extrapolation magnitude, $e = e_p + \alpha(e_p-e_n)$ where $e$ is the final text embedding used, $e_p$ is the text embedding of the positive prompt (the desired target image) , and $e_n$ is the negative prompt (images to avoid)   - we use the empty string. We use $\alpha=100$ in all experiments. Unfortunately, as noted in the above papers, setting a high guidance leads to saturated colors, which is even more pronounced in 2D than in the previous 3D applications which used SDS. Thus, we choose to mainly use greyscale textures. Once a better image-based text-guided approach will be available instead of SDS, our method will seamlessly be able to swap out one for the other, as it does not rely on it except for defining the final loss to optimize. We additionally note that in our current optimization, colored images lead to the resulting tiles containing slightly more background colors (see Figure~\ref{fig:ablation}).

\paragraph{Implementation details.}
We use Pytorch \cite{pytorch} and Adam~\cite{kingma2014adam} with batch size $8$. We use~\cite{torchsparsesolve} as a sparse solver,  Nvdiffrast~\cite{nvdiffrast} for differentiable rendering, the SDS  implementation of Threestudio~\cite{threestudio,diffusers} with Stable Diffusion~\cite{stablediffusion}.  We use a learning rate of $0.1$ for the laplacian parameters $\theta$, and $0.01$ for $\image$. We apply the sigmoid function $f$ to $\theta$ and $\image$ and  affinely rescale the result to the range $\brac{0.05,0.95}$.
We  optimize for $6000$ iterations which take $30$ minutes on a single A100 GPU. We have not optimized hyper-parameters for speed.

\section{Results}

\begin{figure}[t]
    \centering    \captionsetup{justification=centering}
        \begin{subfigure}[b]{0.49\linewidth}
    \includegraphics[width=\linewidth]{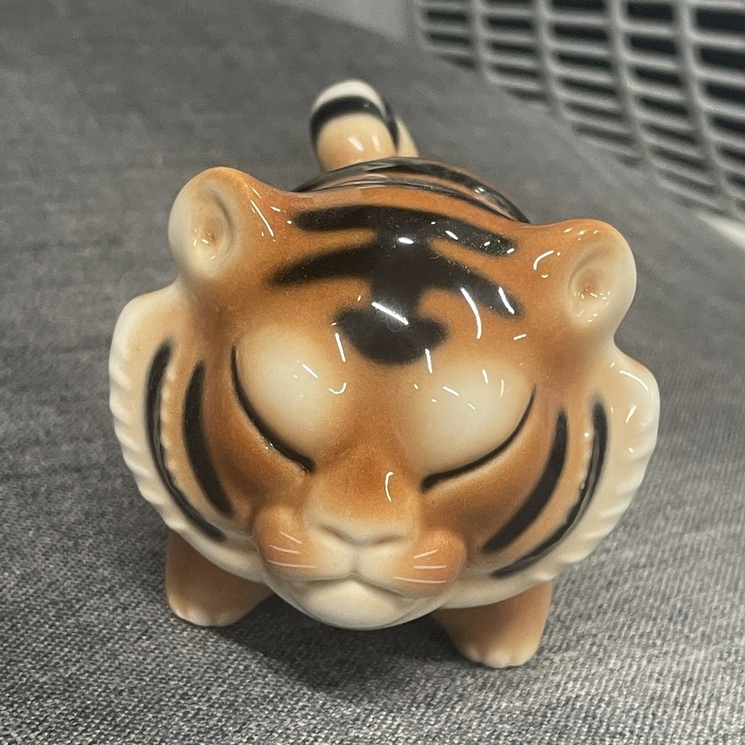}
        \end{subfigure}    
         \begin{subfigure}[b]{0.49\linewidth}
    \includegraphics[width=\linewidth]{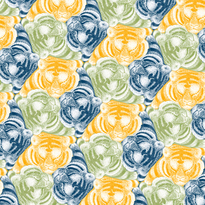}
        \end{subfigure}
        \\
        \begin{subfigure}[b]{0.49\linewidth}
    \includegraphics[width=\linewidth]{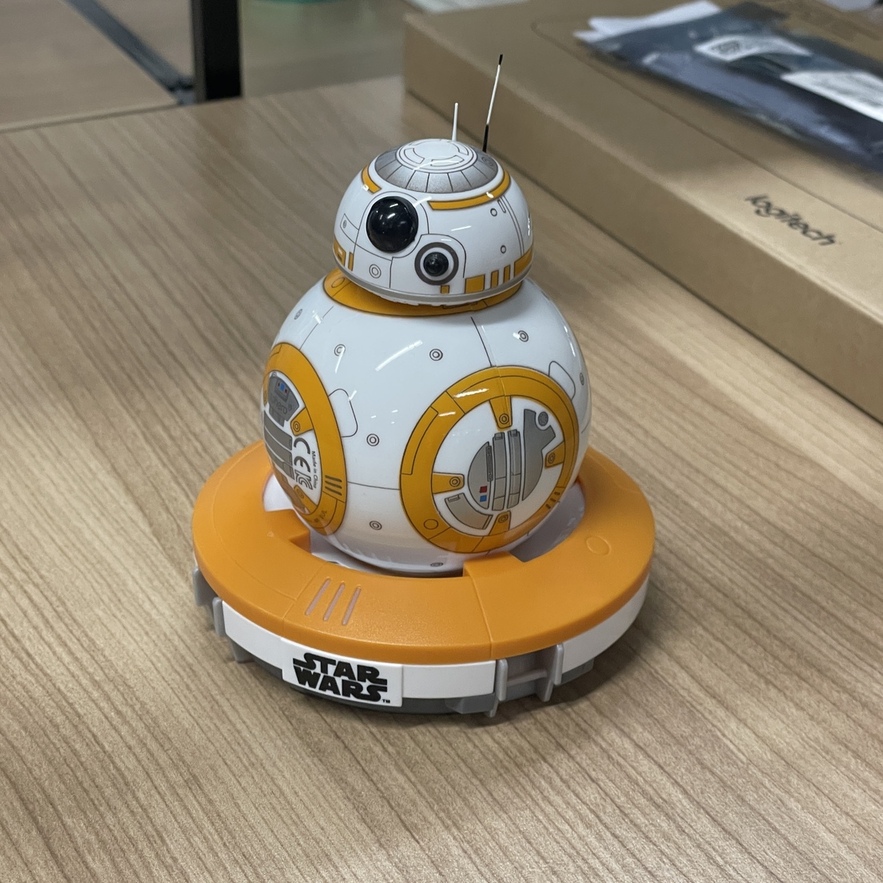}
        \end{subfigure}   
        \begin{subfigure}[b]{0.49\linewidth}
        \includegraphics[width=\linewidth]{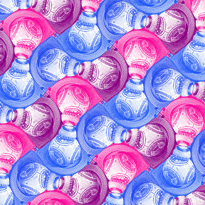}
        \end{subfigure}
        \\
        \begin{subfigure}[b]{0.49\linewidth}
\includegraphics[width=\linewidth]{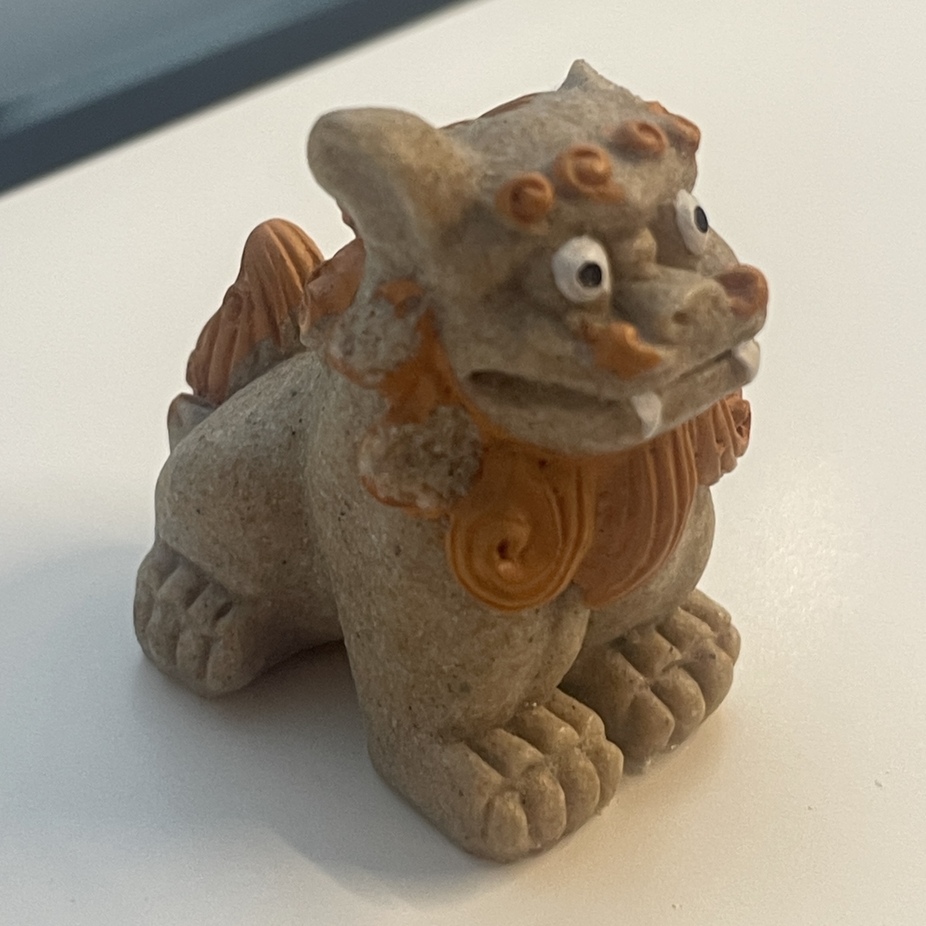}
        \end{subfigure}
        \begin{subfigure}[b]{0.49\linewidth}
        \includegraphics[width=\linewidth]{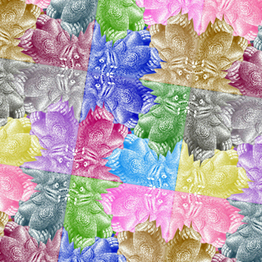}
        \end{subfigure}

\captionsetup{justification=justified}
    \caption{\textbf{Tiles from given photos.} We can generate tiles of objects from given photos, following Dreambooth~\cite{dreambooth}.}
    \label{fig:dreambooth}
\end{figure}
\begin{figure}[t]
    \centering
    \captionsetup{justification=centering}
        \begin{subfigure}[b]{0.32666666666666666\linewidth}
        \includegraphics[width=\linewidth]{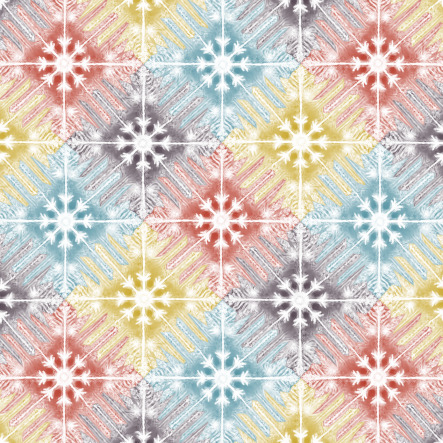}
        \caption{\theprompt{A snowflake} \orbI}
        \end{subfigure}
        \begin{subfigure}[b]{0.32666666666666666\linewidth}
        \includegraphics[width=\linewidth]{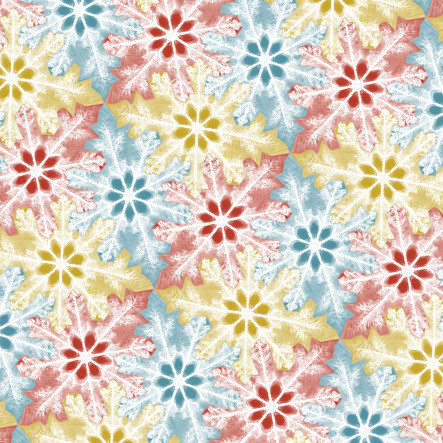}
        \caption{\theprompt{A snowflake} \orbII}
        \end{subfigure}
        \begin{subfigure}[b]{0.32666666666666666\linewidth}
        \includegraphics[width=\linewidth]{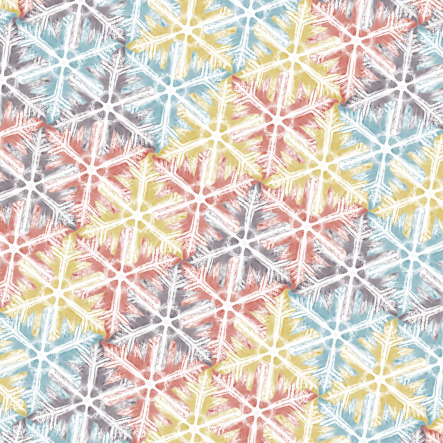}
        \caption{\theprompt{A snowflake} \orbII}
        \end{subfigure}
\\
\captionsetup{justification=justified}
\vspace{-5pt}
    \caption{\textbf{Crystals generated in $3$ different symmetry groups.} Our method's expressiveness enables it to generate additional visual inner-symmetries (reflections) within each tile.}
    \label{fig:snowklakes}
\end{figure}

We now detail our main results, ablations, and comparisons. We evaluated our method over a large collection of prompts describing diverse objects, in several different styles, for all wallpaper groups (see supp. for a detailed discussion). In most cases, the tiles capture solely the desired object, with its appearance being highly-plausible in accordance with the desired prompt. Our method always produces tiles which tile the plane without overlaps.

To visualize the infinite tiling produced by the optimized tiles, we render copies of the mesh transformed via the  isometries of the wallpaper group, coloring each copied tile with a different color. \change{The selection of color is left as an aesthetic choice for the user - we show examples where we color each copy consistently based on its orientation (similarly to Escher's works) or with a random arbitrary color.} \textit{We emphasize the final tiling is only constructed for visualization purposes and is not part of the optimization process - we optimize a single grayscale rendering (see Figure~\ref{fig:overview}).}

 We annotate a short summary of the text prompt under each tiling, as our full prompts are lengthy and contain instructions for the desired visual style, as is usually done in generative imaging. The necessity for this ``prompt engineering''  is mainly indicative to SD's~\cite{stablediffusion} behavior. We emphasize that the prompts do not contain instructions regarding the image's background nor tileability.  We list the full prompts in the supplemental. 

Out of the 17 wallpaper groups, the 13 shown in Figure~\ref{fig:teaser} and Figure~\ref{fig:other_five} are ``interesting'', in the sense of having tiles that are not fixed. The remaining 4 are generated solely by reflections, and have a fixed boundary that forms a convex polygon, which is uninteresting for our task as it implies solely optimizing texture of a fixed convex polygon. We thus conduct experiments solely on the 13. Please refer to the supp. for results on the other 4. 

Please also refer to the supplemental for additional examples of each of the experiments, as well as other results: generating tilings for all 17 wallpaper groups for a single prompt in order to evaluate the success rate of the method,  textual-inversion~\cite{dreambooth} to generate tiles from captured photos, and fabrication of the tiles.
\subsection{Experiments}
\paragraph{More results.}Figure~\ref{fig:gallery} shows a gallery of various results of our method (see two more such galleries in the supplemental). In spite of the stringent geometric constraints on the shape of the tile, our method still manages to produce tiles with highly non-convex, elongated and convoluted shapes (puzzle piece, yoga pose, letter Z). The optimization often results in a ``creative'' way to plausibly represent the desired object, e.g., rendering the dalmatian dog in a wide-angle shot. In cases in which a realistic photograph of the object is requested, the optimization may choose more convex silhouettes (chimpanzee), or truncating less-salient parts (kangaroo). \textit{Please refer to the supplementary for several more galleries}.


\paragraph{Split tiles.} Inspired by Escher's tilings which often have two interleaved objects~\cite{dihedral}, we can split the tile into $k$ sub-tiles, assign them $k$ prompts, render each separately, and sum the SDS loss across all of them (we still treat the mesh as a single tile). Figure~\ref{fig:dual} shows a couple of our results - see the supplementary for more. We note that we observed the losses competing with one another, leading to less-plausible results, and more convex shapes.

\begin{figure}[t] 
    \centering
    \captionsetup{justification=centering}

    \begin{subfigure}[t]{0.49\linewidth}
        \includegraphics[width=\linewidth]{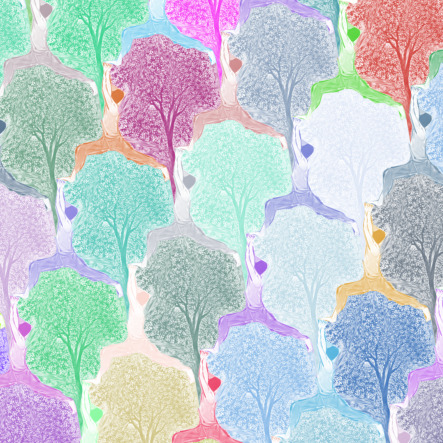}
        \caption{\theprompt{Man doing yoga in a tree pose} + \theprompt{A tree} \torus} 
    \end{subfigure}
    \hfill
    \begin{subfigure}[t]{0.49\linewidth}
        \includegraphics[width=\linewidth]{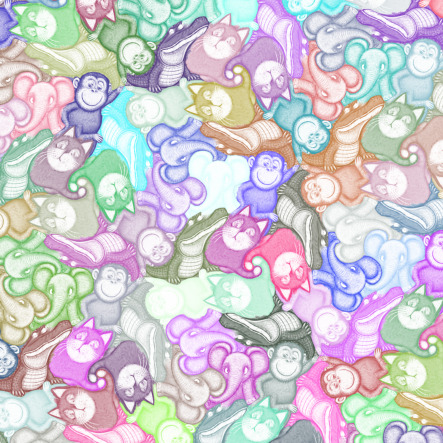}
        \caption{\theprompt{A cat} +
        \theprompt{An alligator} +
        \theprompt{A monkey} +
        \theprompt{An elephant} \orbII } 
    \end{subfigure}
    
    \captionsetup{justification=justified }
    \caption{\textbf{Multi-tile generation.} By dividing the mesh, our approach can optimize  $k$ sub-tiles, for $k$ different prompts (here $k=2,4$).}
    \label{fig:dual}
\end{figure}

\paragraph{Internal symmetries.} Instead of a single tile, We can optimize the appearance of the result of several tiles tiled together. As Figure~\ref{fig:symmetry_figure} shows, this enables generating tilings in which each tile is an ``uber tile'' made up of copies of smaller tiles, thus holding internal symmetries (each crab is mirror-symmetric and each rose is 6-way rotationally-symmetric).   
\begin{figure}
    \centering
    \captionsetup{justification=centering}
        \begin{subfigure}[b]{0.49\linewidth}
        \includegraphics[width=\linewidth]{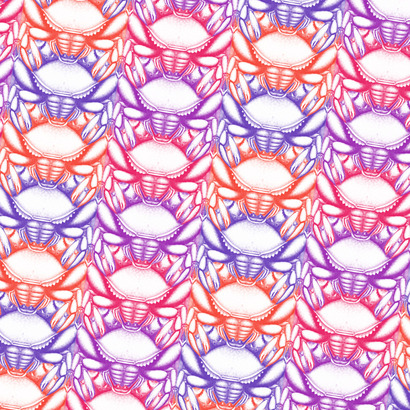}
        \caption{\theprompt{A crab} \cylinder}
        \end{subfigure}
        \begin{subfigure}[b]{0.49\linewidth}
        \includegraphics[width=\linewidth]{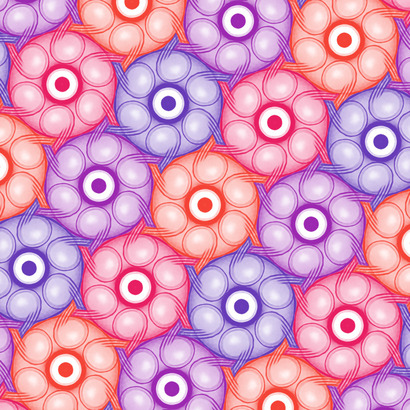}
        \caption{\theprompt{Ninja eye with powers} \orbIII}
        \end{subfigure}
\\
        \begin{subfigure}[b]{0.49\linewidth}
        \includegraphics[width=\linewidth]{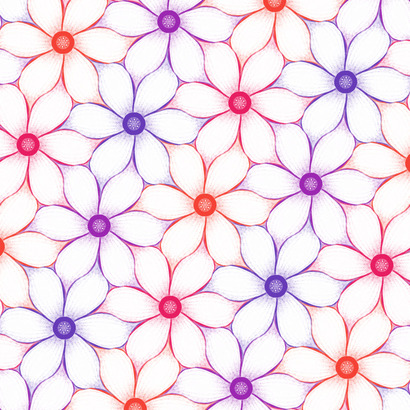}
        \caption{\theprompt{A flower} \orbIII}
        \end{subfigure}
        \begin{subfigure}[b]{0.49\linewidth}
        \includegraphics[width=\linewidth]{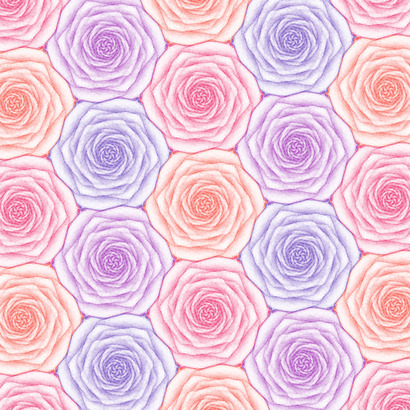}
        \caption{\theprompt{A rose} \orbIII}
        \end{subfigure}
\captionsetup{justification=justified}
    \caption{\textbf{Internally-symmetric tiles.} Considering several tiles as one copy while optimizing for visual loss enables generating tiles which have internal symmetries (reflectional for $\cylinder$, 6-way rotational for $\orbIII$, in these examples).}
    \label{fig:symmetry_figure}
\end{figure}

\paragraph{Fabrication.} Our approach could potentially be used to design fabricable tiles for interior design - we show a small proof of concept with 3D-printed tiles in  Figure~\ref{fig:farbication}. As our method produces perfectly-tileable meshes, the printed pieces perfectly-fit into one another.

\paragraph{Textual Inversion}
In Figure~\ref{fig:dreambooth} we show we can additionally use DreamBooth~\cite{dreambooth} to receive several photos of an object, find its corresponding text embedding and then produce a tile that matches its visuals. This approach is more limited as the method usually either does not match the object's appearance, or does not produce an interesting tile.


\paragraph{Failure cases.} A few examples of the common failure modes we encountered are shown in Figure~\ref{fig:limitation}. We inherit the limitations of the underlying diffusion model~\cite{stablediffusion}, such as its inability to consistently produce five-fingered hands. For the Espresso, a close-to-convex tile is achieved due to the unwarranted addition of the saucer, leading to an uninspiring tiling. Objects such as the ants, which require many thin, elongated features often lead to a significant increase in background pixels in the tile.

To illustrate the success rate of our method, we provide three unfiltered full galleries in the supplementary for the output of our method for all 17 wallpaper groups from a single given prompt.

\begin{figure}[b]
    \centering
    \captionsetup{justification=centering}
        \begin{subfigure}[b]{0.32666666666666666\linewidth}
        \includegraphics[width=\linewidth]{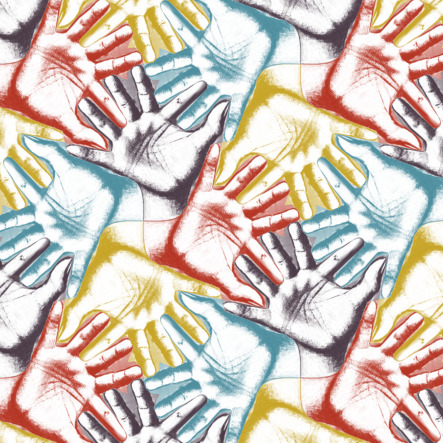}
        \caption{\theprompt{Open hand} \orbI}
        \end{subfigure}
        \begin{subfigure}[b]{0.32666666666666666\linewidth}
        \includegraphics[width=\linewidth]{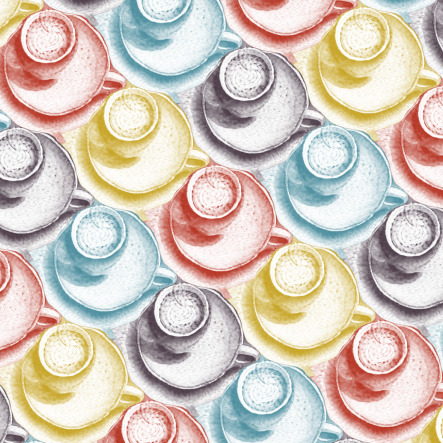}
        \caption{\theprompt{An espresso} \torus}
        \end{subfigure}
        \begin{subfigure}[b]{0.32666666666666666\linewidth}
        \includegraphics[width=\linewidth]{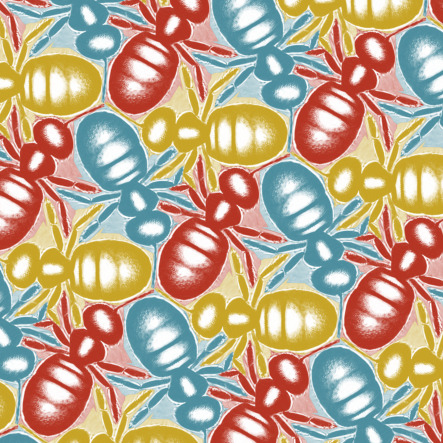}
        \caption{\theprompt{An ant} \orbII}
        \end{subfigure}
\\
\captionsetup{justification=justified}
    \caption{\textbf{Limitations.} We display some failure modes we encountered. Left: StableDiffusion struggles at generating five-fingered results. Middle: tiles with uninteresting shapes lead to underwhelming tilings. Right: thin features lead to unfilled areas and missing elements (left antenna). }
    \label{fig:limitation}
    \vspace{-10pt}
\end{figure}

\subsection{Ablations}
\begin{figure}[t]
    \centering
    \captionsetup{justification=centering}
\rotatebox{90}{\ \ \ \  \ \ \ \ \ \ \ \ \ Ours\textcolor{white}{\textbf{g}}}
          \vspace{0pt}\begin{subfigure}[t]{0.31\linewidth}
        \includegraphics[width=\linewidth]{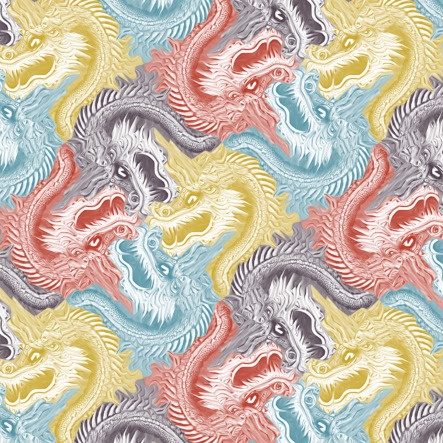}
        \end{subfigure}
          \vspace{0pt}\begin{subfigure}[t]{0.31\linewidth}
        \includegraphics[width=\linewidth]{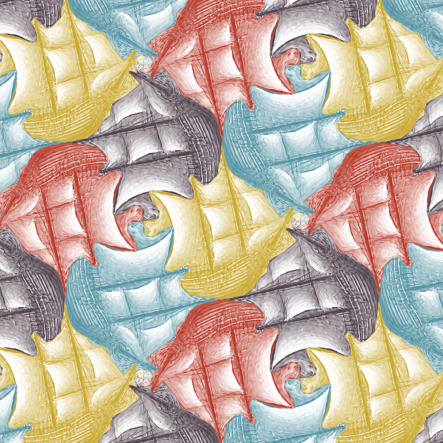}
        \end{subfigure}
          \vspace{0pt}\begin{subfigure}[t]{0.31\linewidth}
        \includegraphics[width=\linewidth]{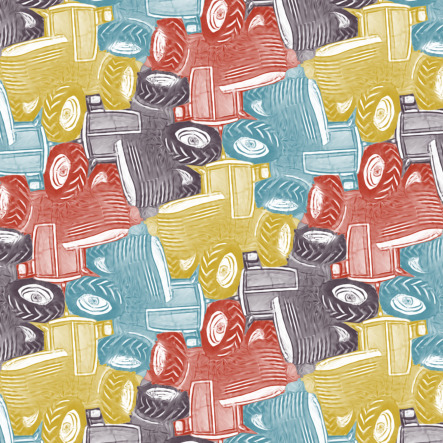}
        \end{subfigure}
\\
\rotatebox{90}{\ \ \ Fixed Background}
          \vspace{0pt}\begin{subfigure}[t]{0.31\linewidth}
        \includegraphics[width=\linewidth]{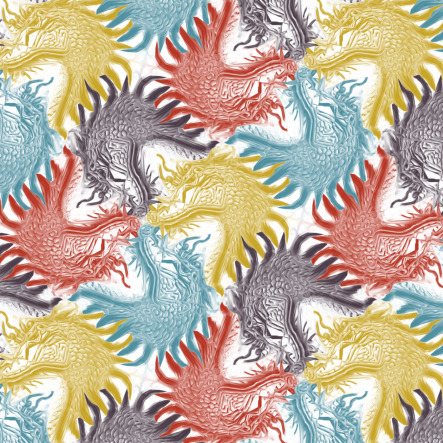}
        \end{subfigure}
          \vspace{0pt}\begin{subfigure}[t]{0.31\linewidth}
        \includegraphics[width=\linewidth]{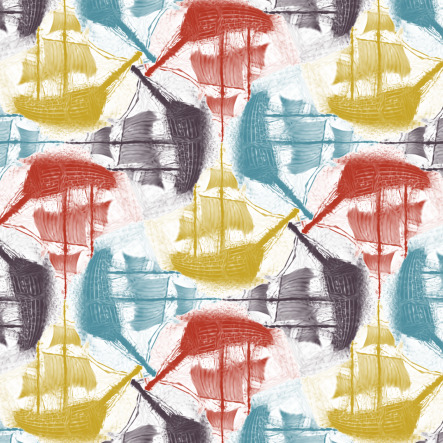}
        \end{subfigure}
          \vspace{0pt}\begin{subfigure}[t]{0.31\linewidth}
        \includegraphics[width=\linewidth]{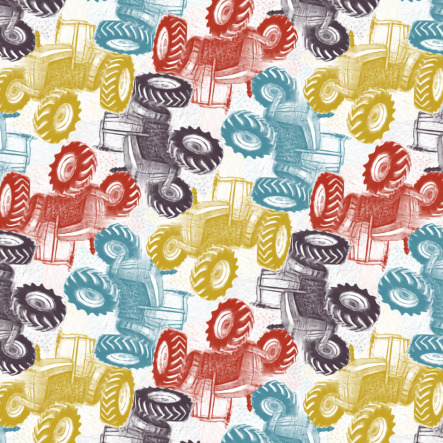}
        \end{subfigure}
\\
\rotatebox{90}{\ \ \ \ \ \ \ \ \ \ \ \ \ \  RGB\textcolor{white}{\textbf{g}}}
          \vspace{0pt}\begin{subfigure}[t]{0.31\linewidth}
        \includegraphics[width=\linewidth]{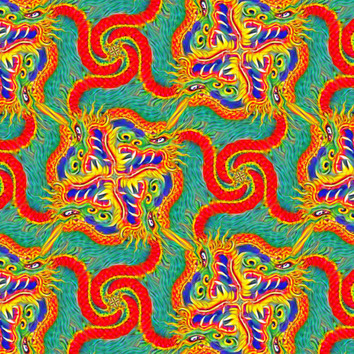}
        \caption{\theprompt{Dragon} \orbI}
        \end{subfigure}
          \vspace{0pt}\begin{subfigure}[t]{0.31\linewidth}
        \includegraphics[width=\linewidth]{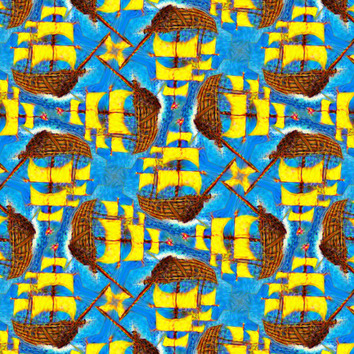}
        \caption{\theprompt{Pirate ship} \orbI}
        \end{subfigure}
          \vspace{0pt}\begin{subfigure}[t]{0.31\linewidth}
        \includegraphics[width=\linewidth]{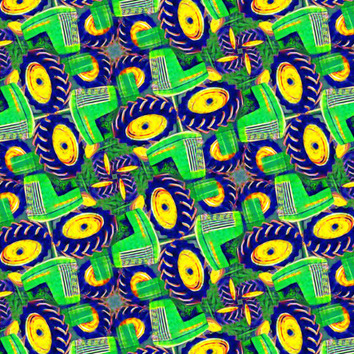}
        \caption{\theprompt{A tractor} \orbI}
        \end{subfigure}
\\
\captionsetup{justification=justified}
\vspace{-10pt}
    \caption{\textbf{Ablation.} Comparison between \textit{our} method (top), maintaining a \textit{fixed background} color during optimization (middle), and using \textit{RGB} textures for the meshes (bottom). Both alternatives result in areas unoccupied by the foreground object. \textit{RGB} additionally leads to loss of plausibility. }
    \label{fig:ablation}
\end{figure}

We evaluated the effect of our design choices. Figure~\ref{fig:ablation} shows representative examples, comparing the results for the same prompt and symmetry group between our full method (top row) and two alternative choices. \textit{Middle row}: not randomizing the background color of the rendering during optimization (see Section~\ref{ssec:tile_optim}), leading to the optimization not modifying the tile's geometry, instead coloring-out the areas it aims to remove using the background color. 
\textit{Bottom row}: results generated with full RGB textures instead of black and white.  RGB often leads to background pixels (pirate ship), possibly due to the increased expressiveness of the texture space. The saturation of the colors seems to also lead to loss of plausibility (tractor). Refer to the supp. for additional examples.

\subsection{Comparisons}
\begin{figure}
    \centering
    \captionsetup{justification=centering}
        \captionsetup{justification=centering}
\rotatebox{90}{\ \ \ \  \ \ \ \ \ \ \ \ \ Ours\textcolor{white}{\textbf{g}}}
        \begin{subfigure}[t]{0.31\linewidth}
        \includegraphics[width=\linewidth]{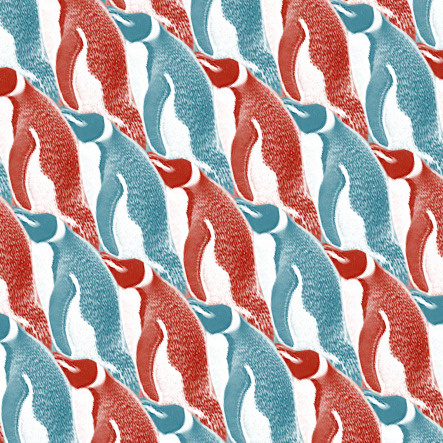}
        \end{subfigure}
        \begin{subfigure}[t]{0.31\linewidth}
        \includegraphics[width=\linewidth]{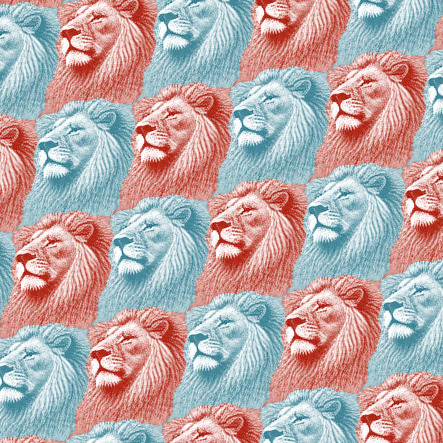}
        \end{subfigure}
        \begin{subfigure}[t]{0.31\linewidth}
        \includegraphics[width=\linewidth]{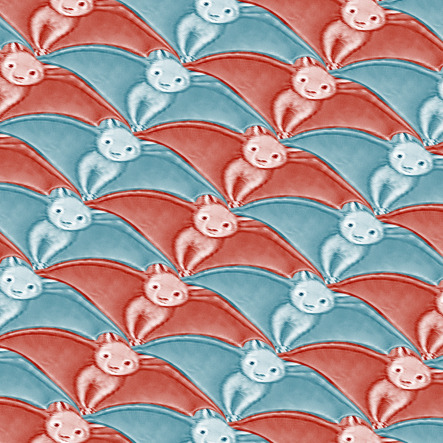}
        \end{subfigure}
\\


\captionsetup{justification=centering}
\rotatebox{90}{\ \ \ \  \  \ OTE + SDS \textcolor{white}{\textbf{g}}}
        \begin{subfigure}[t]{0.31\linewidth}
        \includegraphics[width=\linewidth]{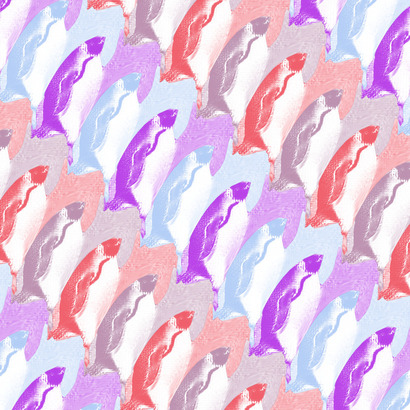}
        \end{subfigure}
        \begin{subfigure}[t]{0.31\linewidth}
        \includegraphics[width=\linewidth]{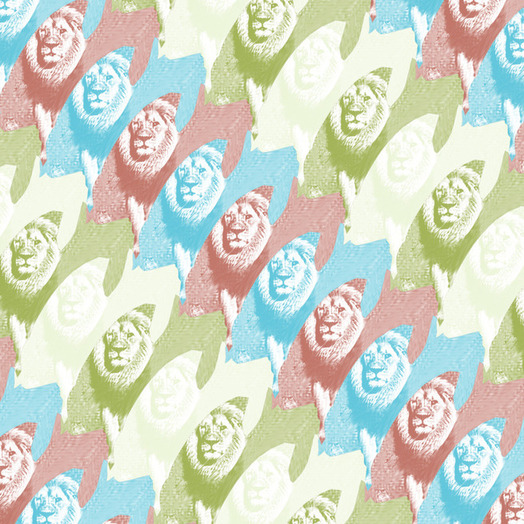}
        \end{subfigure}
        \begin{subfigure}[t]{0.31\linewidth}
        \includegraphics[width=\linewidth]{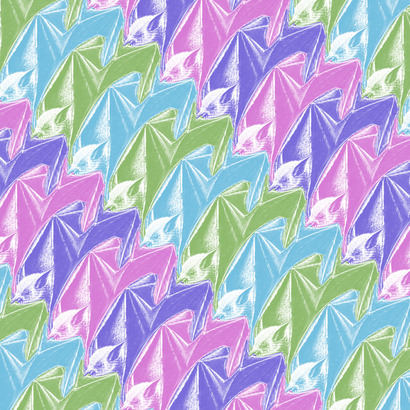}
        \end{subfigure}



        \captionsetup{justification=centering}
\rotatebox{90}{\ \ \ \  \  Tiled Image SDS \textcolor{white}{\textbf{g}}}
        \begin{subfigure}[t]{0.31\linewidth}
        \includegraphics[width=\linewidth]{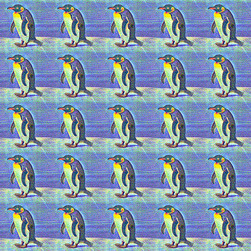}
        \caption{\theprompt{A penguin} \torus}
        \end{subfigure}
        \begin{subfigure}[t]{0.31\linewidth}
        \includegraphics[width=\linewidth]{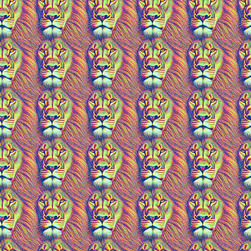}
        \caption{\theprompt{A lion} \torus}
        \end{subfigure}
        \begin{subfigure}[t]{0.31\linewidth}
        \includegraphics[width=\linewidth]{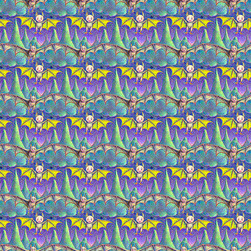}
        \caption{\theprompt{A bat} \torus}
        \end{subfigure}
\\
        \captionsetup{justification=centering}
\rotatebox{90}{\ Prompt + Diffusion\textcolor{white}{\textbf{g}}}
        \subcaptionbox{\tiny{``An Escher-like image of a penguin.''}}[0.31\linewidth]{\includegraphics[width=\linewidth]{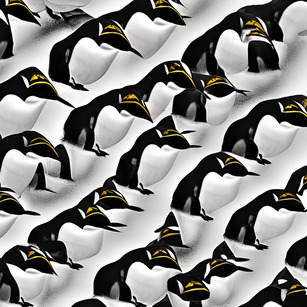}}
        \subcaptionbox{\tiny{``An image composed of the repetition of a lion, best quality, extremely detailed in the style of Escher. There should be no gaps nor overlaps.''}}[0.31\linewidth]{\includegraphics[width=\linewidth]{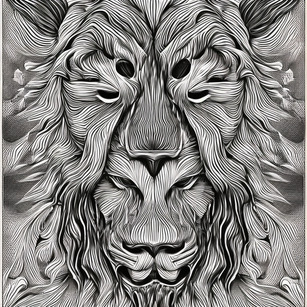}}\hfill
        \subcaptionbox{\tiny{``An Escher-like tiling of a bat.''}}[0.31\linewidth]{\includegraphics[width=\linewidth]{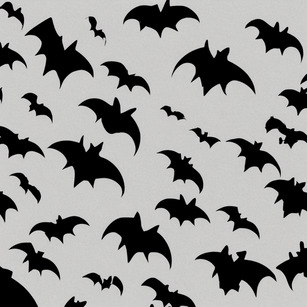}}
\captionsetup{justification=justified}
    \vspace{-10pt}\caption{\textbf{Baseline comparisons.} We compare \textit{our} method to three alternatives. 
    \textit{OTE + SDS}: using Orbifold Tutte Embeddings~\cite{aigerman2015orbifold} to get a (arbitrary) tileable mesh, and then optimizing its texture using SDS~\cite{dreamfusion}.\textit{Tiled Image SDS}: duplicating a pixel image in a 5x5 grid and optimizing the grid w.r.t SDS so as to account for the tiling during optimization; \textit{Prompt + Diffusion}: guiding the image diffusion model~\cite{stablediffusion}  to produce tilings through given prompts (the results are far from perfectly-repeating). All baselines produce results with significant areas unoccupied by the foreground object.}
    \label{fig:comparison}
\end{figure}


Figure~\ref{fig:comparison} shows comparisons between our method (\textit{top row}) and other methods. As there are no existing techniques that directly aim to solve the same task of perfectly-repeating tilings of a foreground object, we compare to the available alternatives. One such alternative is to use Orbifold Tutte Embeddings~\cite{aigerman2015orbifold} to produce a valid tile (without direct control on the shape of the resulting tile), and then texture the static mesh by optimizing its texture using SDS~\cite{dreamfusion}. As discussed before, the lack of ability to to modify the geometry leads to significant parts of the image occupied by background colors. Another alternative (\textit{third row}) is to operate on image tiles without a mesh: optimize the pixels of a square image, duplicate it to create a tiling (we create a $5\times5$ grid), and then feed random crops of the tiling to the same loss we use (SDS~\cite{dreamfusion}) so that the optimization accounts for the tiling.  Ignoring the saturated colors which are an artifact of SDS, this method succeeds in creating images with seamless borders, but produces tiles with a significant amount of background pixels, and cannot account for all wallpaper groups.  \noam{reminder add TUtte comparison}.

 The third alternative ({bottom row}) uses an image diffusion model (StableDiffusion~\cite{stablediffusion}),
 and attempts to manually tweak the text prompt in order to produce a tiling. We chose the results with the prompts that yielded the best visuals out of many attempts, see the supplementary for all prompts that were tried. This approach evidently does not create tileable shapes, and as before, many parts of the image do not contain the foreground object.

These results verify the necessity of  optimization of geometry, as these alternatives only support a fixed, predetermined boundary between the generated objects. This boundary cannot be directly optimized without an intermediate deformation such as ours, especially when considering rotational symmetries with other angles than $90^\circ$. 

\section{Conclusion}
We provide a text-guided method to automatically generate intricate and plausible textured objects which perfectly tile the plane. This implies possible uses in architecture, interior design, and fabric design which we are eager to explore. 

We do believe that our key technical contribution - the differentiable  representation of tileable meshes - may lead to important followups. For one, as shown in~\cite{aigerman2015orbifold}, embedding two sphere-topology meshes into an appropriate tile structure enables computing a seamless bijection between the two. Through our parameterization, this 1-to-1 map can be optimized, or otherwise predicted by a neural network in a data-driven setting. Following this line of thought,~\cite{aigerman2016hyperbolic} show that~\cite{aigerman2015orbifold} can be extended to the hyperbolic plane, where the set of possible tilings is infinite, enabling a larger class of embeddings and maps. Hence, control of the embedding via the weights can be highly beneficial. Lastly, we note that while an extension to 3D volumetric tiles is extremely appealing, it is unfortunately well known that Tutte's embedding~\cite{tutte1963draw} does not yield overlap-free meshes when applied to volumetric meshes~\cite{campen2016bijective} hence this extension is unlikely.

We note our method is limited to a specific class of tilings, prohibiting, e.g., aperiodic tiles~\cite{penrose1974role,smith2023aperiodic}. We further note that the tiling constraints usually prevent successful generation for prompts containing multiple objects, or very specific descriptions. Lastly, our method relies on Score Distillation Sampling~\cite{dreamfusion} and inherits its limitations: slow runtime, over-saturated colors, and lack of granular control over the resulting imagery. Since we treat SDS merely as a blackbox for providing our framework with a loss, we could seamlessly swap it out for any future technique which will exhibit better performance. 

\bibliographystyle{ACM-Reference-Format}
\bibliography{refs}
\clearpage


\begin{figure*}[h]
    \centering
    \captionsetup{justification=centering}
        \begin{subfigure}[t]{0.196\linewidth}
        \includegraphics[width=\linewidth]{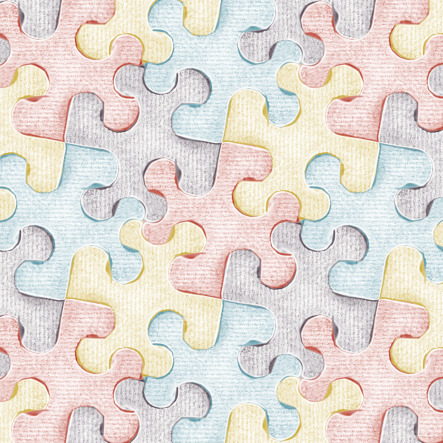}
        \caption{\theprompt{A puzzle piece} \orbI}
        \end{subfigure}
        \begin{subfigure}[t]{0.196\linewidth}
        \includegraphics[width=\linewidth]{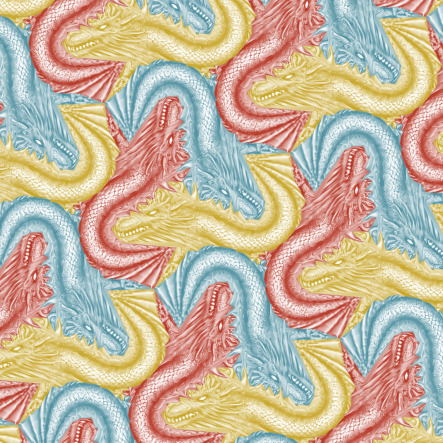}
        \caption{\theprompt{A dragon} \orbIII}
        \end{subfigure}
        \begin{subfigure}[t]{0.196\linewidth}
        \includegraphics[width=\linewidth]{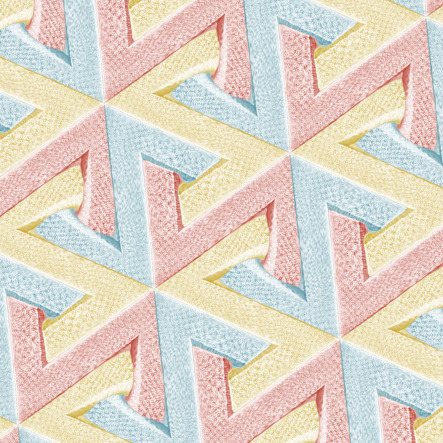}
        \caption{\theprompt{The letter Z} \orbII}
        \end{subfigure}
        \begin{subfigure}[t]{0.196\linewidth}
        \includegraphics[width=\linewidth]{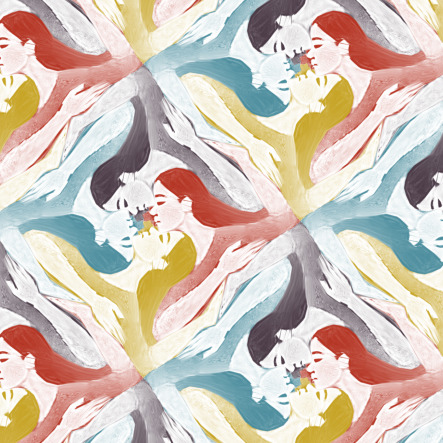}
        \caption{\theprompt{Yoga pose} \orbI}
        \end{subfigure}
        \begin{subfigure}[t]{0.196\linewidth}
        \includegraphics[width=\linewidth]{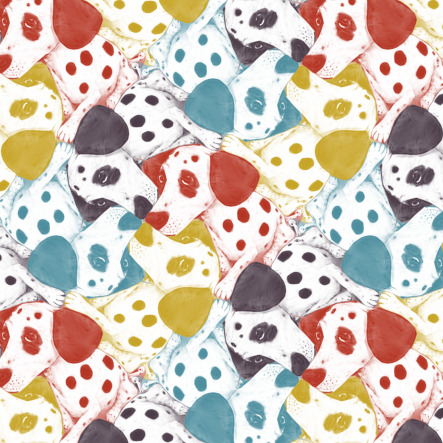}
        \caption{\theprompt{A dalmatian dog} \orbI}
        \end{subfigure}
\\
        \begin{subfigure}[t]{0.196\linewidth}
        \includegraphics[width=\linewidth]{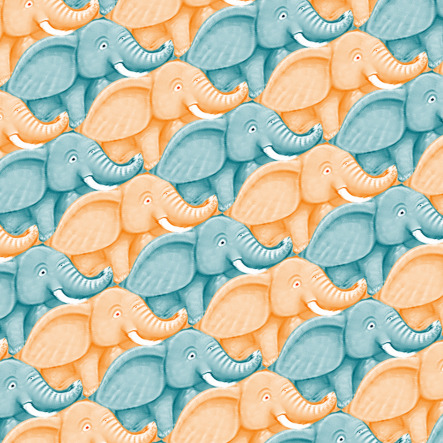}
        \caption{\theprompt{An elephant} \torus}
        \end{subfigure}
        \begin{subfigure}[t]{0.196\linewidth}
        \includegraphics[width=\linewidth]{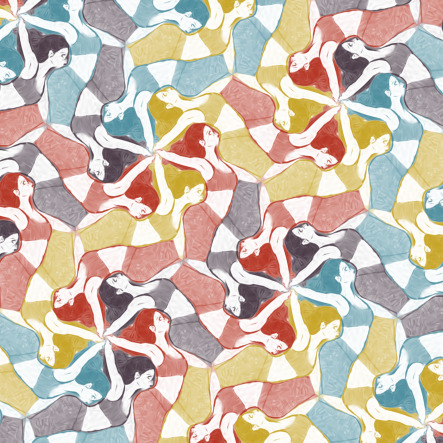}
        \caption{\theprompt{Yoga pose} \orbIII}
        \end{subfigure}
        \begin{subfigure}[t]{0.196\linewidth}
        \includegraphics[width=\linewidth]{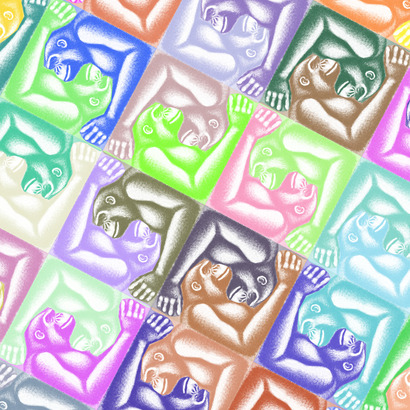}
        \caption{\theprompt{A gorilla} \orbRhyb}
        \end{subfigure}
        \begin{subfigure}[t]{0.196\linewidth}
        \includegraphics[width=\linewidth]{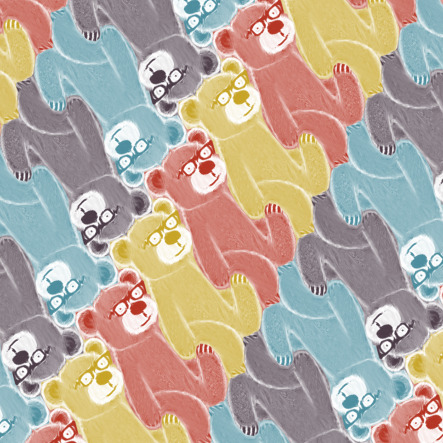}
        \caption{\theprompt{A nerdy bear} \orbIV}
        \end{subfigure}
        \begin{subfigure}[t]{0.196\linewidth}
        \includegraphics[width=\linewidth]{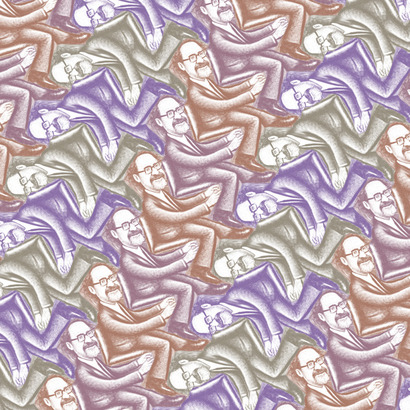}
        \caption{\theprompt{A math professor} \klein}
        \end{subfigure}
\\
        \begin{subfigure}[t]{0.196\linewidth}
        \includegraphics[width=\linewidth]{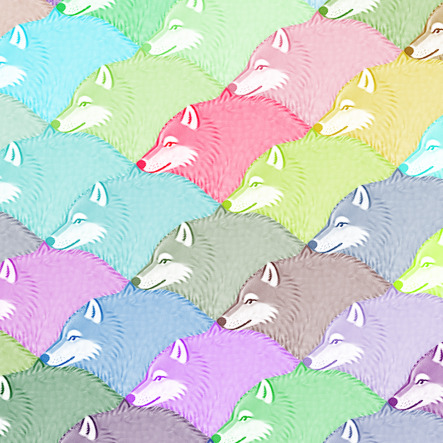}
        \caption{\theprompt{A wolf} \torus}
        \end{subfigure}
        \begin{subfigure}[t]{0.196\linewidth}
        \includegraphics[width=\linewidth]{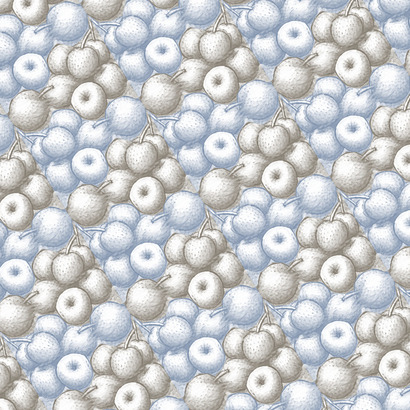}
        \caption{\theprompt{Fruits} \mob}
        \end{subfigure}
        \begin{subfigure}[t]{0.196\linewidth}
        \includegraphics[width=\linewidth]{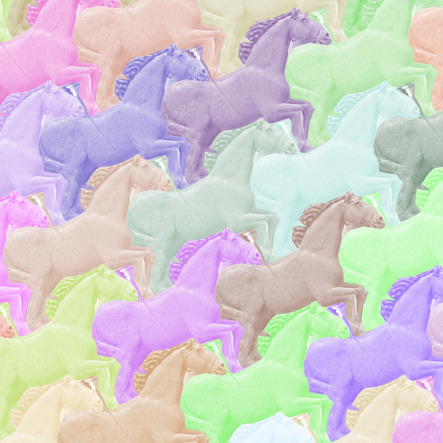}
        \caption{\theprompt{A horse} \torus}
        \end{subfigure}
        \begin{subfigure}[t]{0.196\linewidth}
        \includegraphics[width=\linewidth]{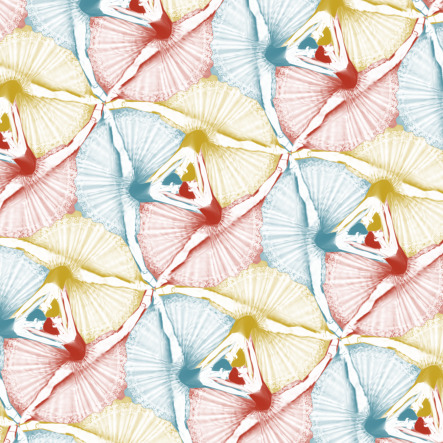}
        \caption{\theprompt{A ballet dancer} \orbII}
        \end{subfigure}
        \begin{subfigure}[t]{0.196\linewidth}
        \includegraphics[width=\linewidth]{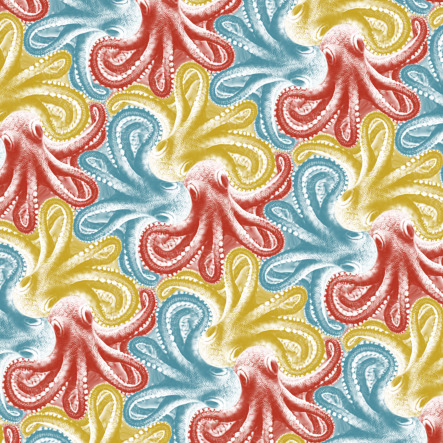}
        \caption{\theprompt{An octopus} \orbII}
        \end{subfigure}
\\
        \begin{subfigure}[t]{0.196\linewidth}
        \includegraphics[width=\linewidth]{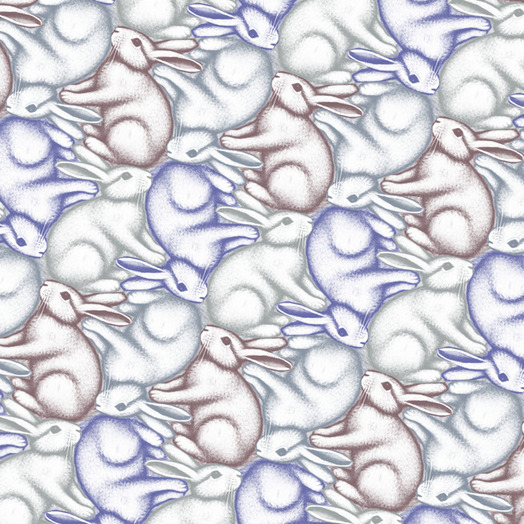}
        \caption{\theprompt{A rabbit} \projective}
        \end{subfigure}
        \begin{subfigure}[t]{0.196\linewidth}
        \includegraphics[width=\linewidth]{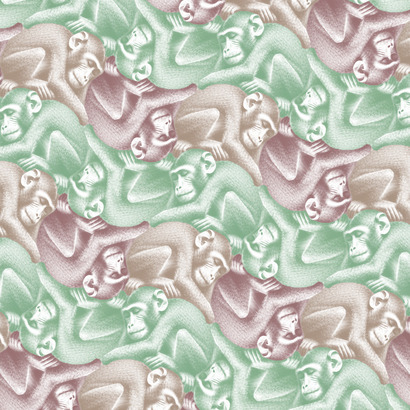}
        \caption{\theprompt{Sleeping monkey} \projective}
        \end{subfigure}
        \begin{subfigure}[t]{0.196\linewidth}
        \includegraphics[width=\linewidth]{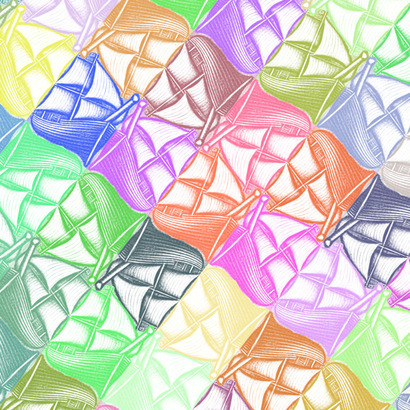}
        \caption{\theprompt{A  ship} \mob}
        \end{subfigure}
        \begin{subfigure}[t]{0.196\linewidth}
        \includegraphics[width=\linewidth]{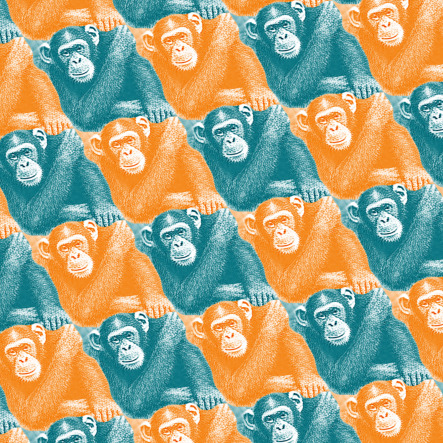}
        \caption{\theprompt{A chimpanzee} \torus}
        \end{subfigure}
        \begin{subfigure}[t]{0.196\linewidth}
        \includegraphics[width=\linewidth]{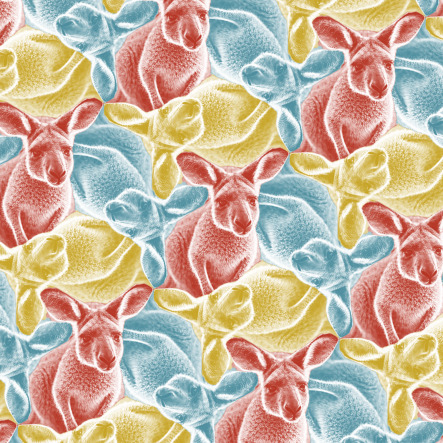}
        \caption{\theprompt{A kangaroo} \orbII}
        \end{subfigure}
\\
        \begin{subfigure}[t]{0.196\linewidth}
        \includegraphics[width=\linewidth]{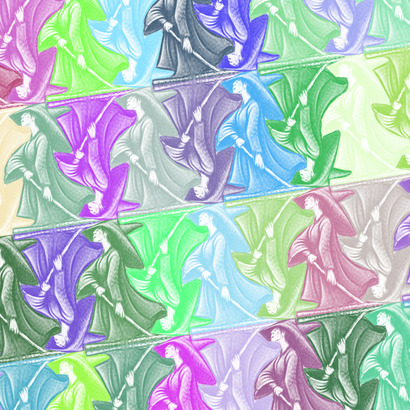}
        \caption{\theprompt{A witch} \mob}
        \end{subfigure}
        \begin{subfigure}[t]{0.196\linewidth}
        \includegraphics[width=\linewidth]{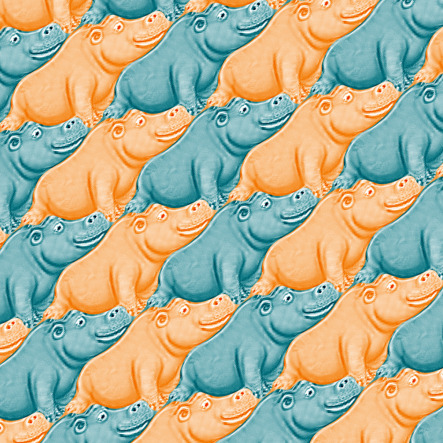}
        \caption{\theprompt{A hippopotamus} \torus}
        \end{subfigure}
        \begin{subfigure}[t]{0.196\linewidth}
        \includegraphics[width=\linewidth]{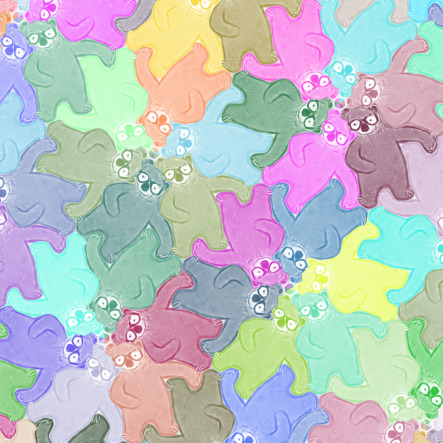}
        \caption{\theprompt{A nerdy bear} \orbIII}
        \end{subfigure}
        \begin{subfigure}[t]{0.196\linewidth}
        \includegraphics[width=\linewidth]{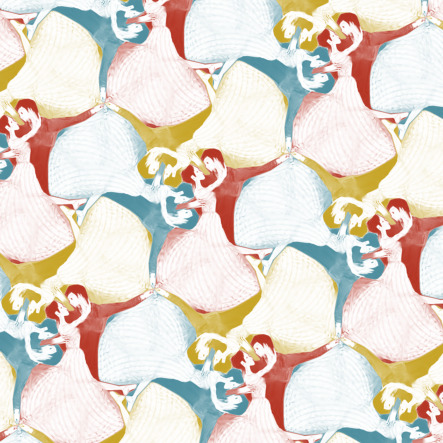}
        \caption{\theprompt{A dancing couple} \orbII}
        \end{subfigure}
        \begin{subfigure}[t]{0.196\linewidth}
        \includegraphics[width=\linewidth]{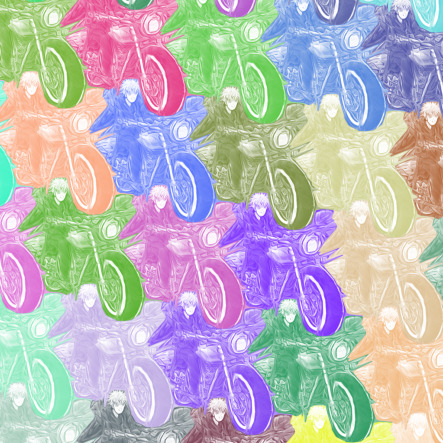}
        \caption{\theprompt{An anime biker} \torus}
        \end{subfigure}
\\
\captionsetup{justification=justified}
\vspace{-10pt}
    \caption{\textbf{Tiling Menagerie.} Examples of tilings produced by our method for different prompts and symmetries. Our method produces appealing, plausible results which contain solely the desired object, and cover the plane without overlaps.}
    \label{fig:gallery}
\end{figure*}

\clearpage
\appendix

\section{Full prompts from the paper}

\section*{Main:Teaser}
\begin{enumerate}
\item A beautiful illustration of a flamenco dancer
\item A beautiful botanical illustration of poppies, a masterpiece
\item A beautiful illustration of a ballet dancer, a masterpiece
\item A beautiful botanical illustration of a pinecone
\item A dolphin, best quality, extremely detailed
\item A heron, a masterpiece
\item A beautiful illustration of an alligator, a masterpiece
\item A children's book watercolor illustration of a fat unicorn, a masterpiece
\end{enumerate}

\section*{Main:Other interesting wallpaper groups}
\begin{enumerate}
\item A Godzilla, a masterpiece
\item A charcoal illustration of a  witch riding a broom, a masterpiece
\item A super long Dachshund, a masterpiece
\item A pirate, a masterpiece
\item A charcoal illustration of a sleeping monkey, a masterpiece
\end{enumerate}

\section*{Main:Multi-tile generation}
\begin{enumerate}
\item A beautiful illustration of a man doing yoga in a tree pose - a masterpiece, a beautiful illustration of a tree, a masterpiece
\item A beautiful illustration of a cat, a masterpiece - a beautiful illustration of an alligator, a masterpiece - a beautiful illustration of a monkey, a masterpiece - a beautiful illustration of an elephant, a masterpiece
\end{enumerate}

\section*{Main:Baseline}
\begin{enumerate}
\item A beautiful illustration of a penguin, a masterpiece.
\item A beautiful illustration of a lion, a masterpiece.
\item A beautiful illustration of a bat, a masterpiece.
\end{enumerate}

\section*{Main:Ablation}
\begin{enumerate}
\item A beautiful illustration of a Chinese dragon head.
\item A beautiful illustration of a pirate ship.
\item A beautiful illustration of a tractor.
\end{enumerate}

\section*{Main:RGB tilings}
\begin{enumerate}
\item A beautiful illustration of a rainbow, a masterpiece.
\item A beautiful illustration of an orange tree with fruits, a masterpiece.
\item A beautiful illustration of a giraffe, a masterpiece.
\end{enumerate}

\section*{Main:Fabrication}
\begin{enumerate}
\item A beautiful illustration of a flamenco dancer, a masterpiece.
\end{enumerate}

\section*{Main:Limitations}
\begin{enumerate}
\item A beautiful illustration of an open hand
\item A beautiful illustration of an espresso cup, a masterpiece.
\item A black ink illustration of an ant, 4k
\end{enumerate}

\section*{Main:Tiling Menagerie}
\begin{enumerate}
\item An illustration of a Puzzle piece
\item A beautiful illustration of a dragon
\item The letter Z
\item A beautiful illustration of a woman in a yoga pose
\item A beautiful illustration of a Dalmatian dog, in the style of a cartoon
\item A children's book illustration of an elephant
\item A beautiful illustration of a woman in a yoga pose
\item A gorilla, a masterpiece
\item A beautiful illustration of a nerdy bear
\item A math professor, a masterpiece
\item A children's book illustration of a wolf
\item A bunch of fruits, a masterpiece
\item A children's book illustration of a horse
\item A beautiful illustration of a ballet dancer
\item A photo of an octopus, 4K resolution
\item A beautiful watercolor painting of a rabbit, a masterpiece
\item A beautiful watercolor painting of a sleeping monkey, a masterpiece
\item A beautiful illustration of a sailing boat
\item Chimpanzee
\item A photograph of a kangaroo, 4K
\item A witch, a masterpice
\item Hippopotamus
\item A children's book illustration of a nerdy bear
\item A beautiful illustration of dancing couple
\item An anime biker
\end{enumerate}

\section*{Supmat: More Baseline}
\begin{enumerate}
\item A beautiful illustration of a nerdy bear, a masterpiece.
\item A beautiful illustration of Saturn with shooting stars, a masterpiece.
\end{enumerate}

\section*{Supmat:RGB tilings}
\begin{enumerate}
\item A beautiful illustration of a cherry flower, a masterpiece.
\item A beautiful illustration of a wrestler, a masterpiece.
\end{enumerate}

\section*{Supmat:Crystals generated}
\begin{enumerate}
\item A beautiful illustration of a snowflake
\end{enumerate}

\section*{Supmat:Multi-tile generation}
\begin{enumerate}
\item A beautiful illustration of a scuba diver, a masterpiece - a beautiful illustration of a tropical fish, a masterpiece
\item A beautiful illustration of a man doing yoga in a tree pose - a masterpiece, a beautiful illustration of a rose, a masterpiece - a masterpiece, a beautiful illustration of a tulip, a masterpiece - a masterpiece, a beautiful illustration of a sunflower, a masterpiece
\end{enumerate}

\section*{Supmat:All wallpapers}
\begin{enumerate}
\item A professional cartoon of a monkey, a masterpiece
\item A professional cartoon of nerdy bear, a masterpiece
\item A professional cartoon of an elephant, a masterpiece
\end{enumerate}

\section*{Supmat:Tiling Menagerie 2}
\begin{enumerate}
\item A children's book illustration of a sardine
\item A children's book illustration of a sardine
\item A children's book illustration of a sardine
\item A beautiful botanical illustration of cherries, a masterpiece
\item A beautiful botanical illustration of a tropical flower, a masterpiece
\item A beautiful watercolor painting of a coyfish, a masterpiece
\item A beautiful illustration of a wave, Hokusai, a masterpiece
\item A children's book illustration of a mother cat nursing her kittens
\item A beautiful illustration of a sailing boat, a masterpiece
\item A beautiful illustration of corncobs, a masterpiece
\item A centaur in the style of ancient Greece
\item A beautiful illustration of a dragon, a masterpiece
\item An anchor in the style of a sailor tattoo
\item A beautiful watercolor illustration of a Scuba Diver with flippers and oxygen tanks, a masterpiece
\item A penguin
\item A beautiful illustration of a woman in a yoga pose, a masterpiece
\item A beautiful illustration of a woman in a yoga pose, a masterpiece
\item An anime biker
\item A beautiful illustration of a guitar, a masterpiece
\item A children's book illustration of a bat
\item A comic book illustration of Saturn with shooting stars
\item A cow, best quality, extremely detailed
\item A children's book illustration of a nerdy bear
\item A children's book illustration of a nerdy bear
\item A children's book illustration of a nerdy bear
\end{enumerate}

\section*{Supmat:Tiling Menagerie 3}
\begin{enumerate}
    \item A photo of an octopus, 4K resolution, a masterpiece
    \item A beautiful illustration of a ballet dancer, a masterpiece
    \item A photograph of a cherry
    \item Crocodile
    \item A beautiful illustration of a tree, a masterpiece
    \item A beautiful illustration of a bicycle, a Penny-farthing, a masterpiece
    \item An illustration of an oyster with a pearl inside, a masterpiece
    \item A beautiful black and white ink hatchings illustration of a staircase, a masterpiece
    \item A centaur in the style of ancient Greece
    \item A cartoon musical note
    \item Peacock feather
    \item A beautiful illustration of a dancing couple, a masterpiece
    \item A beautiful botanical illustration of a tropical flower, a masterpiece
    \item A photograph of a kangaroo, 4K, a masterpiece
    \item A beautiful botanical illustration of a pineapple, 4k, a masterpiece
    \item A beautiful illustration of a sunflower, a masterpiece
    \item A beautiful fluffy cartoon cloud
    \item An illustration of a hermit crab for a
    \item An illustration of a sorcerer, intricate, elegant, highly detailed, lifelike, photorealistic digital painting, ArtStation
    \item Pyramids of Giza 
    \item A lion, best quality, extremely detailed
    \item A children's book illustration of a buffalo
    \item A 3D rendering of a low-poly gecko with legs spread, 4K, a masterpiece
    \item A children's book illustration of a mother cat nursing her kittens
    \item An illustration of a realistic traditional Boomerang, clean lines, a masterpiece
\end{enumerate}

\section*{Supmat:Tiling Menagerie 4}
\begin{enumerate}
        \item A children's book watercolor illustration of a fat unicorn, a masterpiece
        \item A heron, a masterpiece
        \item A professional photo of a submarine, a masterpiece
        \item A charcoal illustration of a snowy cabin, a masterpiece
        \item A professional photo of a funky gorilla, a masterpiece
        \item A children's book watercolor illustration of a bunch of fruits, a masterpiece
        \item A children's book watercolor illustration of a funky gorilla, a masterpiece
        \item A charcoal illustration of a Godzilla, a masterpiece
        \item A Godzilla, a masterpiece
        \item A charcoal illustration of a heron, a masterpiece
        \item A beautiful illustration of a Godzilla, a masterpiece
        \item A children's book watercolor illustration of a crazy math professor, a masterpiece
        \item A children's book watercolor illustration of a sleeping monkey, a masterpiece
        \item A professional photo of a pirate, a masterpiece
        \item A charcoal illustration of a super long Dachshund, a masterpiece
        \item A super long Dachshund, a masterpiece
        \item A charcoal illustration of a heron, a masterpiece
        \item A charcoal illustration of a bunch of fruits, a masterpiece
        \item A children's book watercolor illustration of a snowy cabin, a masterpiece
        \item A charcoal illustration of a hippie rabbit, a masterpiece
        \item A hippie rabbit, a masterpiece
        \item A charcoal illustration of a sleeping monkey, a masterpiece
        \item A fat unicorn, a masterpiece
        \item A beaver, a masterpiece
        \item A professional photo of a piano, a masterpiece
\end{enumerate}

\section*{Supmat:Tiling Menagerie 5}
\begin{enumerate}
        \item A charcoal illustration of a dancing Skeletton, a masterpiece
        \item A charcoal illustration of a dancing Skeletton, a masterpiece
        \item A children's book watercolor illustration of a beaver, a masterpiece
        \item A charcoal illustration of a heron, a masterpiece
        \item A funky gorilla, a masterpiece
        \item A charcoal illustration of a bunch of fruits, a masterpiece
        \item A charcoal illustration of a hippie rabbit, a masterpiece
        \item A charcoal illustration of a crazy math professor, a masterpiece
        \item A children's book watercolor illustration of a skull, a masterpiece
        \item A children's book watercolor illustration of a sleeping monkey, a masterpiece
        \item A beaver, a masterpiece
        \item A charcoal illustration of a  witch riding a broom, a masterpiece
        \item A children's book watercolor illustration of a crazy math professor, a masterpiece
        \item A children's book watercolor illustration of a bunch of fruits, a masterpiece
        \item A children's book watercolor illustration of a sleeping monkey, a masterpiece
        \item A charcoal illustration of a beaver, a masterpiece
        \item A charcoal illustration of a  witch riding a broom, a masterpiece
        \item A children's book watercolor illustration of a  witch riding a broom, a masterpiece
        \item A witch riding a broom, a masterpiece
        \item A charcoal illustration of a bunch of fruits, a masterpiece
        \item A charcoal illustration of a funky gorilla, a masterpiece
        \item A children's book watercolor illustration of a Godzilla, a masterpiece
        \item A charcoal illustration of a sleeping monkey, a masterpiece
        \item A charcoal illustration of a beaver, a masterpiece
        \item A heron, a masterpiece
\end{enumerate}
\section*{Supmat:Tiling Menagerie 6}
\begin{enumerate}
        \item A children's book watercolor illustration of a flamingo, a masterpiece
        \item A professional photo of a pirate, a masterpiece
        \item A witch riding a broom, a masterpiece
        \item A charcoal illustration of a  witch riding a broom, a masterpiece
        \item A children's book watercolor illustration of a nastronaut, a masterpiece
        \item A charcoal illustration of a hippie rabbit, a masterpiece
        \item A children's book watercolor illustration of a hippie rabbit, a masterpiece
        \item A hippie rabbit, a masterpiece
        \item A professional photo of a hippie rabbit, a masterpiece
        \item A charcoal illustration of a super long Dachshund, a masterpiece
        \item A children's book watercolor illustration of a snowy cabin, a masterpiece
        \item A children's book watercolor illustration of a sleeping monkey, a masterpiece
        \item A beaver, a masterpiece
        \item A children's book watercolor illustration of a beaver, a masterpiece
        \item A children's book watercolor illustration of a  witch riding a broom, a masterpiece
        \item A charcoal illustration of a pirate ship, a masterpiece
        \item A witch riding a broom, a masterpiece
        \item A fat unicorn, a masterpiece
        \item A children's book watercolor illustration of a beaver, a masterpiece
        \item A beaver, a masterpiece
        \item A charcoal illustration of a heron, a masterpiece
        \item A professional photo of a heron, a masterpiece
        \item A witch riding a broom, a masterpiece
        \item A charcoal illustration of a  witch riding a broom, a masterpiece
        \item A funky gorilla, a masterpiece
\end{enumerate}
\section*{Supmat:Tiling Menagerie 7}
\begin{enumerate}
        \item A charcoal illustration of a crazy math professor, a masterpiece
        \item A hippie rabbit, a masterpiece
        \item A children's book watercolor illustration of a snowy cabin, a masterpiece
        \item A children's book watercolor illustration of a super long Dachshund, a masterpiece
        \item A professional photo of a snowy cabin, a masterpiece
        \item A skull, a masterpiece
        \item A sleeping monkey, a masterpiece
        \item A professional photo of a sleeping monkey, a masterpiece
        \item A professional photo of a beaver, a masterpiece
        \item A beautiful illustration of a funky gorilla, a masterpiece
        \item A witch riding a broom, a masterpiece
        \item A charcoal illustration of a piano, a masterpiece
        \item A charcoal illustration of a pirate, a masterpiece
        \item A children's book watercolor illustration of a bunch of fruits, a masterpiece
        \item A charcoal illustration of a beaver, a masterpiece
        \item A pirate, a masterpiece
        \item A charcoal illustration of a pirate ship, a masterpiece
        \item A witch riding a broom, a masterpiece
        \item A charcoal illustration of a  witch riding a broom, a masterpiece
        \item A charcoal illustration of a funky gorilla, a masterpiece
        \item A charcoal illustration of a nastronaut, a masterpiece
        \item A charcoal illustration of a snowy cabin, a masterpiece
        \item A professional photo of a snowy cabin, a masterpiece
        \item A professional photo of a dancing Skeletton, a masterpiece        
        \item A sleeping monkey, a masterpiece
\end{enumerate}

\section*{Comparisons: ``prompt engineering'' attempts}
(Replace $X$ with the desired object)
\begin{enumerate}
\item a $X$
\item a children's book illustration of a $X$
\item An Escher drawing of a children's book illustration of a $X$
\item An Escher tiling of a children's book illustration of a $X$
\item An Escher pattern of a children's book illustration of a $X$
\item An Escherization of a children's book illustration of a $X$
\item An Escher-like drawing of a children's book illustration of a $X$
\item An Escher-like tiling of a children's book illustration of a $X$
\item An Escher-like pattern of a children's book illustration of a $X$
\item An Escher-like image of a children's book illustration of a $X$
\item An Escher-like picture of a children's book illustration of a $X$
\item a children's book illustration of a $X$ in the style of Escher
\item An image composed of the repetition of a children's book illustration of a $X$
\item An image composed of the repetition of a children's book illustration of a $X$just like Escher drawings in the metamorphosis series.
\item An image composed of the repetition of a children's book illustration of a  $X$just like Escher drawings.
\item An image composed of the repetition of a children's book illustration of a  $X$ in the style of Escher
\item An image composed of the repetition of a children's book illustration of a  $X$ in the style of Escher. There should be no gaps nor overlaps.
\item An image composed of the repetition of a children's book illustration of a $X$. There should be no gaps nor overlaps.
\item An image composed of the repetition of a children's book illustration of a $X$. There should be no gaps nor overlaps. The object is always the same in all repetitions
\item An image made by repeating a children's book illustration of a $X$ with translations in such a way that there is no gaps nor overlaps. It is always the same object.
\item An image made by repeating a children's book illustration of a $X$ with translations in such a way that there is no gaps nor overlaps. It is always the same object. The style is Escher-like.
\item An Escher drawing of a $X$
\item An Escher tiling of a $X$
\item An Escher pattern of a $X$
\item An Escherization of a $X$
\item An Escher-like drawing of a $X$
\item An Escher-like tiling of a $X$
\item An Escher-like pattern of a $X$
\item An Escher-like image of a $X$
\item An Escher-like picture of a $X$
\item a $X$ in the style of Escher
\item An image composed of the repetition of a $X$
\item An image composed of the repetition of a $X$ just like Escher drawings in the metamorphosis series.
\item An image composed of the repetition of a $X$ just like Escher drawings.
\item An image composed of the repetition of a $X$ in the style of Escher
\item An image composed of the repetition of a $X$ in the style of Escher. There should be no gaps nor overlaps.
\item An image composed of the repetition of a $X$. There should be no gaps nor overlaps.
\item An image composed of the repetition of a $X$. There should be no gaps nor overlaps. The object is always the same in all repetitions
\item An image made by repeating a $X$ with translations in such a way that there is no gaps nor overlaps. It is always the same object.
\item An image made by repeating a $X$ with translations in such a way that there is no gaps nor overlaps. It is always the same object. The style is Escher-like. 
\end{enumerate}

\begin{figure}[h]
    \centering
    \captionsetup{justification=centering}
        \begin{subfigure}[t]{0.32\linewidth}
        \includegraphics[width=\linewidth]{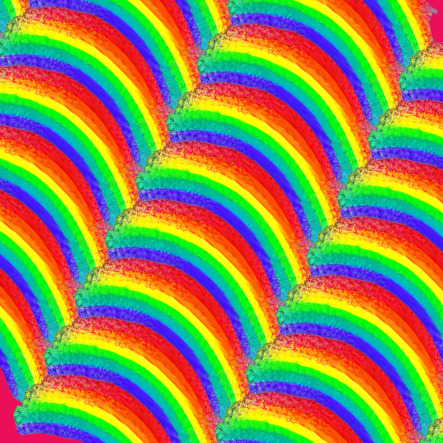}
        \caption{\theprompt{A rainbow} \torus}
        \end{subfigure}
        \begin{subfigure}[t]{0.32\linewidth}
        \includegraphics[width=\linewidth]{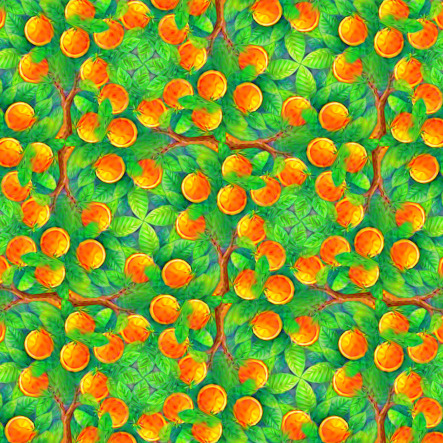}
        \caption{\theprompt{Orange tree} \orbI}
        \end{subfigure}
        \begin{subfigure}[t]{0.32\linewidth}
        \includegraphics[width=\linewidth]{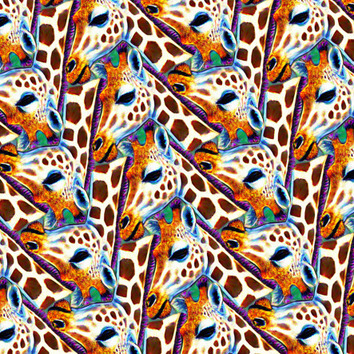}
        \caption{\theprompt{A giraffe} \mob}
        \end{subfigure}
\\
\captionsetup{justification=justified}
\vspace{-10pt}
    \caption{\textbf{RGB tilings.} Our method supports RGB-color textures out of the box, however it suffers from Score Distillation's saturated colors, which in some cases leads to uncompelling results as well as degradation in plausibility.}
    \label{fig:color_gallery}
\end{figure}
\begin{figure*}[h]
        \includegraphics[width=\textwidth]{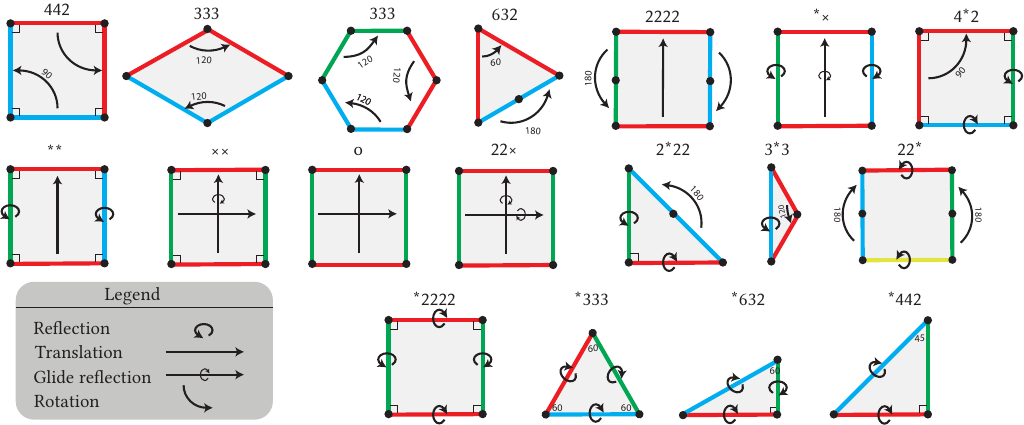}
    \caption{The 17 wallpaper groups, each represented through its ``basic tile''. The diagrams illustrate the relations between different parts of the boundary of the tile - a rotation, translation, rotation, or glide reflection (translation composed with a reflection) - which describe how to shift copies of the tile and align them in order to tile the plane. In order to infer the boundary conditions used in Equation (1) in the paper, one needs to place a point on one side of the boundary, and see where the transformation associated with that point on the boundary maps that point - this yields the appropriate linear equation to assign to that vertex and its corresponding ``twin'' vertex on the other side of the boundary. Each wallpaper group is marked with its corresponding \emph{orbifold notation} (see~\cite{conway2008symmetries}), written above the tile. }
\end{figure*}

\begin{figure*}
    \centering
    \captionsetup{justification=centering}
        \begin{subfigure}[b]{0.23\linewidth}
        \includegraphics[width=\linewidth]{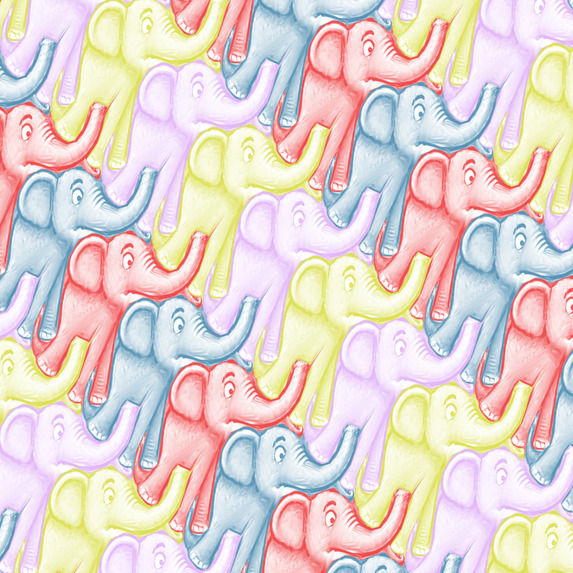}
        \end{subfigure}
        \begin{subfigure}[b]{0.23\linewidth}
        \includegraphics[width=\linewidth]{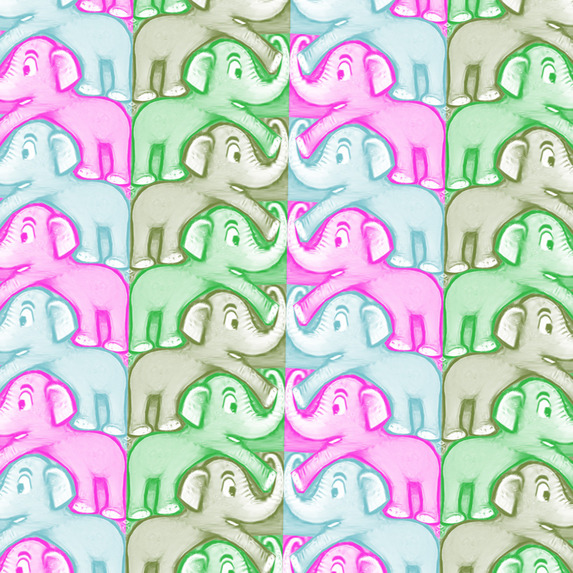}
        \end{subfigure}
        \begin{subfigure}[b]{0.23\linewidth}
        \includegraphics[width=\linewidth]{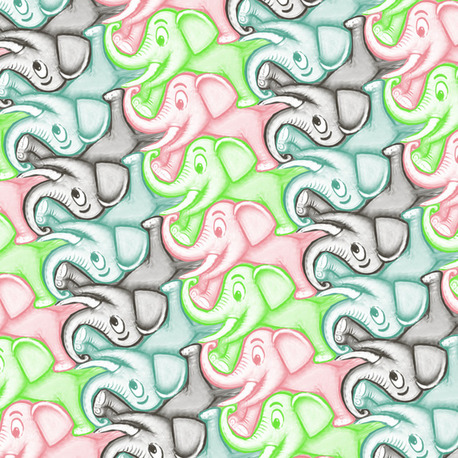}
        \end{subfigure}
        \begin{subfigure}[b]{0.23\linewidth}
        \includegraphics[width=\linewidth]{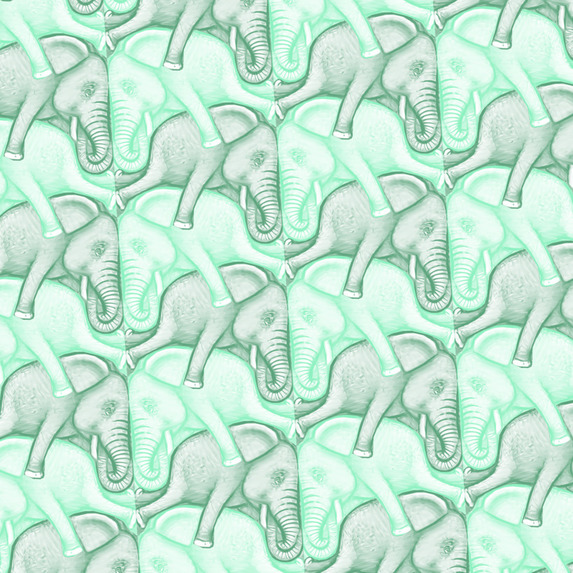}
        \end{subfigure}
\\
        \begin{subfigure}[b]{0.23\linewidth}
        \includegraphics[width=\linewidth]{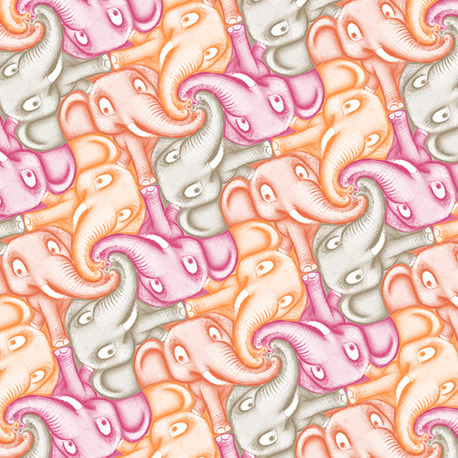}
        \end{subfigure}
        \begin{subfigure}[b]{0.23\linewidth}
        \includegraphics[width=\linewidth]{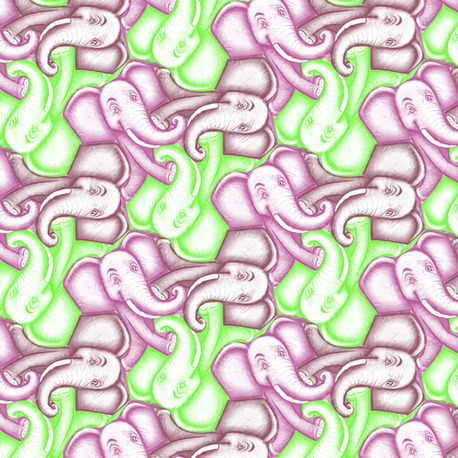}
        \end{subfigure}
        \begin{subfigure}[b]{0.23\linewidth}
        \includegraphics[width=\linewidth]{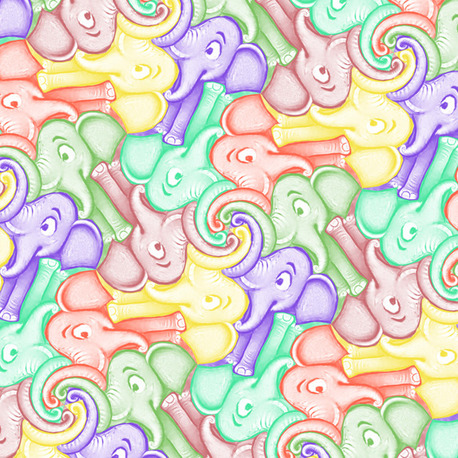}
        \end{subfigure}
\\
        \begin{subfigure}[b]{0.23\linewidth}
        \includegraphics[width=\linewidth]{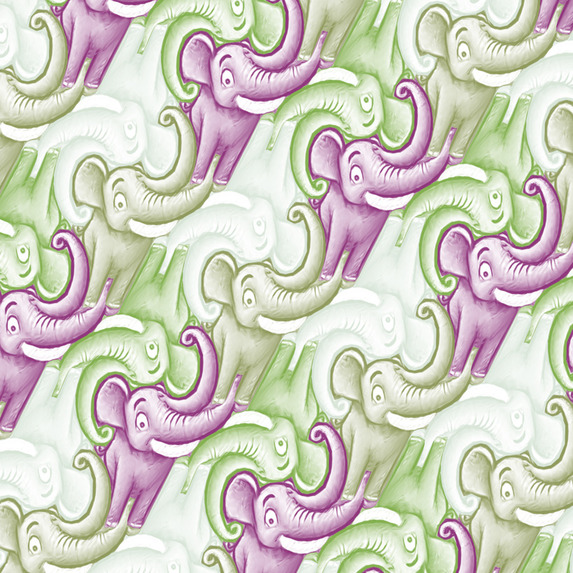}
        \end{subfigure}
        \begin{subfigure}[b]{0.23\linewidth}
        \includegraphics[width=\linewidth]{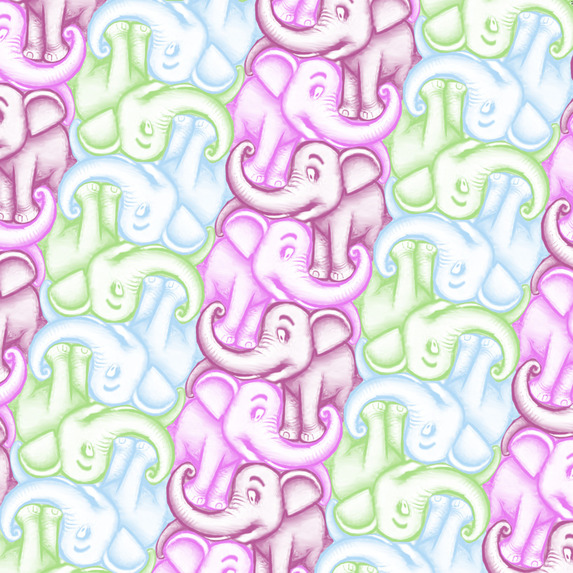}
        \end{subfigure}
        \begin{subfigure}[b]{0.23\linewidth}
        \includegraphics[width=\linewidth]{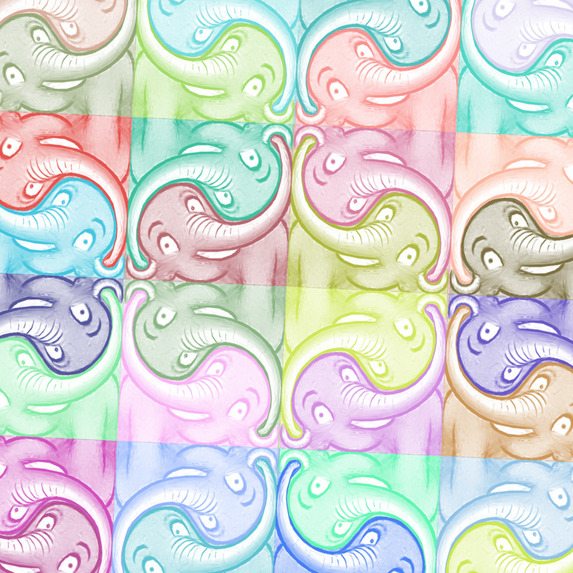}
        \end{subfigure}
\\
        \begin{subfigure}[b]{0.23\linewidth}
        \includegraphics[width=\linewidth]{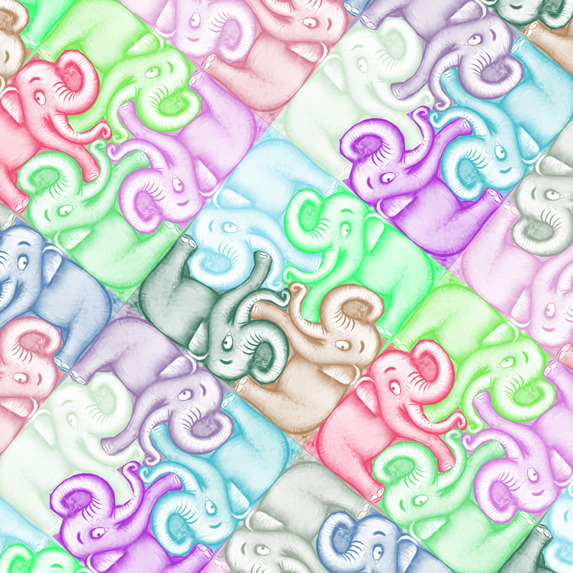}
        \end{subfigure}
        \begin{subfigure}[b]{0.23\linewidth}
        \includegraphics[width=\linewidth]{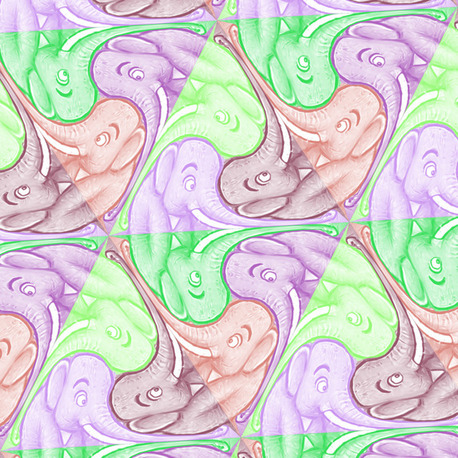}
        \end{subfigure}
        \begin{subfigure}[b]{0.23\linewidth}
        \includegraphics[width=\linewidth]{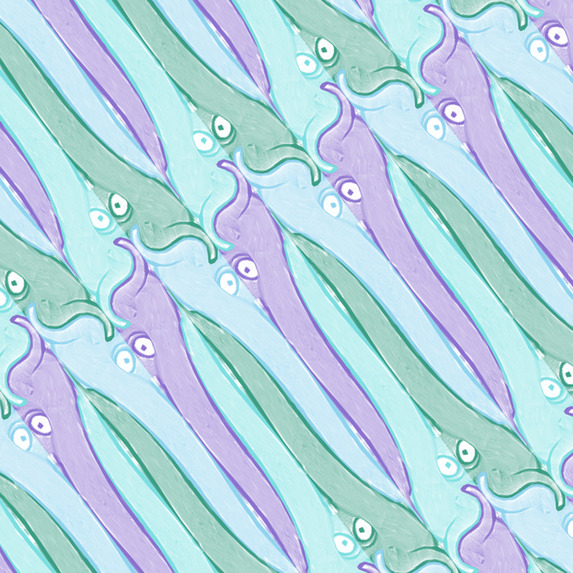}
        \end{subfigure}
\\
        \begin{subfigure}[b]{0.23\linewidth}
        \includegraphics[width=\linewidth]{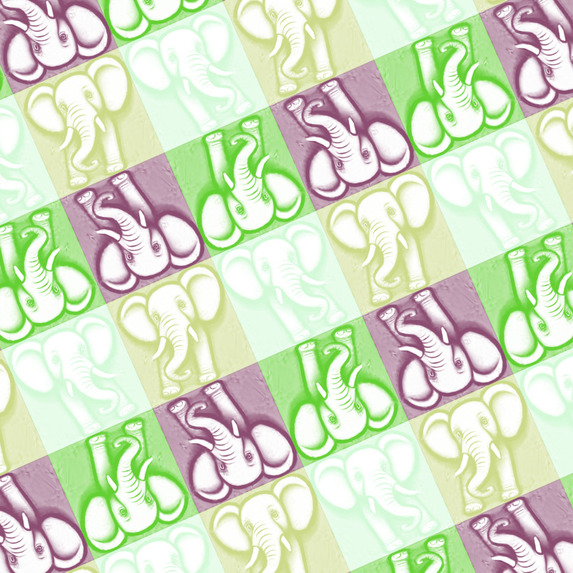}
        \end{subfigure}
        \begin{subfigure}[b]{0.23\linewidth}
        \includegraphics[width=\linewidth]{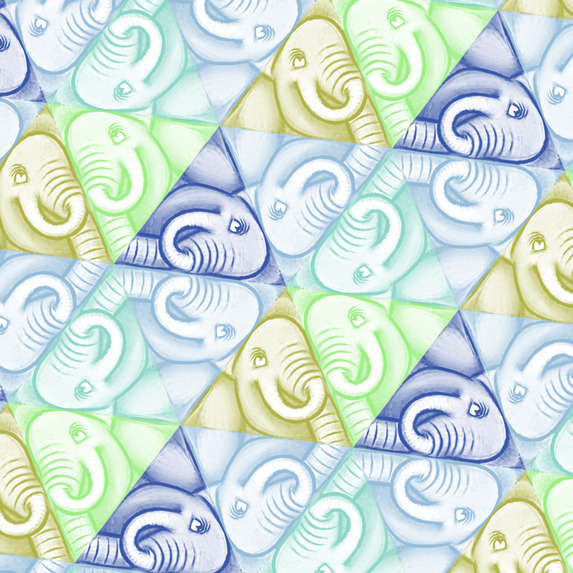}
        \end{subfigure}
        \begin{subfigure}[b]{0.23\linewidth}
        \includegraphics[width=\linewidth]{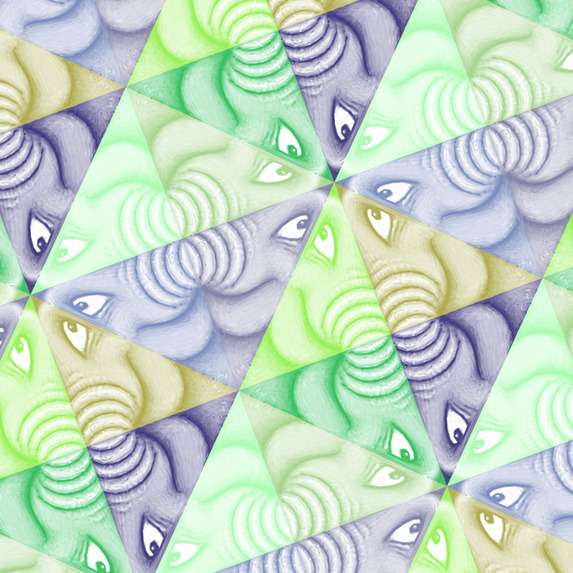}
        \end{subfigure}
        \begin{subfigure}[b]{0.23\linewidth}
        \includegraphics[width=\linewidth]{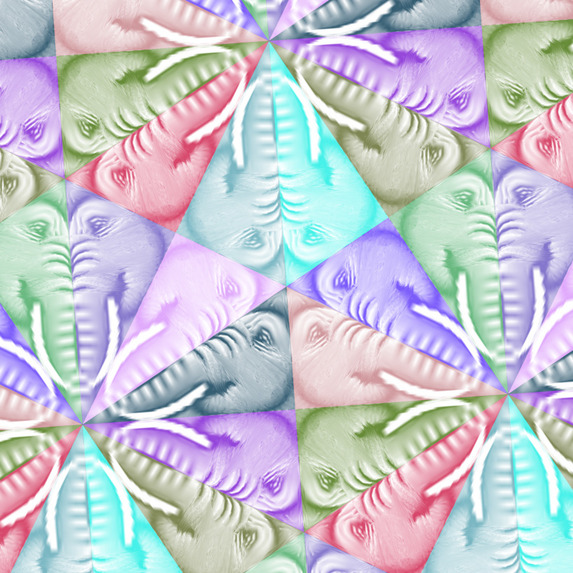}
        \end{subfigure}
\\
\captionsetup{justification=justified}
    \vspace*{-2mm}
    \caption{\textbf{All wallpaper groups for same prompt.} We showcase the 17 wallpaper groups for the prompt ``An elephant''. Note that the last four groups (last row) produce uninteresting results in terms of shape because their symmetry group is made out of reflection, forcing their boundary to stay a static convex polygon (either square or triangle), preventing the generation of different tile-shapes. Order: \torus, \cylinder, \klein, \mob, \orbI, \orbII, \orbIII, \orbIV, \projective, \orbRhyb, \orbIhyb, \orbIIhyb, \orbIVhyb,   \reflectII, \reflectI, \reflectIII, \reflectIV .}
    \label{fig:wallpaper}
\end{figure*}
\begin{figure*}
    \centering
    \captionsetup{justification=centering}
        \begin{subfigure}[b]{0.23\linewidth}
        \includegraphics[width=\linewidth]{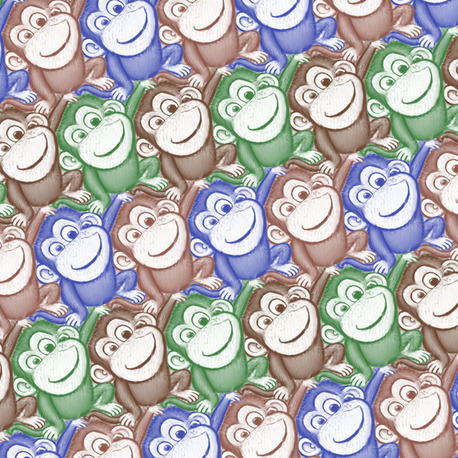}
        \end{subfigure}
        \begin{subfigure}[b]{0.23\linewidth}
        \includegraphics[width=\linewidth]{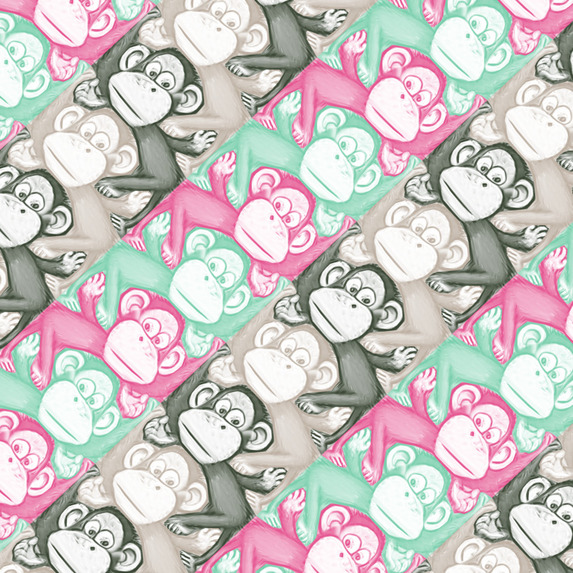}
        \end{subfigure}
        \begin{subfigure}[b]{0.23\linewidth}
        \includegraphics[width=\linewidth]{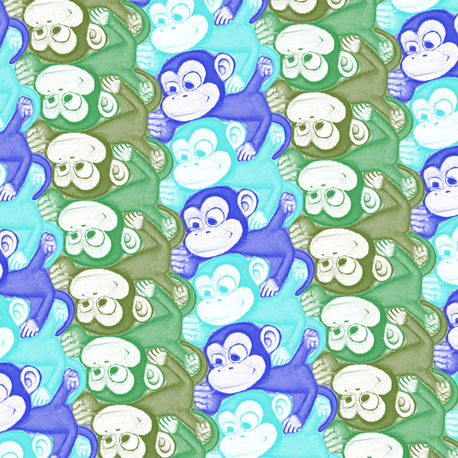}
        \end{subfigure}
        \begin{subfigure}[b]{0.23\linewidth}
        \includegraphics[width=\linewidth]{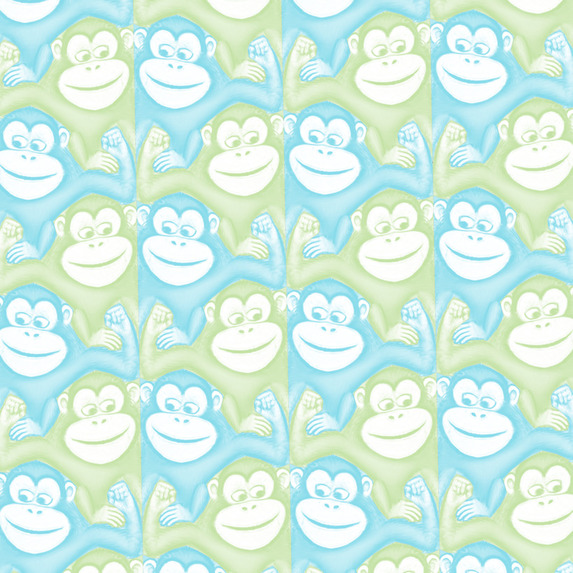}
        \end{subfigure}
\\
        \begin{subfigure}[b]{0.23\linewidth}
        \includegraphics[width=\linewidth]{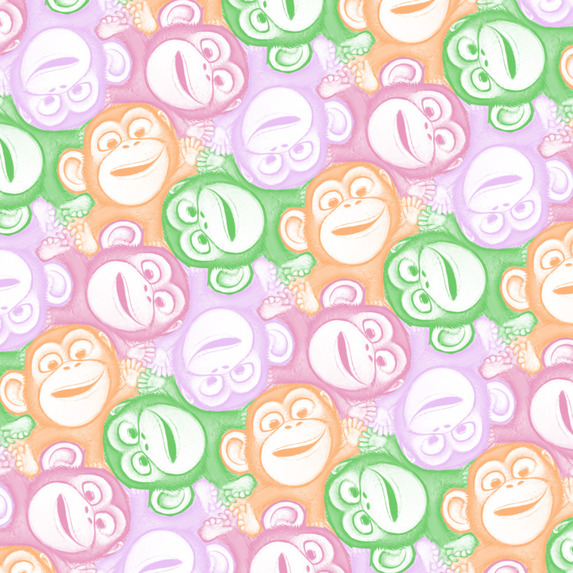}
        \end{subfigure}
        \begin{subfigure}[b]{0.23\linewidth}
        \includegraphics[width=\linewidth]{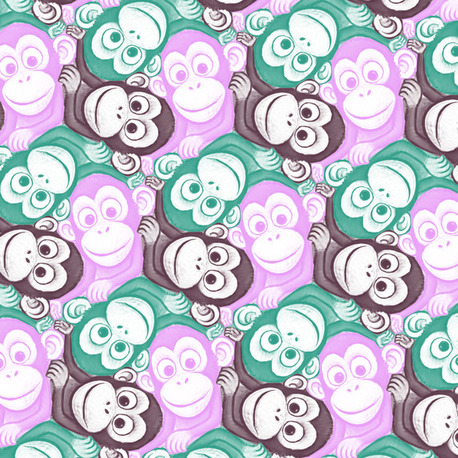}
        \end{subfigure}
        \begin{subfigure}[b]{0.23\linewidth}
        \includegraphics[width=\linewidth]{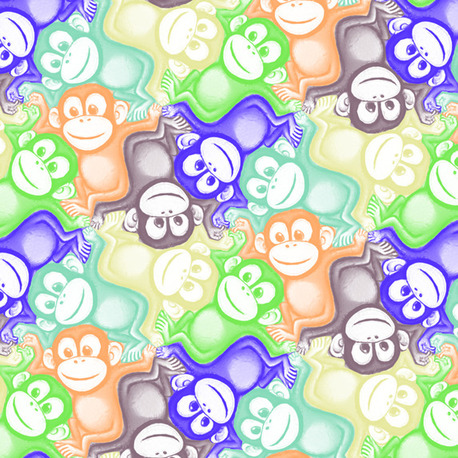}
        \end{subfigure}
\\
        \begin{subfigure}[b]{0.23\linewidth}
        \includegraphics[width=\linewidth]{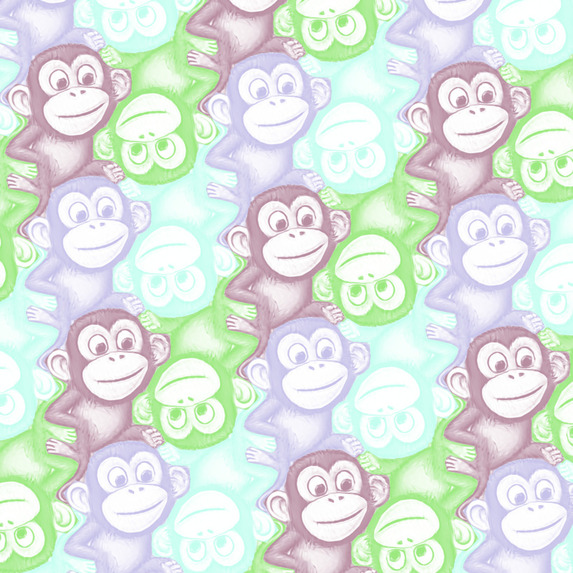}
        \end{subfigure}
        \begin{subfigure}[b]{0.23\linewidth}
        \includegraphics[width=\linewidth]{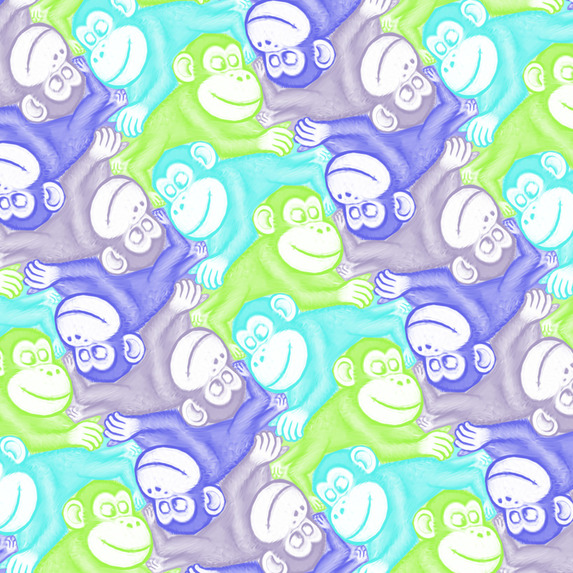}
        \end{subfigure}
        \begin{subfigure}[b]{0.23\linewidth}
        \includegraphics[width=\linewidth]{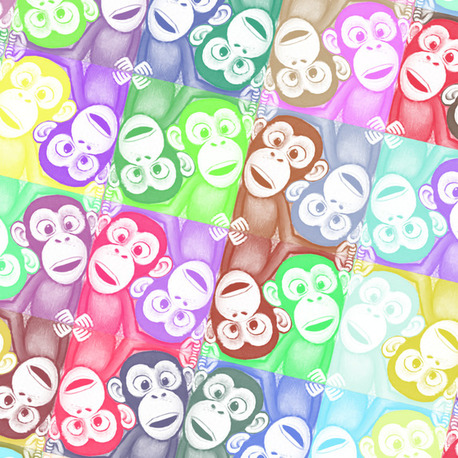}
        \end{subfigure}
\\
        \begin{subfigure}[b]{0.23\linewidth}
        \includegraphics[width=\linewidth]{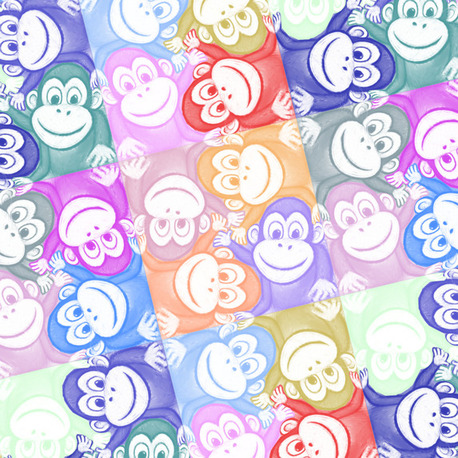}
        \end{subfigure}
        \begin{subfigure}[b]{0.23\linewidth}
        \includegraphics[width=\linewidth]{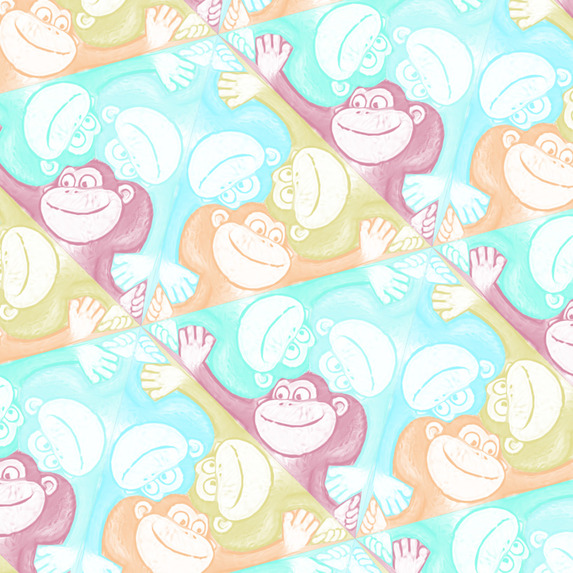}
        \end{subfigure}
        \begin{subfigure}[b]{0.23\linewidth}
        \includegraphics[width=\linewidth]{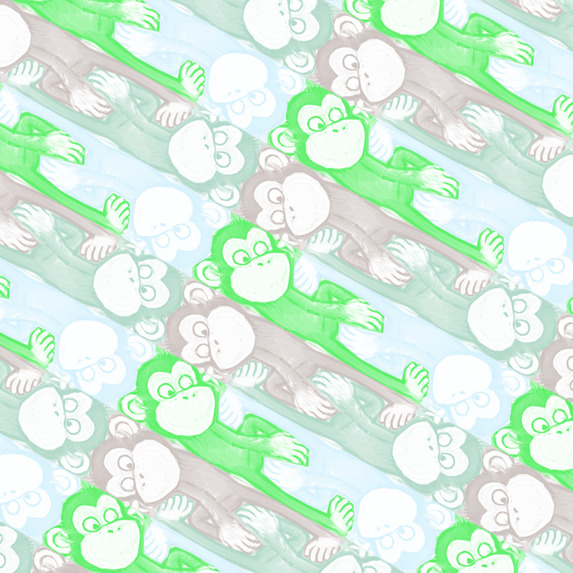}
        \end{subfigure}
\\
        \begin{subfigure}[b]{0.23\linewidth}
        \includegraphics[width=\linewidth]{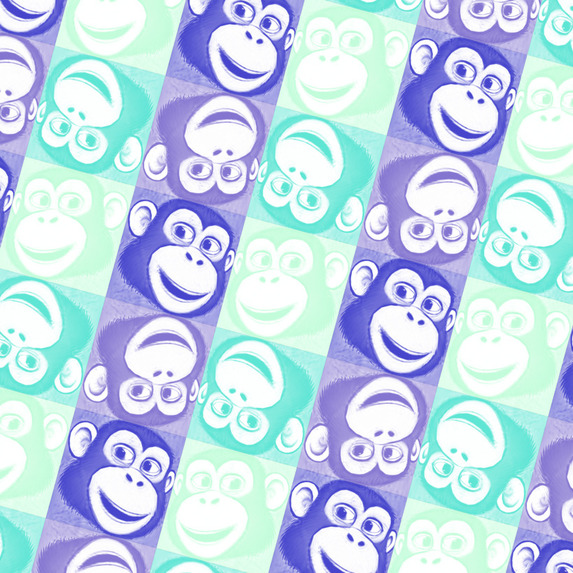}
        \end{subfigure}
        \begin{subfigure}[b]{0.23\linewidth}
        \includegraphics[width=\linewidth]{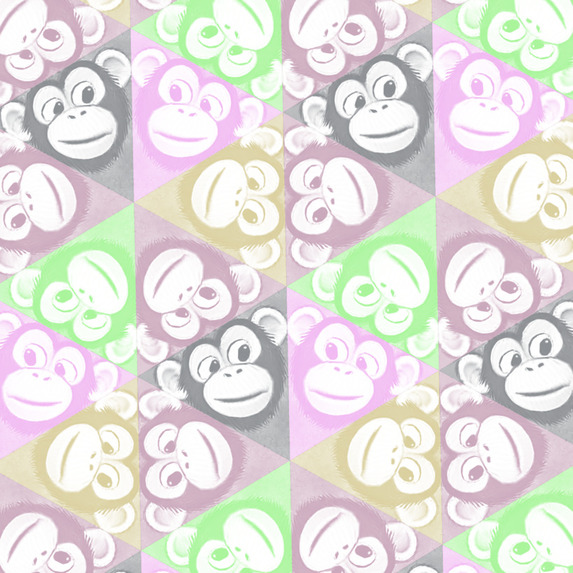}
        \end{subfigure}
        \begin{subfigure}[b]{0.23\linewidth}
        \includegraphics[width=\linewidth]{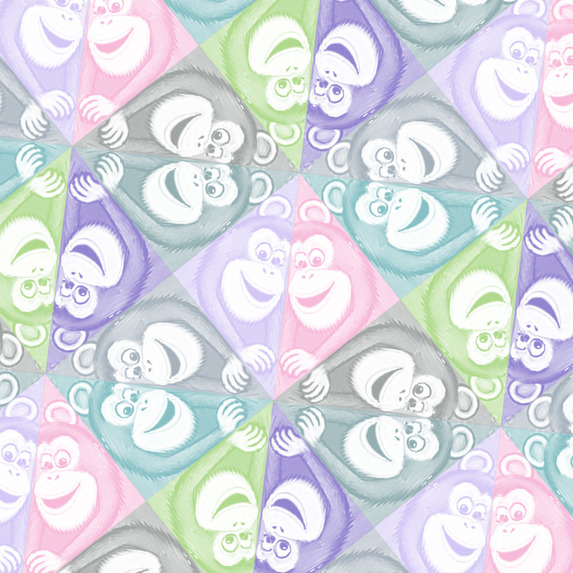}
        \end{subfigure}
        \begin{subfigure}[b]{0.23\linewidth}
        \includegraphics[width=\linewidth]{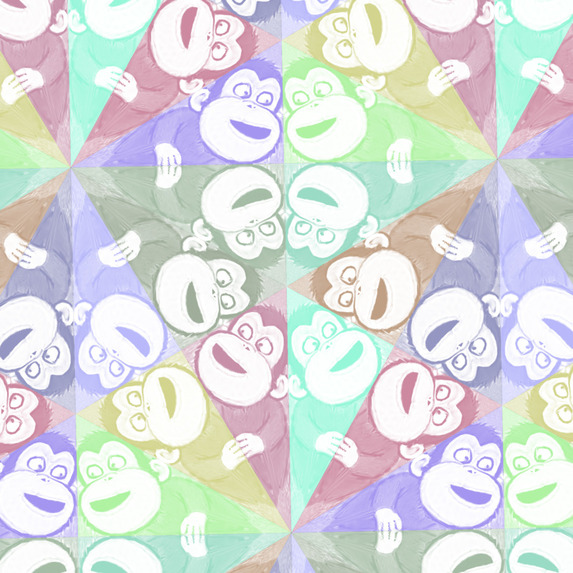}
        \end{subfigure}
\\
\captionsetup{justification=justified}
    \vspace*{-2mm}
    \caption{\textbf{All wallpaper groups for same prompt.} We showcase the 17 wallpaper groups for the prompt ``A monkey''. Note that the last four groups (last row) produce uninteresting results in terms of shape because their symmetry group is made out of reflection, forcing their boundary to stay a static convex polygon (either square or triangle), preventing the generation of different tile-shapes. Order: \torus, \cylinder, \klein, \mob, \orbI, \orbII, \orbIII, \orbIV, \projective, \orbRhyb, \orbIhyb, \orbIIhyb, \orbIVhyb,   \reflectII, \reflectI, \reflectIII, \reflectIV .}
    \label{fig:wallpaper}
\end{figure*}
\begin{figure*}
    \centering
    \captionsetup{justification=centering}
        \begin{subfigure}[b]{0.23\linewidth}
        \includegraphics[width=\linewidth]{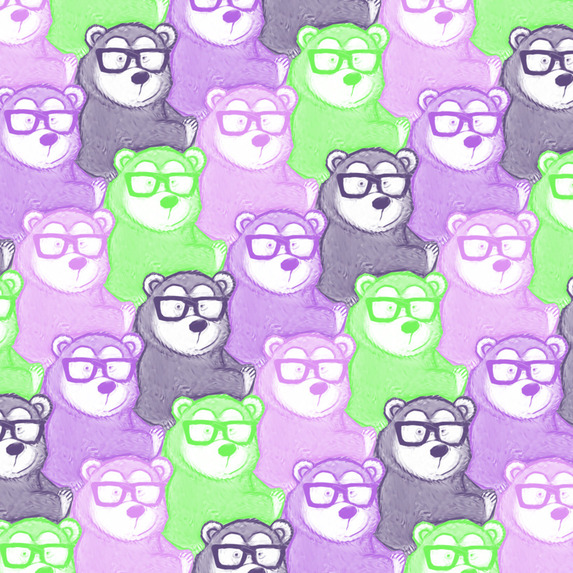}
        \end{subfigure}
        \begin{subfigure}[b]{0.23\linewidth}
        \includegraphics[width=\linewidth]{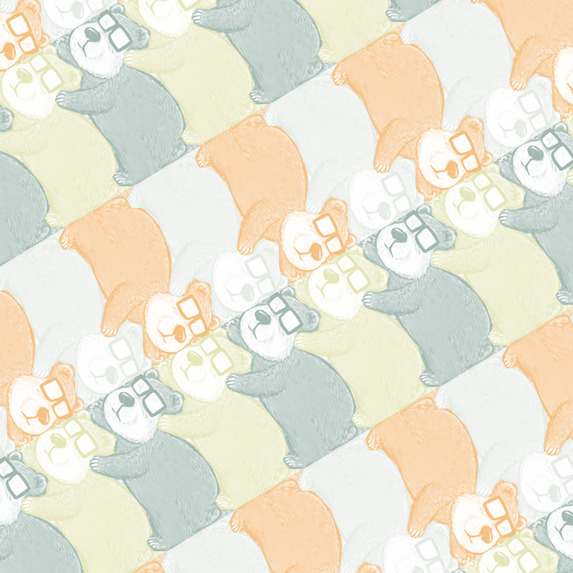}
        \end{subfigure}
        \begin{subfigure}[b]{0.23\linewidth}
        \includegraphics[width=\linewidth]{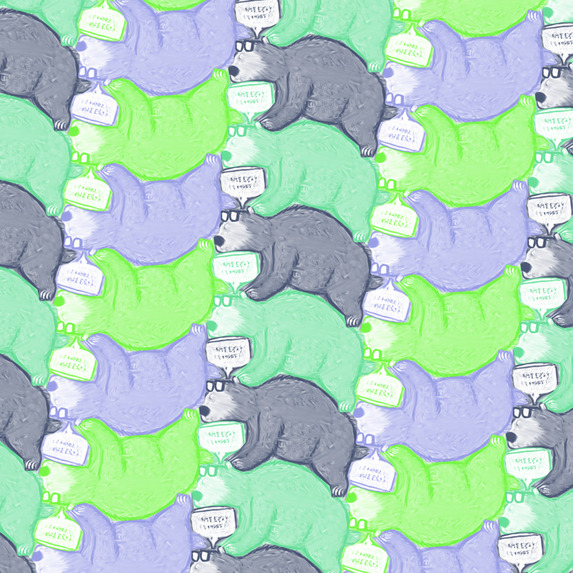}
        \end{subfigure}
        \begin{subfigure}[b]{0.23\linewidth}
        \includegraphics[width=\linewidth]{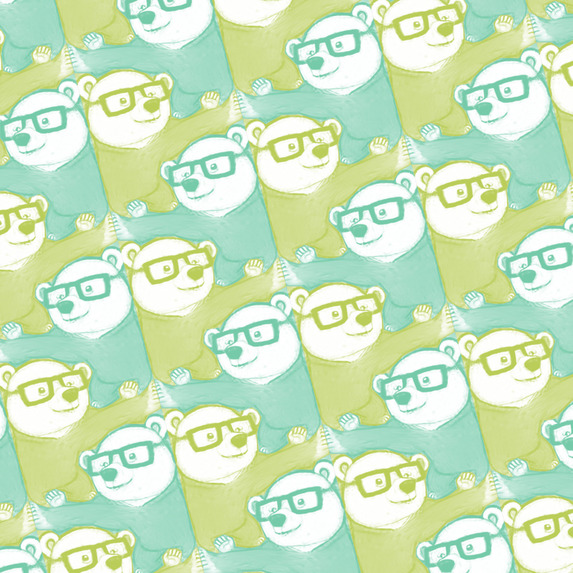}
        \end{subfigure}
\\
        \begin{subfigure}[b]{0.23\linewidth}
        \includegraphics[width=\linewidth]{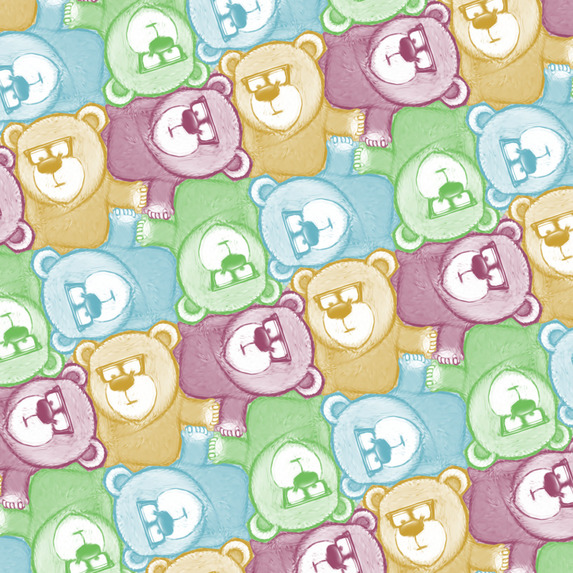}
        \end{subfigure}
        \begin{subfigure}[b]{0.23\linewidth}
        \includegraphics[width=\linewidth]{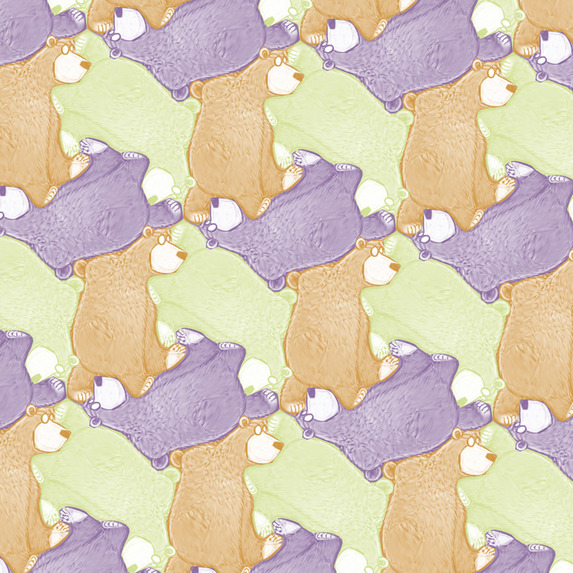}
        \end{subfigure}
        \begin{subfigure}[b]{0.23\linewidth}
        \includegraphics[width=\linewidth]{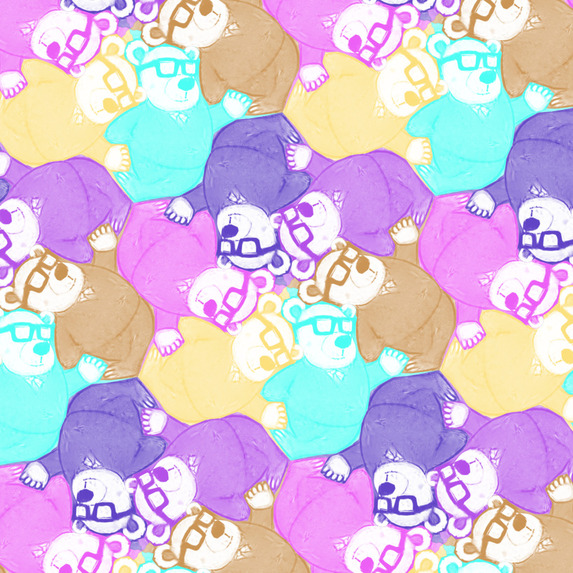}
        \end{subfigure}
\\
        \begin{subfigure}[b]{0.23\linewidth}
        \includegraphics[width=\linewidth]{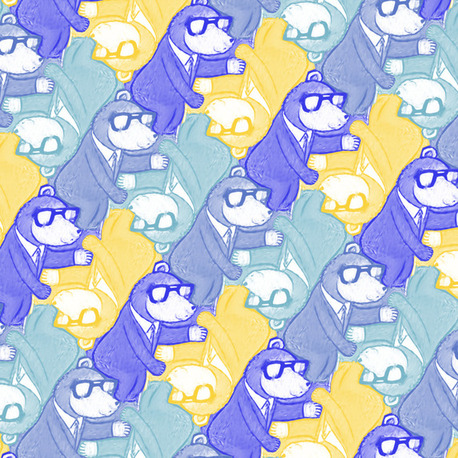}
        \end{subfigure}
        \begin{subfigure}[b]{0.23\linewidth}
        \includegraphics[width=\linewidth]{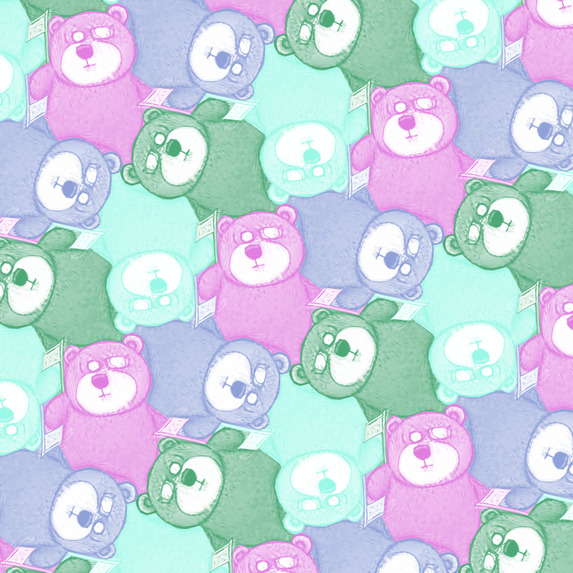}
        \end{subfigure}
        \begin{subfigure}[b]{0.23\linewidth}
        \includegraphics[width=\linewidth]{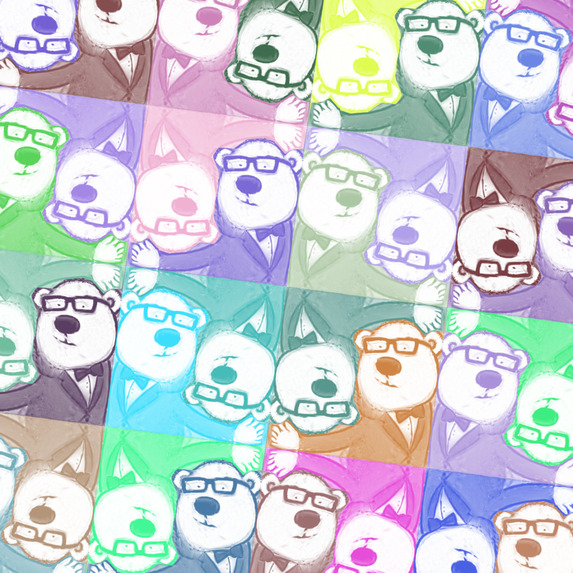}
        \end{subfigure}
\\
        \begin{subfigure}[b]{0.23\linewidth}
        \includegraphics[width=\linewidth]{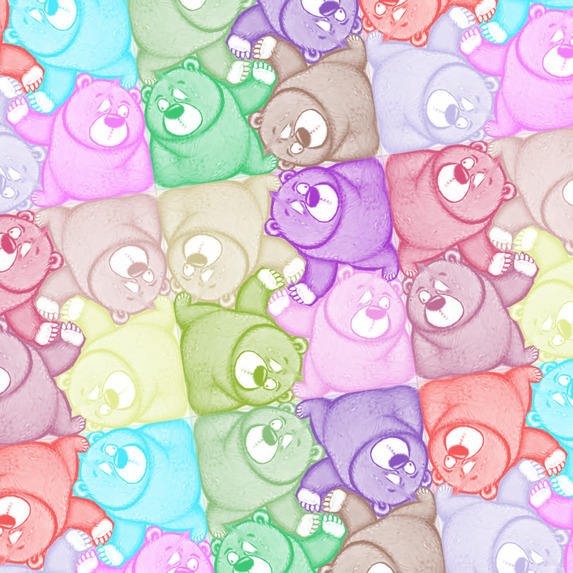}
        \end{subfigure}
        \begin{subfigure}[b]{0.23\linewidth}
        \includegraphics[width=\linewidth]{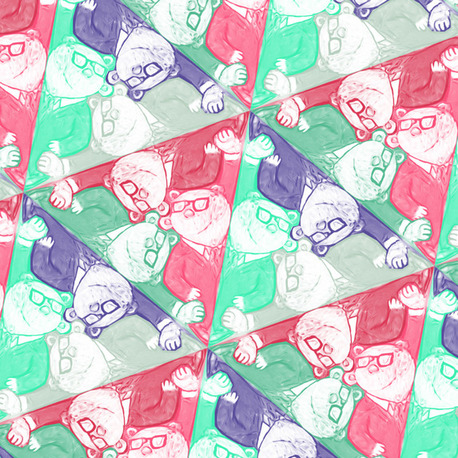}
        \end{subfigure}
        \begin{subfigure}[b]{0.23\linewidth}
        \includegraphics[width=\linewidth]{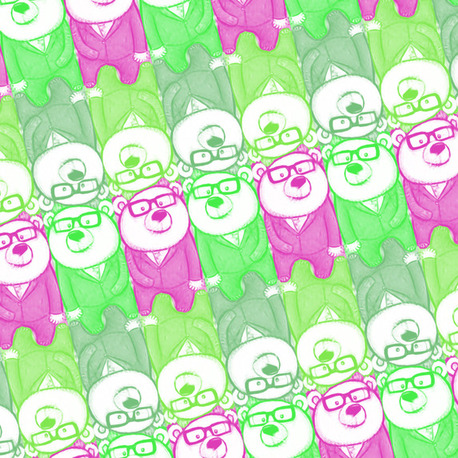}
        \end{subfigure}
\\
        \begin{subfigure}[b]{0.23\linewidth}
        \includegraphics[width=\linewidth]{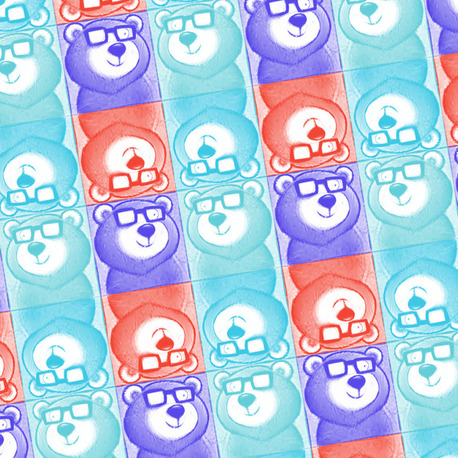}
        \end{subfigure}
        \begin{subfigure}[b]{0.23\linewidth}
        \includegraphics[width=\linewidth]{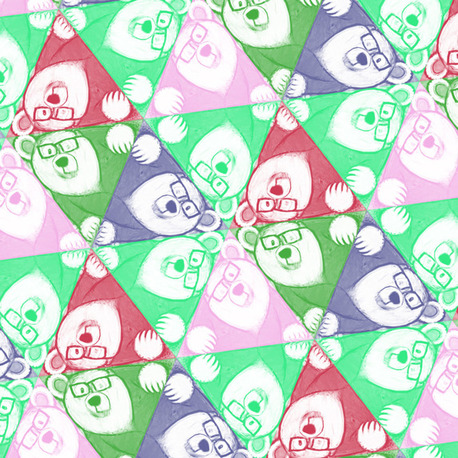}
        \end{subfigure}
        \begin{subfigure}[b]{0.23\linewidth}
        \includegraphics[width=\linewidth]{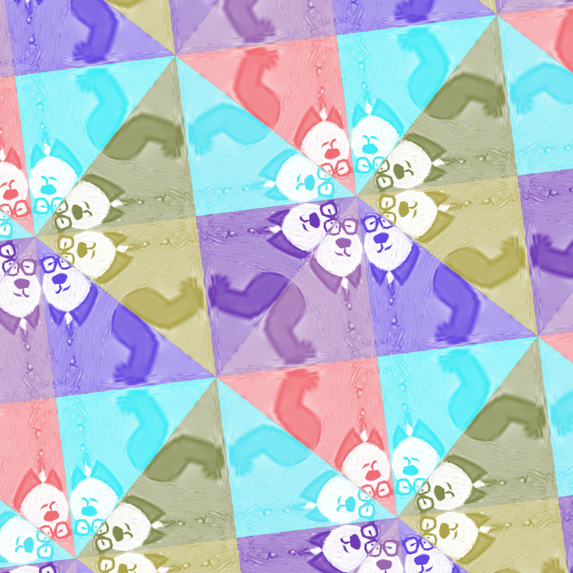}
        \end{subfigure}
        \begin{subfigure}[b]{0.23\linewidth}
        \includegraphics[width=\linewidth]{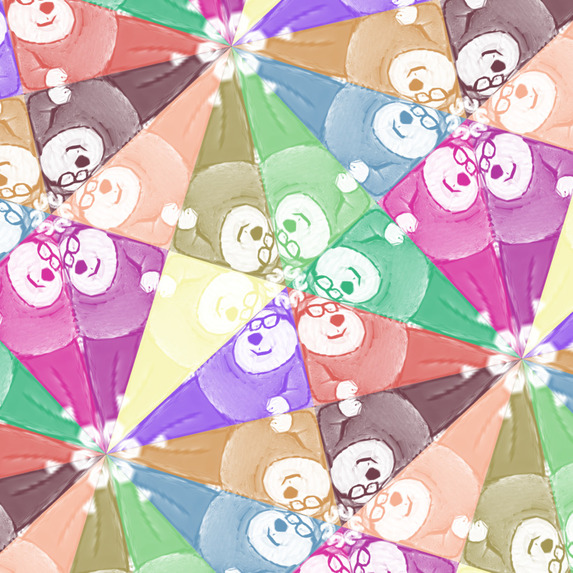}
        \end{subfigure}
\\
\captionsetup{justification=justified}
    \vspace*{-2mm}
    \caption{\textbf{All wallpaper groups for same prompt.} We showcase the 17 wallpaper groups for the prompt ``A nerdy bear''. Note that the last four groups (last row) produce uninteresting results in terms of shape because their symmetry group is made out of reflection, forcing their boundary to stay a static convex polygon (either square or triangle), preventing the generation of different tile-shapes. Order: \torus, \cylinder, \klein, \mob, \orbI, \orbII, \orbIII, \orbIV, \projective, \orbRhyb, \orbIhyb, \orbIIhyb, \orbIVhyb,   \reflectII, \reflectI, \reflectIII, \reflectIV .}
    \label{fig:wallpaper}
\end{figure*}

\newpage
\begin{figure*}
    \centering
    \captionsetup{justification=centering}
        \begin{subfigure}[t]{0.196\linewidth}
        \includegraphics[width=\linewidth]{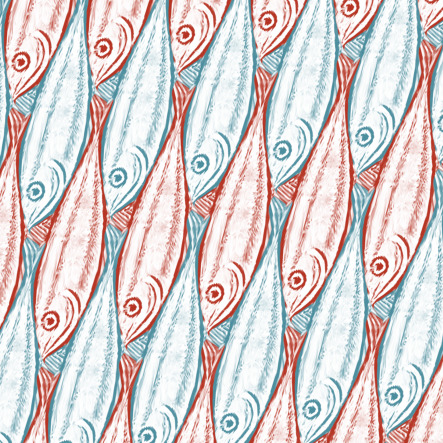}
        \caption{\theprompt{A sardine} \torus}
        \end{subfigure}
        \begin{subfigure}[t]{0.196\linewidth}
        \includegraphics[width=\linewidth]{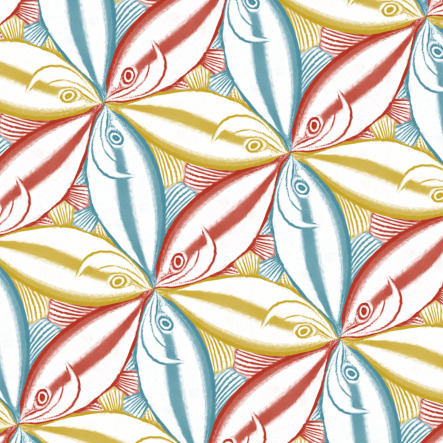}
        \caption{\theprompt{A sardine} \orbII}
        \end{subfigure}
        \begin{subfigure}[t]{0.196\linewidth}
        \includegraphics[width=\linewidth]{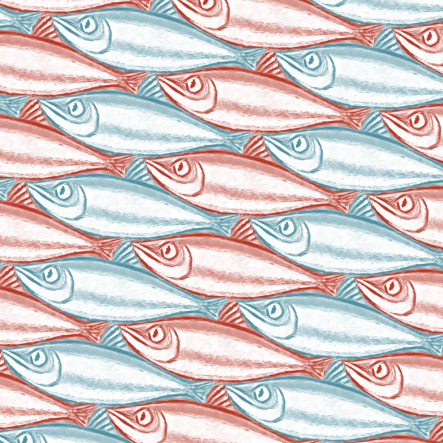}
        \caption{\theprompt{A sardine} \torus}
        \end{subfigure}
        \begin{subfigure}[t]{0.196\linewidth}
        \includegraphics[width=\linewidth]{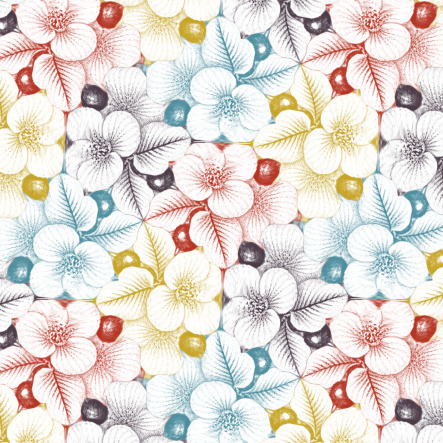}
        \caption{\theprompt{Cherries} \orbI}
        \end{subfigure}
        \begin{subfigure}[t]{0.196\linewidth}
        \includegraphics[width=\linewidth]{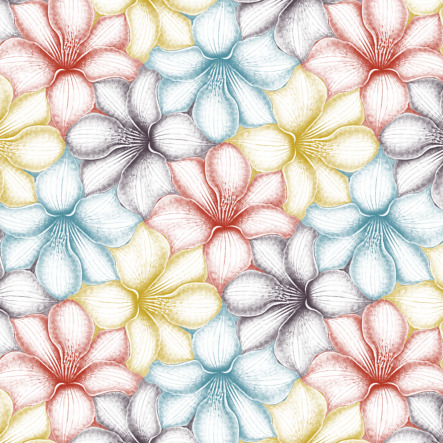}
        \caption{\theprompt{A tropical flower} \orbI}
        \end{subfigure}
\\
        \begin{subfigure}[t]{0.196\linewidth}
        \includegraphics[width=\linewidth]{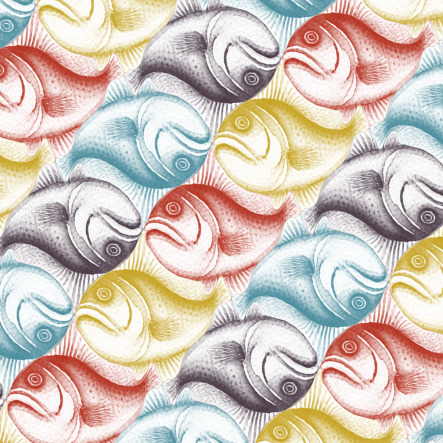}
        \caption{\theprompt{A coy fish} \orbIV}
        \end{subfigure}
        \begin{subfigure}[t]{0.196\linewidth}
        \includegraphics[width=\linewidth]{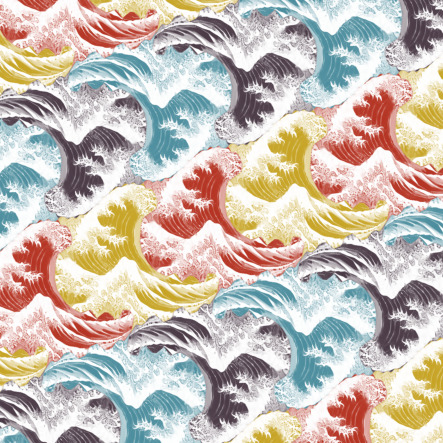}
        \caption{\theprompt{A wave} \orbIV}
        \end{subfigure}
        \begin{subfigure}[t]{0.196\linewidth}
        \includegraphics[width=\linewidth]{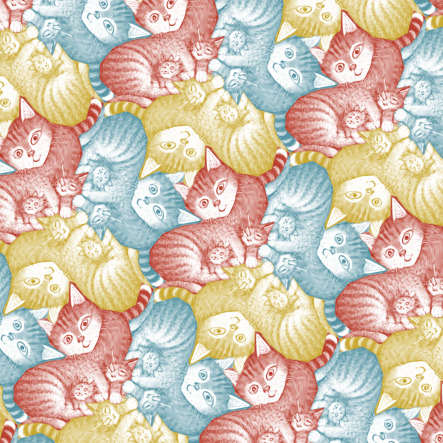}
        \caption{\theprompt{A mother cat} \orbII}
        \end{subfigure}
        \begin{subfigure}[t]{0.196\linewidth}
        \includegraphics[width=\linewidth]{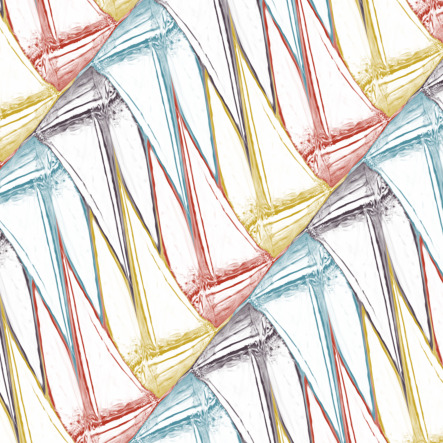}
        \caption{\theprompt{A sailing boat} \orbIV}
        \end{subfigure}
        \begin{subfigure}[t]{0.196\linewidth}
        \includegraphics[width=\linewidth]{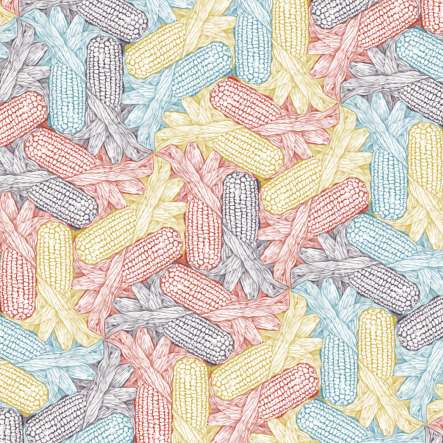}
        \caption{\theprompt{A corn cob} \orbIII}
        \end{subfigure}
\\
        \begin{subfigure}[t]{0.196\linewidth}
        \includegraphics[width=\linewidth]{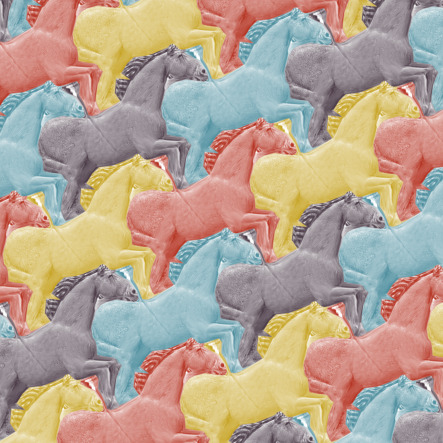}
        \caption{\theprompt{A centaur} \torus}
        \end{subfigure}
        \begin{subfigure}[t]{0.196\linewidth}
        \includegraphics[width=\linewidth]{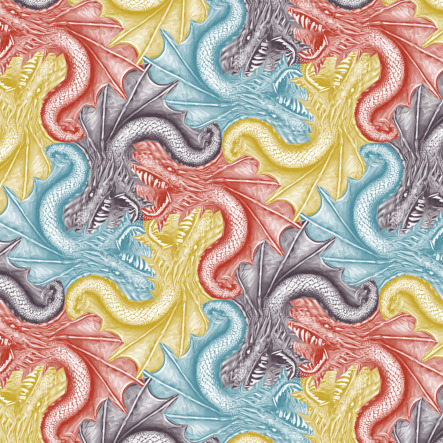}
        \caption{\theprompt{A dragon} \orbI}
        \end{subfigure}
        \begin{subfigure}[t]{0.196\linewidth}
        \includegraphics[width=\linewidth]{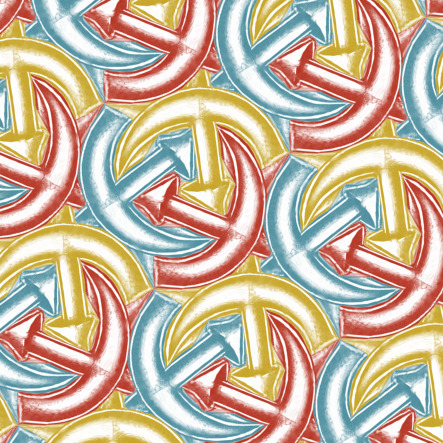}
        \caption{\theprompt{An anchor} \orbII}
        \end{subfigure}
        \begin{subfigure}[t]{0.196\linewidth}
        \includegraphics[width=\linewidth]{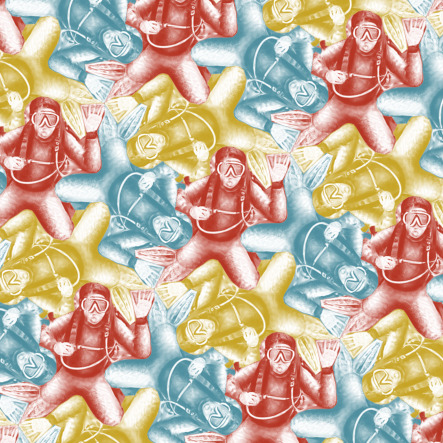}
        \caption{\theprompt{A scooba diver} \orbII}
        \end{subfigure}
        \begin{subfigure}[t]{0.196\linewidth}
        \includegraphics[width=\linewidth]{figures/gallery_2/14.jpg}
        \caption{\theprompt{A penguin} \torus}
        \end{subfigure}
\\
        \begin{subfigure}[t]{0.196\linewidth}
        \includegraphics[width=\linewidth]{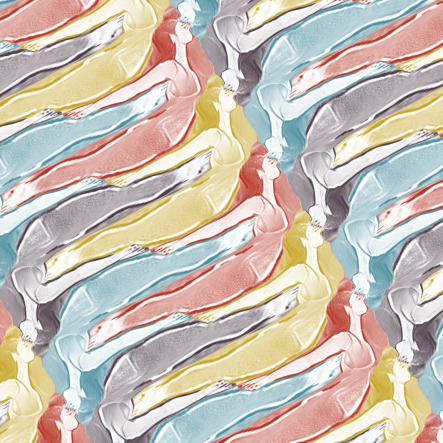}
        \caption{\theprompt{Woman yoga pose} \orbIV}
        \end{subfigure}
        \begin{subfigure}[t]{0.196\linewidth}
        \includegraphics[width=\linewidth]{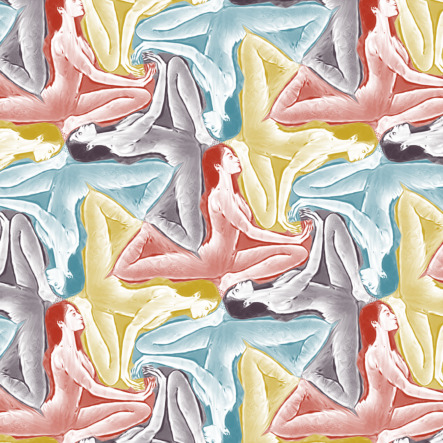}
        \caption{\theprompt{Woman yoga pose} \orbI}
        \end{subfigure}
        \begin{subfigure}[t]{0.196\linewidth}
        \includegraphics[width=\linewidth]{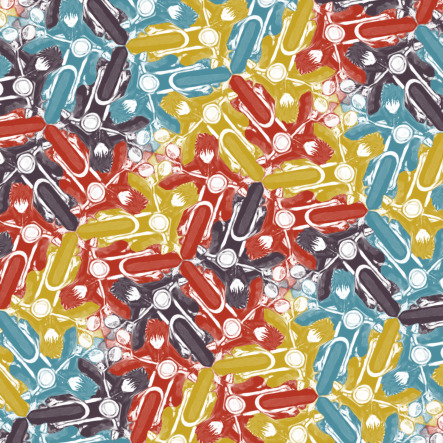}
        \caption{\theprompt{An anime biker} \orbIII}
        \end{subfigure}
        \begin{subfigure}[t]{0.196\linewidth}
        \includegraphics[width=\linewidth]{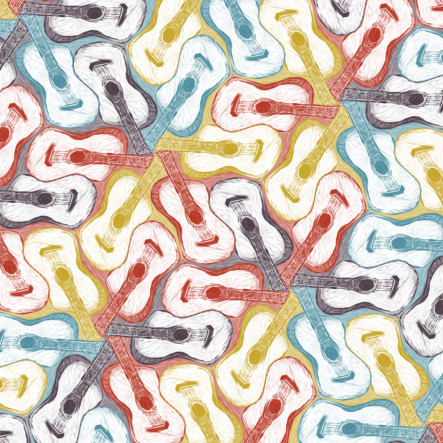}
        \caption{\theprompt{A guitar} \orbIII}
        \end{subfigure}
        \begin{subfigure}[t]{0.196\linewidth}
        \includegraphics[width=\linewidth]{figures/gallery_2/19.jpg}
        \caption{\theprompt{A bat} \torus}
        \end{subfigure}
\\
        \begin{subfigure}[t]{0.196\linewidth}
        \includegraphics[width=\linewidth]{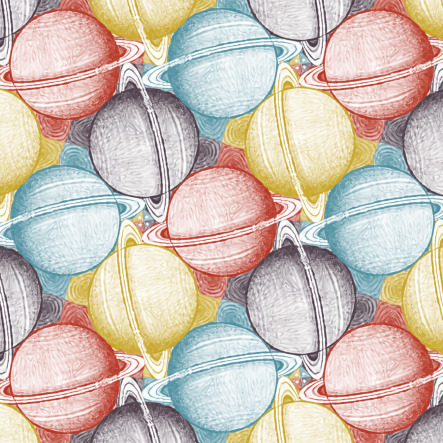}
        \caption{\theprompt{Saturn} \orbI}
        \end{subfigure}
        \begin{subfigure}[t]{0.196\linewidth}
        \includegraphics[width=\linewidth]{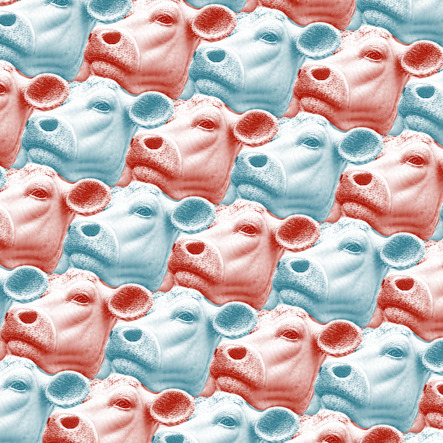}
        \caption{\theprompt{A cow} \torus}
        \end{subfigure}
        \begin{subfigure}[t]{0.196\linewidth}
        \includegraphics[width=\linewidth]{figures/gallery_2/22.jpg}
        \caption{\theprompt{A nerdy bear} \orbIV}
        \end{subfigure}
        \begin{subfigure}[t]{0.196\linewidth}
        \includegraphics[width=\linewidth]{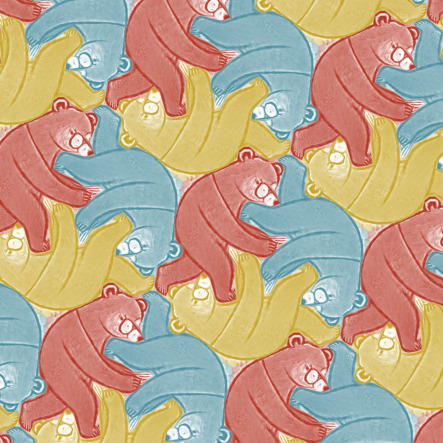}
        \caption{\theprompt{A nerdy bear} \orbII}
        \end{subfigure}
        \begin{subfigure}[t]{0.196\linewidth}
        \includegraphics[width=\linewidth]{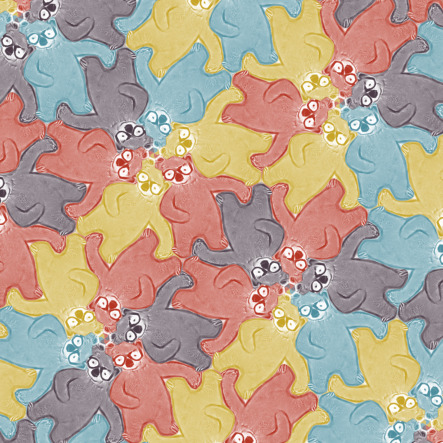}
        \caption{\theprompt{A nerdy bear} \orbIII}
        \end{subfigure}
\\
\captionsetup{justification=justified}
    \caption{\textbf{Tiling menagerie part II.} Examples of tilings produced by our method for different prompts and symmetries. Our method produces appealing, plausible results which contain solely the desired object, and cover the plane without overlaps.}
    \label{fig:gallery_2}
\end{figure*}
\begin{figure*}
    \centering
    \captionsetup{justification=centering}
        \begin{subfigure}[b]{0.196\linewidth}
        \includegraphics[width=\linewidth]{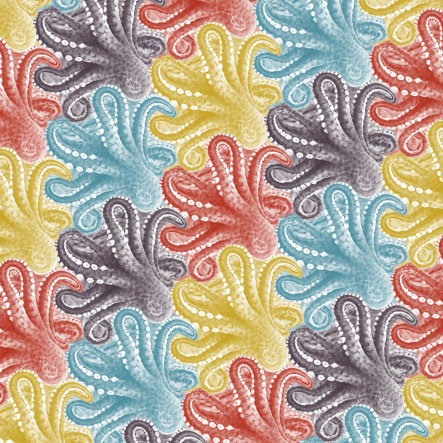}
        \caption{\theprompt{An octopus} \torus}
        \end{subfigure}
        \begin{subfigure}[b]{0.196\linewidth}
        \includegraphics[width=\linewidth]{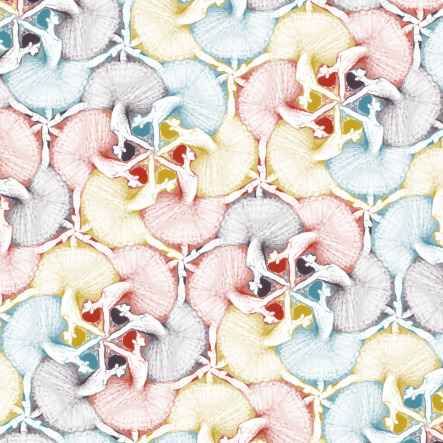}
        \caption{\theprompt{A ballet dancer} \orbIII}
        \end{subfigure}
        \begin{subfigure}[b]{0.196\linewidth}
        \includegraphics[width=\linewidth]{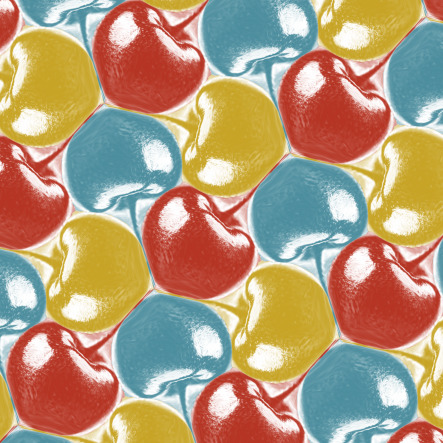}
        \caption{\theprompt{A cherry} \orbII}
        \end{subfigure}
        \begin{subfigure}[b]{0.196\linewidth}
        \includegraphics[width=\linewidth]{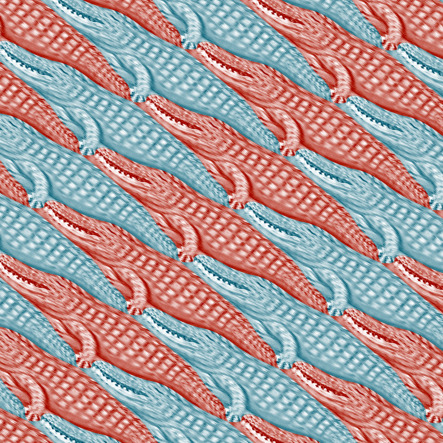}
        \caption{\theprompt{A crocodile} \torus}
        \end{subfigure}
        \begin{subfigure}[b]{0.196\linewidth}
        \includegraphics[width=\linewidth]{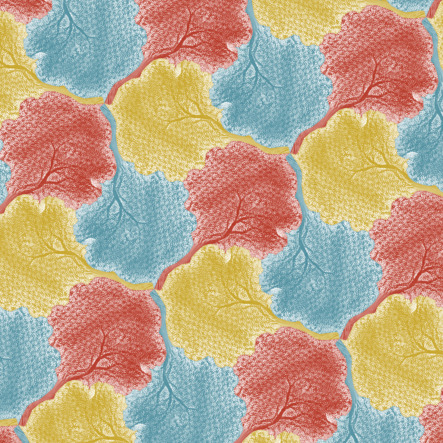}
        \caption{\theprompt{A tree} \orbII}
        \end{subfigure}
\\
        \begin{subfigure}[b]{0.196\linewidth}
        \includegraphics[width=\linewidth]{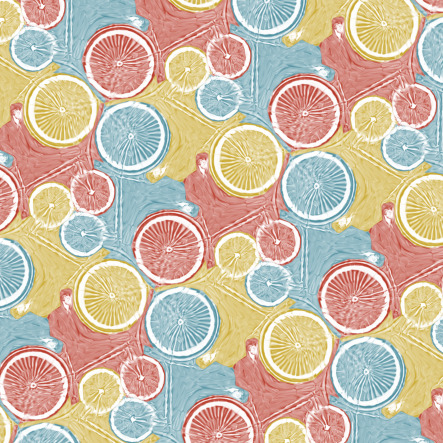}
        \caption{\theprompt{A bicycle} \orbII}
        \end{subfigure}
        \begin{subfigure}[b]{0.196\linewidth}
        \includegraphics[width=\linewidth]{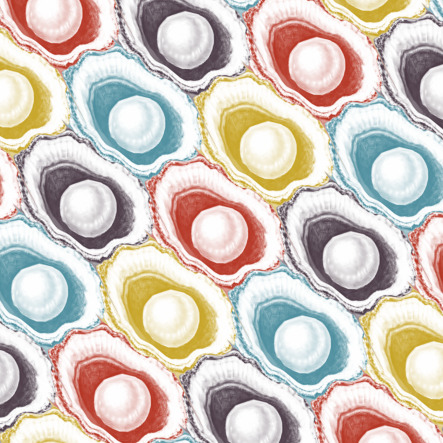}
        \caption{\theprompt{An oyster} \torus}
        \end{subfigure}
        \begin{subfigure}[b]{0.196\linewidth}
        \includegraphics[width=\linewidth]{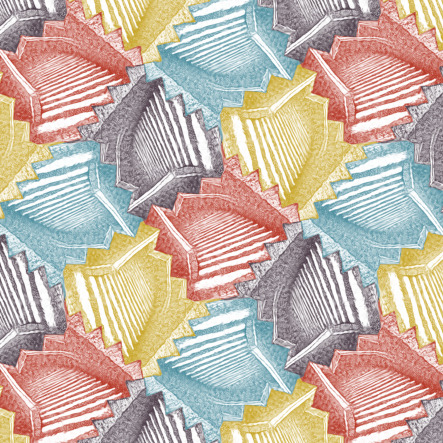}
        \caption{\theprompt{A staircase} \orbI}
        \end{subfigure}
        \begin{subfigure}[b]{0.196\linewidth}
        \includegraphics[width=\linewidth]{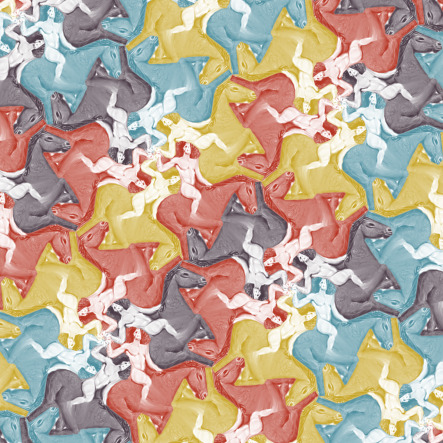}
        \caption{\theprompt{A centaur} \orbIII}
        \end{subfigure}
        \begin{subfigure}[b]{0.196\linewidth}
        \includegraphics[width=\linewidth]{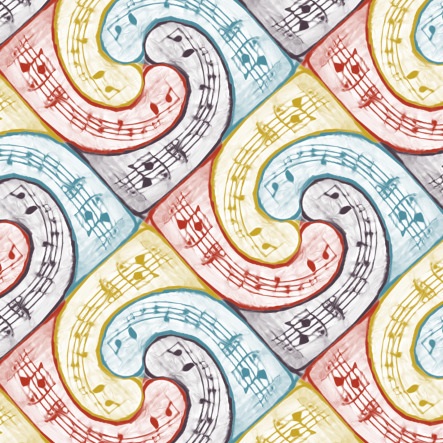}
        \caption{\theprompt{A musical note} \orbI}
        \end{subfigure}
\\
        \begin{subfigure}[b]{0.196\linewidth}
        \includegraphics[width=\linewidth]{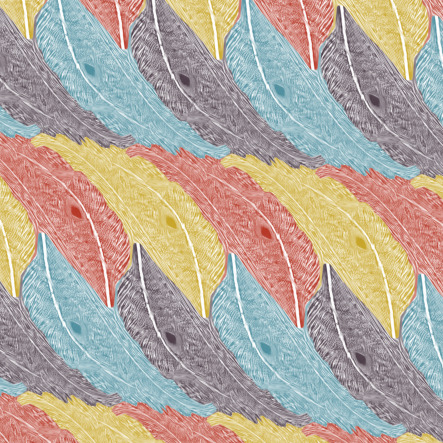}
        \caption{\theprompt{A peacock feather} \orbIV}
        \end{subfigure}
        \begin{subfigure}[b]{0.196\linewidth}
        \includegraphics[width=\linewidth]{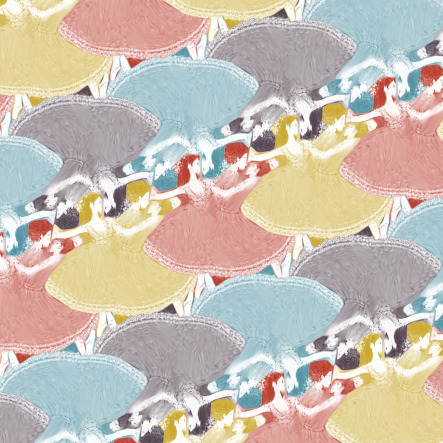}
        \caption{\theprompt{A dancing couple} \orbIV}
        \end{subfigure}
        \begin{subfigure}[b]{0.196\linewidth}
        \includegraphics[width=\linewidth]{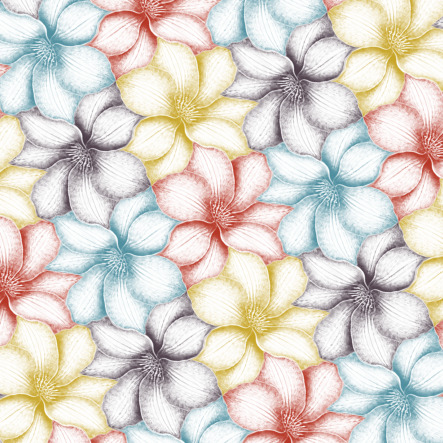}
        \caption{\theprompt{A tropical flower} \orbIV}
        \end{subfigure}
        \begin{subfigure}[b]{0.196\linewidth}
        \includegraphics[width=\linewidth]{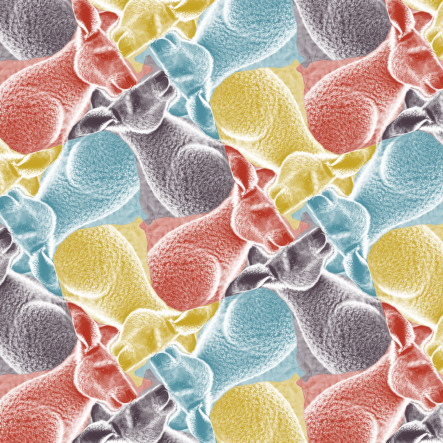}
        \caption{\theprompt{A kangaroo} \orbI}
        \end{subfigure}
        \begin{subfigure}[b]{0.196\linewidth}
        \includegraphics[width=\linewidth]{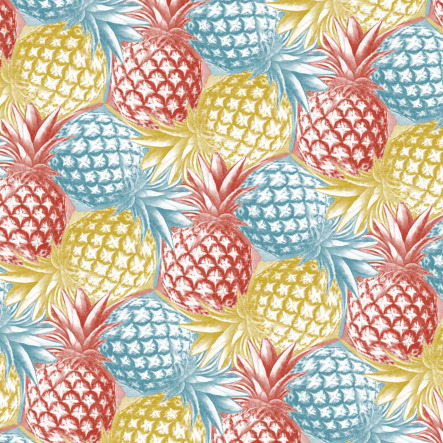}
        \caption{\theprompt{A pineapple} \orbII}
        \end{subfigure}
\\
        \begin{subfigure}[b]{0.196\linewidth}
        \includegraphics[width=\linewidth]{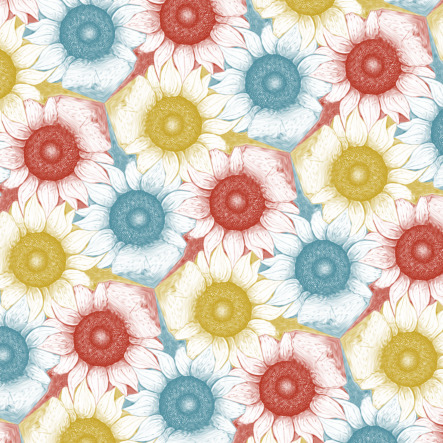}
        \caption{\theprompt{A sunflower} \orbII}
        \end{subfigure}
        \begin{subfigure}[b]{0.196\linewidth}
        \includegraphics[width=\linewidth]{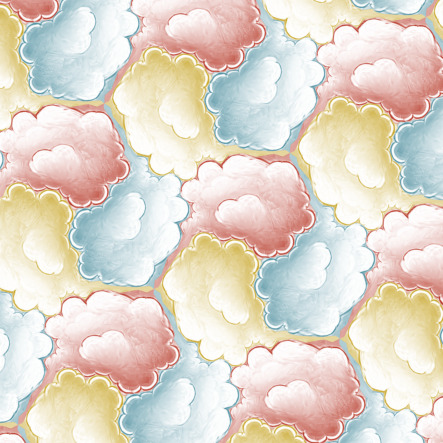}
        \caption{\theprompt{A cloud} \orbII}
        \end{subfigure}
        \begin{subfigure}[b]{0.196\linewidth}
        \includegraphics[width=\linewidth]{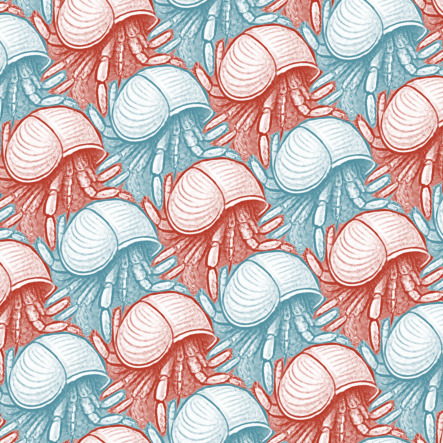}
        \caption{\theprompt{A hermit crab} \torus}
        \end{subfigure}
        \begin{subfigure}[b]{0.196\linewidth}
        \includegraphics[width=\linewidth]{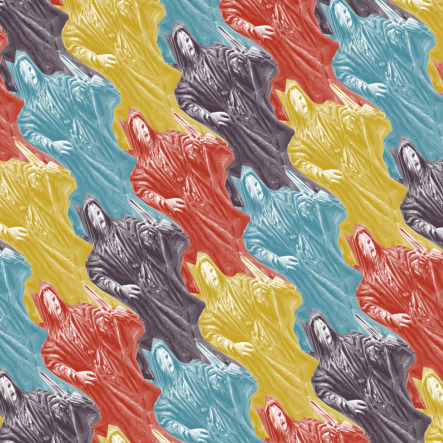}
        \caption{\theprompt{A sorcerer} \torus}
        \end{subfigure}
        \begin{subfigure}[b]{0.196\linewidth}
        \includegraphics[width=\linewidth]{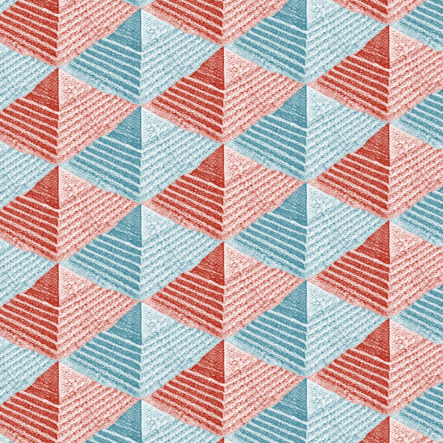}
        \caption{\theprompt{Pyramids of Giza} \torus}
        \end{subfigure}
\\
        \begin{subfigure}[b]{0.196\linewidth}
        \includegraphics[width=\linewidth]{figures/gallery_1/20.jpg}
        \caption{\theprompt{A lion} \torus}
        \end{subfigure}
        \begin{subfigure}[b]{0.196\linewidth}
        \includegraphics[width=\linewidth]{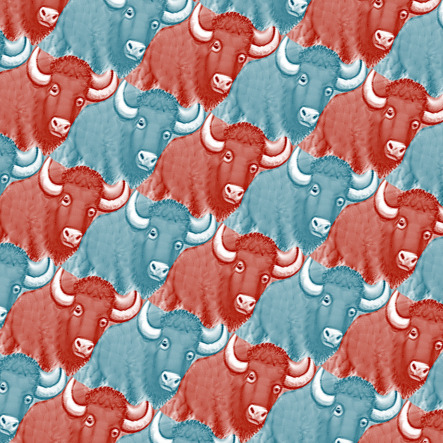}
        \caption{\theprompt{A buffalo} \torus}
        \end{subfigure}
        \begin{subfigure}[b]{0.196\linewidth}
        \includegraphics[width=\linewidth]{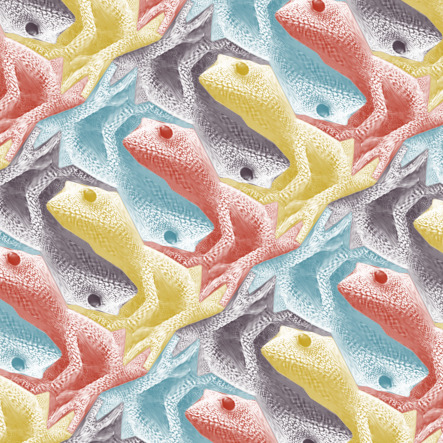}
        \caption{\theprompt{A gecko} \orbIV}
        \end{subfigure}
        \begin{subfigure}[b]{0.196\linewidth}
        \includegraphics[width=\linewidth]{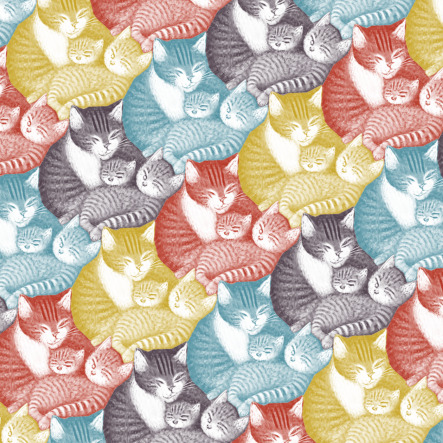}
        \caption{\theprompt{A cat} \torus}
        \end{subfigure}
        \begin{subfigure}[b]{0.196\linewidth}
        \includegraphics[width=\linewidth]{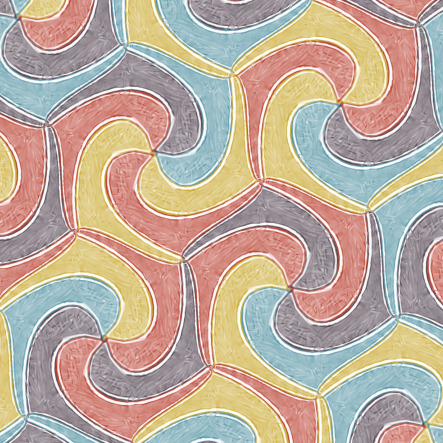}
        \caption{\theprompt{A boomerang} \orbIII}
        \end{subfigure}
\\
\captionsetup{justification=justified}
    \caption{\textbf{Tiling menagerie part III.} Examples of tilings produced by our method for different prompts and symmetries. Our method produces appealing, plausible results which contain solely the desired object, and cover the plane without overlaps.}
    \label{fig:gallery_1}
\end{figure*}
\begin{figure*}
    \centering
    \captionsetup{justification=centering}
        \begin{subfigure}[b]{0.196\linewidth}
        \includegraphics[width=\linewidth]{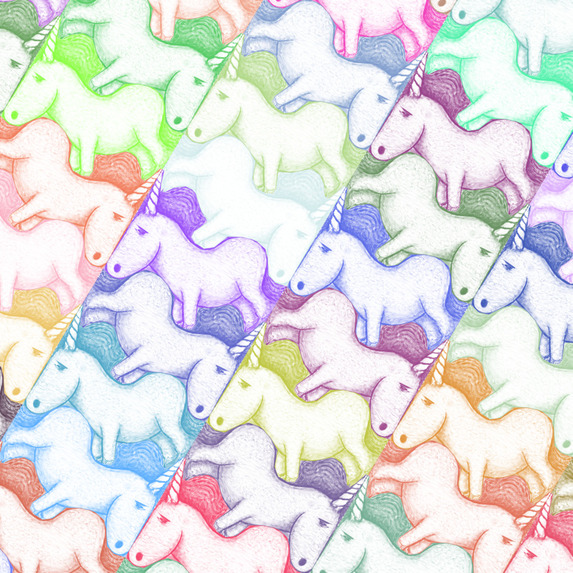}
        \caption{\theprompt{a fat unicorn} \mob}
        \end{subfigure}
        \begin{subfigure}[b]{0.196\linewidth}
        \includegraphics[width=\linewidth]{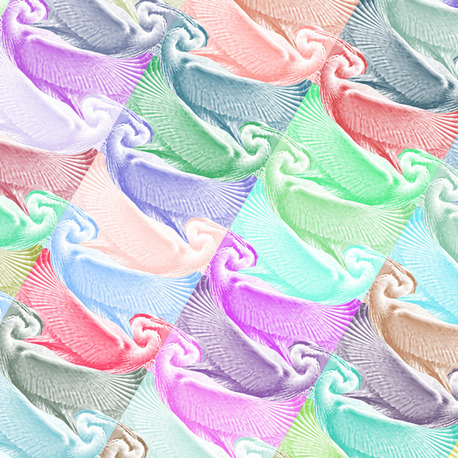}
        \caption{\theprompt{a heron} \mob}
        \end{subfigure}
        \begin{subfigure}[b]{0.196\linewidth}
        \includegraphics[width=\linewidth]{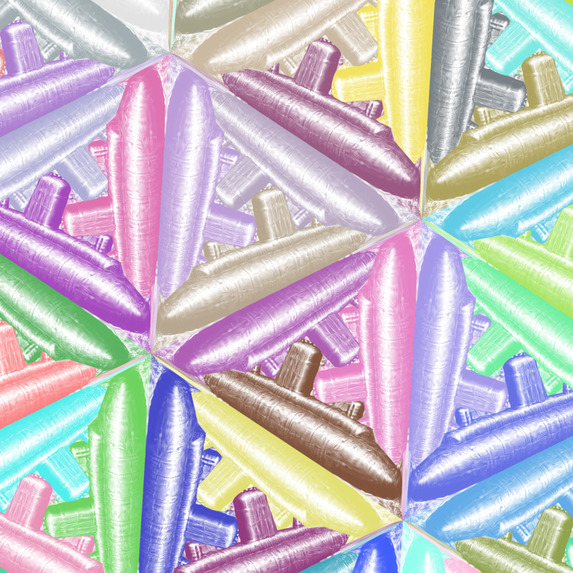}
        \caption{\theprompt{a submarine} \orbIIhyb}
        \end{subfigure}
        \begin{subfigure}[b]{0.196\linewidth}
        \includegraphics[width=\linewidth]{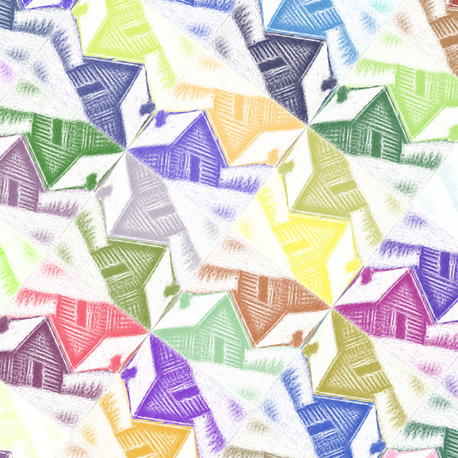}
        \caption{\theprompt{a snowy cabin} \orbRhyb}
        \end{subfigure}
        \begin{subfigure}[b]{0.196\linewidth}
        \includegraphics[width=\linewidth]{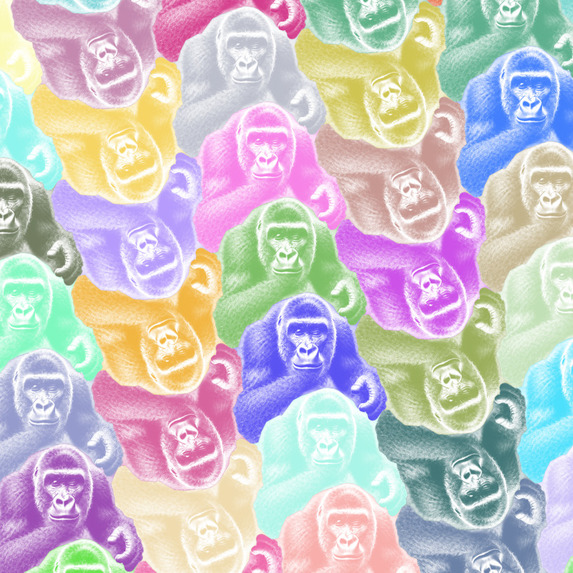}
        \caption{\theprompt{a funky gorilla} \klein}
        \end{subfigure}
\\
        \begin{subfigure}[b]{0.196\linewidth}
        \includegraphics[width=\linewidth]{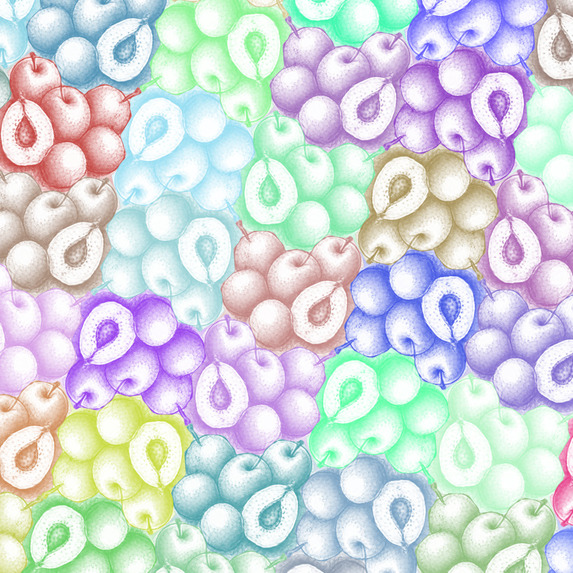}
        \caption{\theprompt{Fruits} \projective}
        \end{subfigure}
        \begin{subfigure}[b]{0.196\linewidth}
        \includegraphics[width=\linewidth]{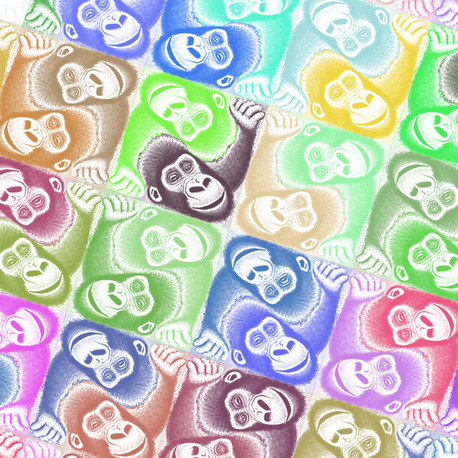}
        \caption{\theprompt{a funky gorilla} \orbRhyb}
        \end{subfigure}
        \begin{subfigure}[b]{0.196\linewidth}
        \includegraphics[width=\linewidth]{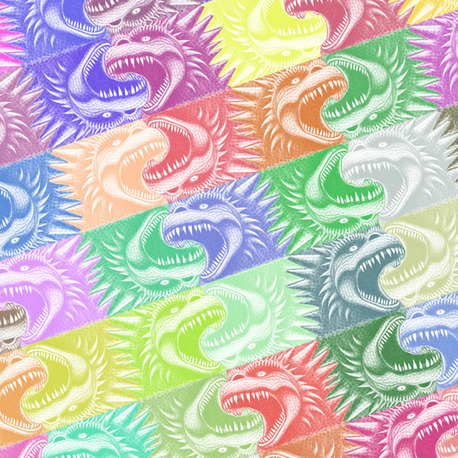}
        \caption{\theprompt{a Godzilla} \orbIVhyb}
        \end{subfigure}
        \begin{subfigure}[b]{0.196\linewidth}
        \includegraphics[width=\linewidth]{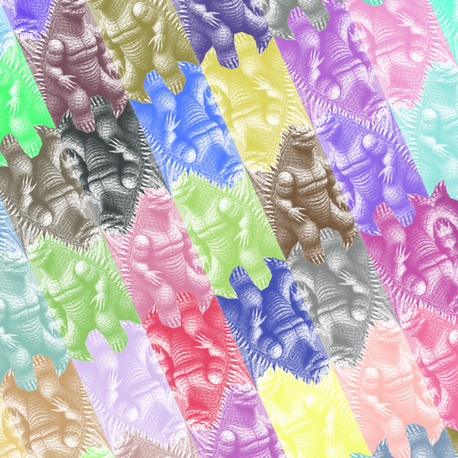}
        \caption{\theprompt{a Godzilla} \orbIVhyb}
        \end{subfigure}
        \begin{subfigure}[b]{0.196\linewidth}
        \includegraphics[width=\linewidth]{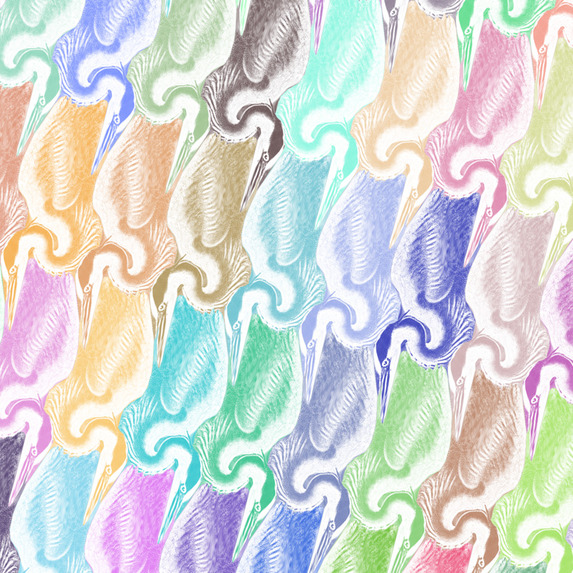}
        \caption{\theprompt{a heron} \klein}
        \end{subfigure}
\\
        \begin{subfigure}[b]{0.196\linewidth}
        \includegraphics[width=\linewidth]{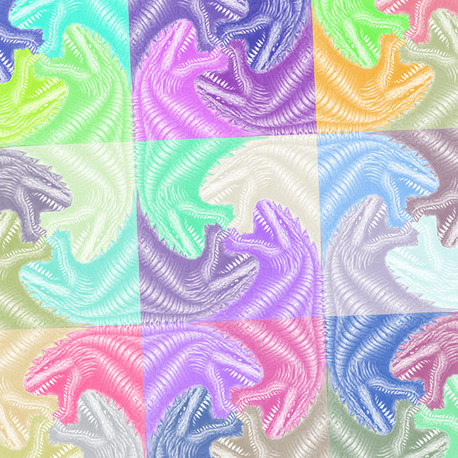}
        \caption{\theprompt{a Godzilla} \orbIhyb}
        \end{subfigure}
        \begin{subfigure}[b]{0.196\linewidth}
        \includegraphics[width=\linewidth]{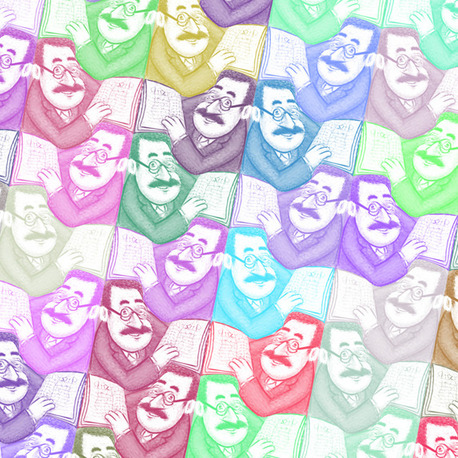}
        \caption{\theprompt{a math professor} \mob}
        \end{subfigure}
        \begin{subfigure}[b]{0.196\linewidth}
        \includegraphics[width=\linewidth]{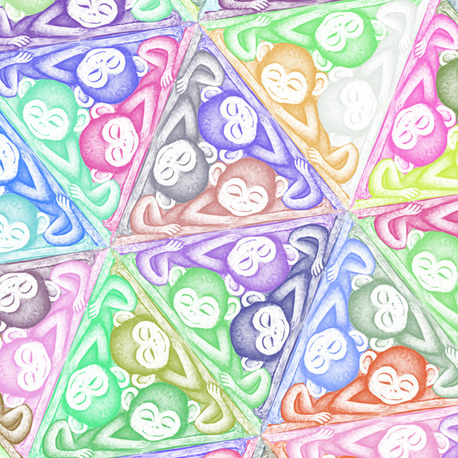}
        \caption{\theprompt{a monkey} \orbIIhyb}
        \end{subfigure}
        \begin{subfigure}[b]{0.196\linewidth}
        \includegraphics[width=\linewidth]{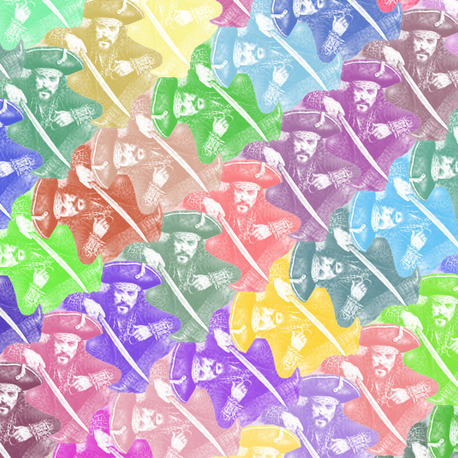}
        \caption{\theprompt{a pirate} \klein}
        \end{subfigure}
        \begin{subfigure}[b]{0.196\linewidth}
        \includegraphics[width=\linewidth]{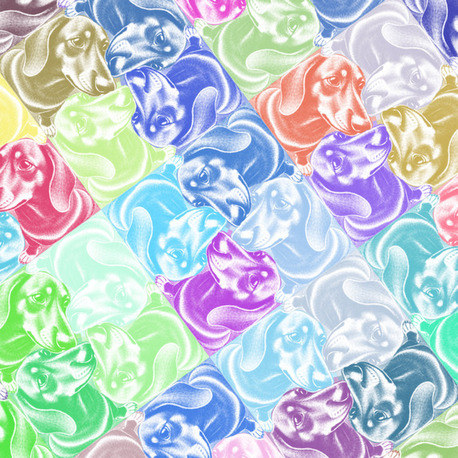}
        \caption{\theprompt{a Dachshund} \orbIhyb}
        \end{subfigure}
\\
        \begin{subfigure}[b]{0.196\linewidth}
        \includegraphics[width=\linewidth]{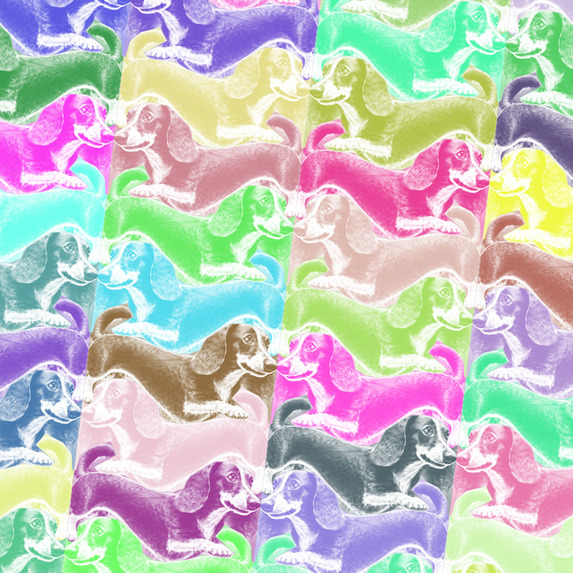}
        \caption{\theprompt{a Dachshund} \mob}
        \end{subfigure}
        \begin{subfigure}[b]{0.196\linewidth}
        \includegraphics[width=\linewidth]{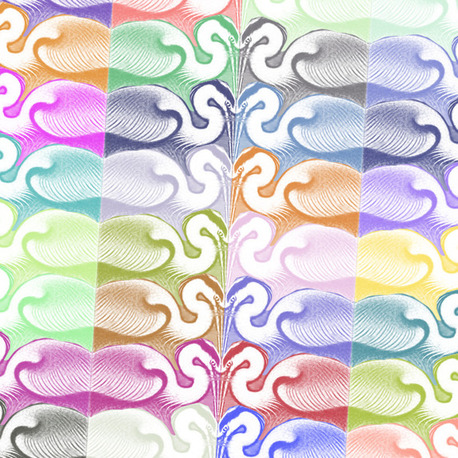}
        \caption{\theprompt{a heron} \cylinder}
        \end{subfigure}
        \begin{subfigure}[b]{0.196\linewidth}
        \includegraphics[width=\linewidth]{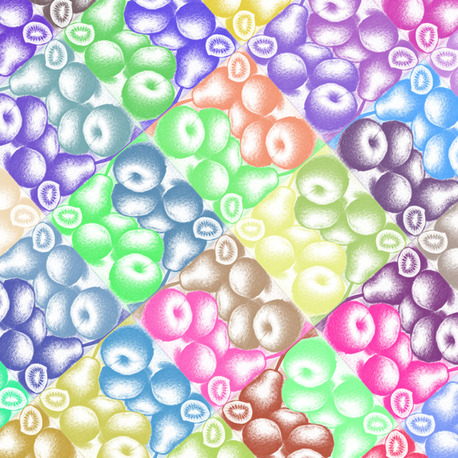}
        \caption{\theprompt{Fruits} \orbRhyb}
        \end{subfigure}
        \begin{subfigure}[b]{0.196\linewidth}
        \includegraphics[width=\linewidth]{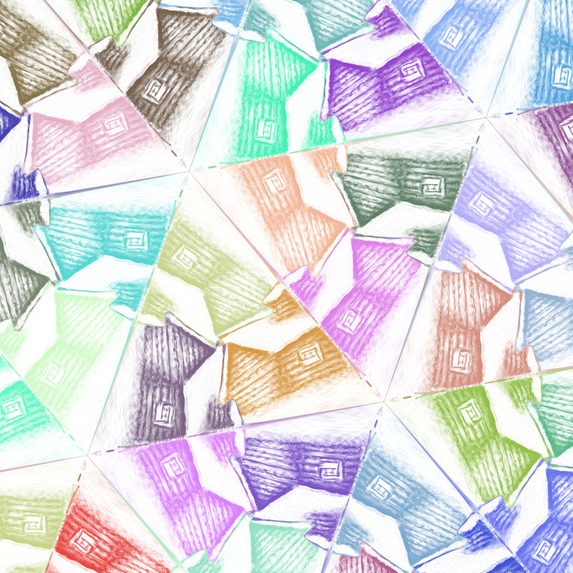}
        \caption{\theprompt{a snowy cabin} \orbIIhyb}
        \end{subfigure}
        \begin{subfigure}[b]{0.196\linewidth}
        \includegraphics[width=\linewidth]{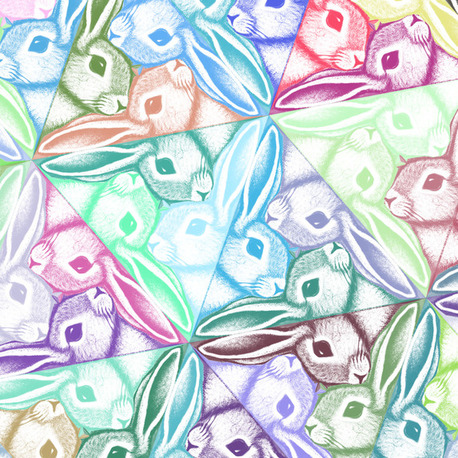}
        \caption{\theprompt{a hippie rabbit} \orbIIhyb}
        \end{subfigure}
\\
        \begin{subfigure}[b]{0.196\linewidth}
        \includegraphics[width=\linewidth]{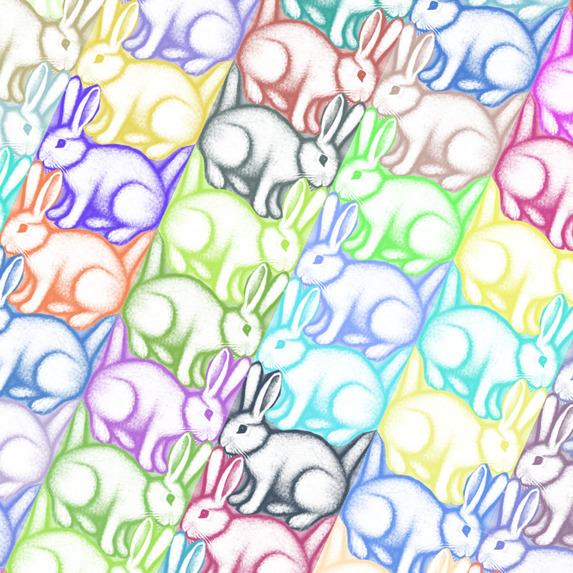}
        \caption{\theprompt{a hippie rabbit} \cylinder}
        \end{subfigure}
        \begin{subfigure}[b]{0.196\linewidth}
        \includegraphics[width=\linewidth]{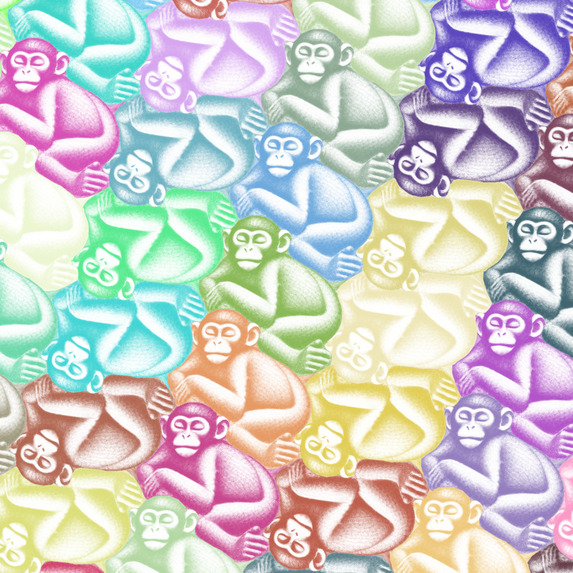}
        \caption{\theprompt{a monkey} \klein}
        \end{subfigure}
        \begin{subfigure}[b]{0.196\linewidth}
        \includegraphics[width=\linewidth]{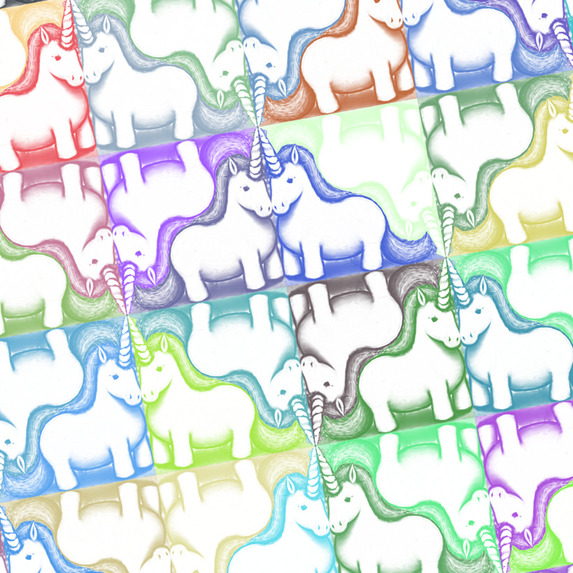}
        \caption{\theprompt{a fat unicorn} \orbRhyb}
        \end{subfigure}
        \begin{subfigure}[b]{0.196\linewidth}
        \includegraphics[width=\linewidth]{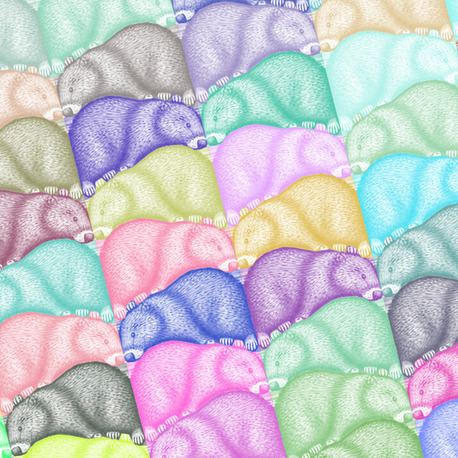}
        \caption{\theprompt{a beaver} \mob}
        \end{subfigure}
        \begin{subfigure}[b]{0.196\linewidth}
        \includegraphics[width=\linewidth]{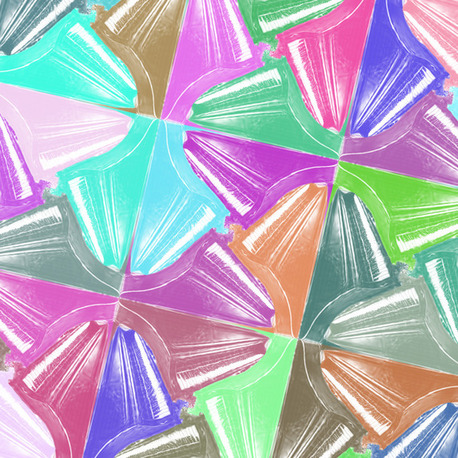}
        \caption{\theprompt{a piano} \orbIhyb}
        \end{subfigure}
\\
\captionsetup{justification=justified}
    \caption{\textbf{Tiling menagerie part IV.} Examples of tilings produced by our method for different prompts and symmetries. Our method produces appealing, plausible results which contain solely the desired object, and cover the plane without overlaps.}
    \label{fig:gallery_0}
\end{figure*}

\begin{figure*}
    \centering
    \captionsetup{justification=centering}
        \begin{subfigure}[b]{0.196\linewidth}
        \includegraphics[width=\linewidth]{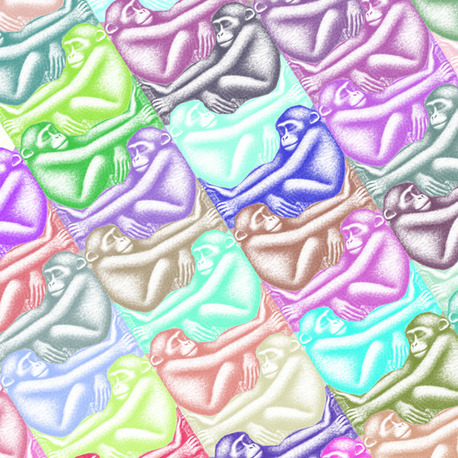}
        \caption{\theprompt{a Skeletton} \mob}
        \end{subfigure}
        \begin{subfigure}[b]{0.196\linewidth}
        \includegraphics[width=\linewidth]{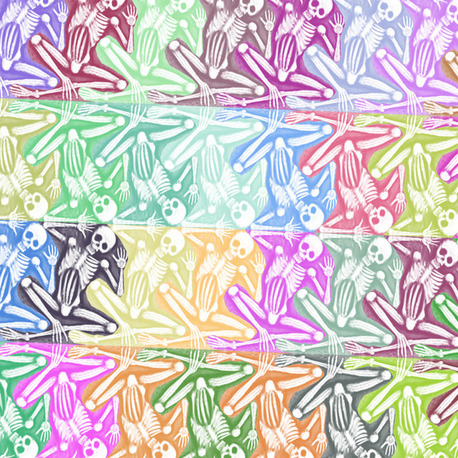}
        \caption{\theprompt{a Skeletton} \cylinder}
        \end{subfigure}
        \begin{subfigure}[b]{0.196\linewidth}
        \includegraphics[width=\linewidth]{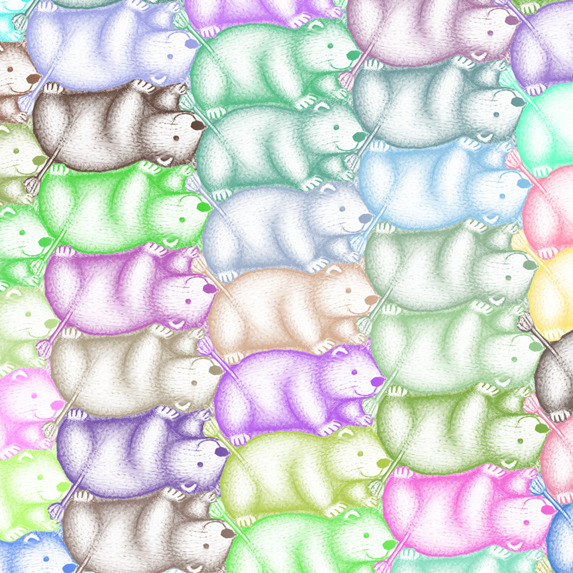}
        \caption{\theprompt{a beaver} \klein}
        \end{subfigure}
        \begin{subfigure}[b]{0.196\linewidth}
        \includegraphics[width=\linewidth]{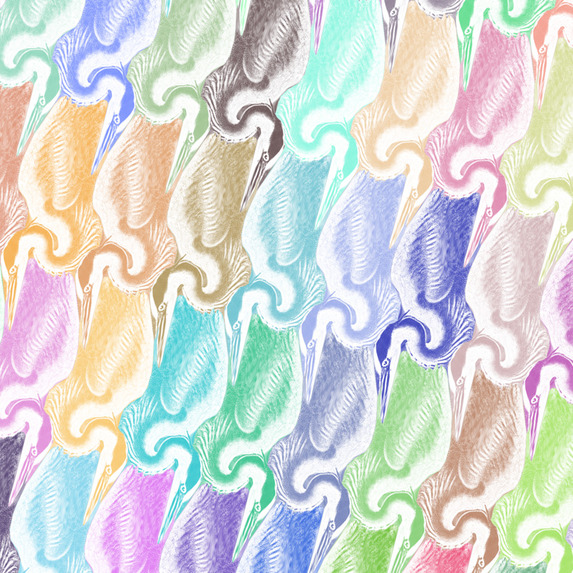}
        \caption{\theprompt{a heron} \klein}
        \end{subfigure}
        \begin{subfigure}[b]{0.196\linewidth}
        \includegraphics[width=\linewidth]{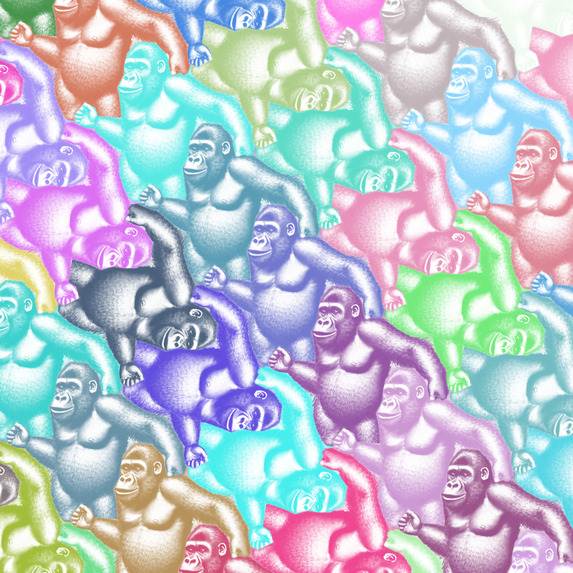}
        \caption{\theprompt{a funky gorilla} \klein}
        \end{subfigure}
\\
        \begin{subfigure}[b]{0.196\linewidth}
        \includegraphics[width=\linewidth]{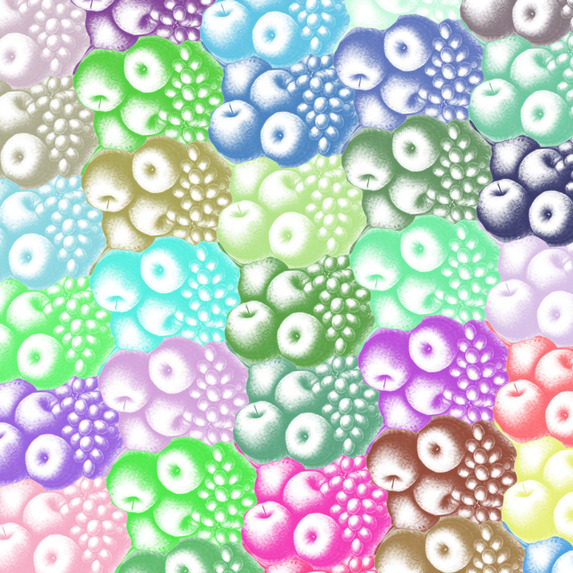}
        \caption{\theprompt{Fruits} \klein}
        \end{subfigure}
        \begin{subfigure}[b]{0.196\linewidth}
        \includegraphics[width=\linewidth]{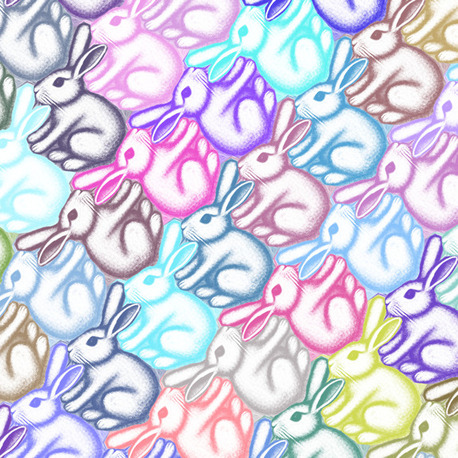}
        \caption{\theprompt{a hippie rabbit} \klein}
        \end{subfigure}
        \begin{subfigure}[b]{0.196\linewidth}
        \includegraphics[width=\linewidth]{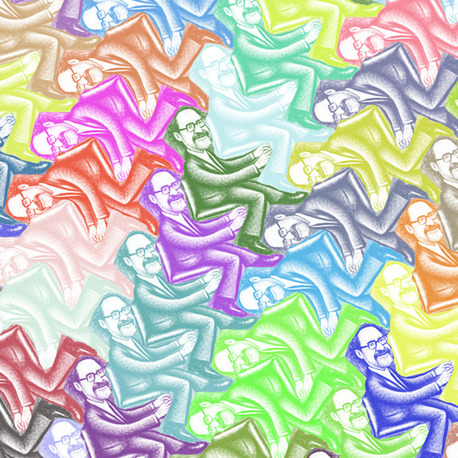}
        \caption{\theprompt{a math professor} \klein}
        \end{subfigure}
        \begin{subfigure}[b]{0.196\linewidth}
        \includegraphics[width=\linewidth]{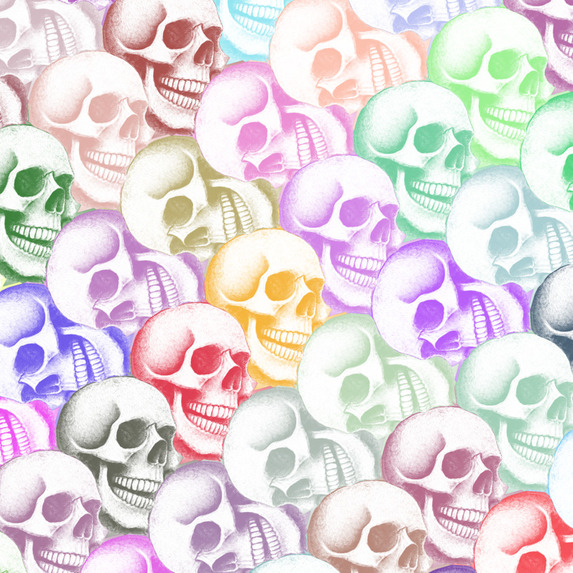}
        \caption{\theprompt{a skull} \klein}
        \end{subfigure}
        \begin{subfigure}[b]{0.196\linewidth}
        \includegraphics[width=\linewidth]{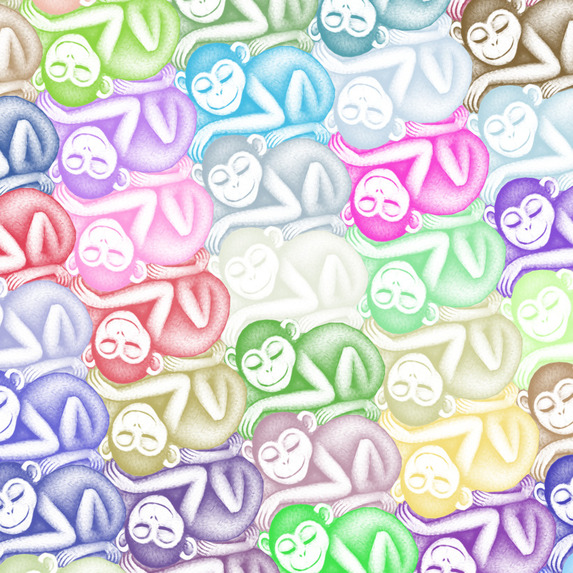}
        \caption{\theprompt{a monkey} \klein}
        \end{subfigure}
\\
        \begin{subfigure}[b]{0.196\linewidth}
        \includegraphics[width=\linewidth]{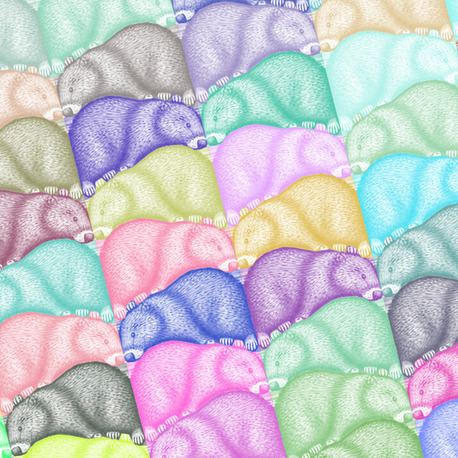}
        \caption{\theprompt{a beaver} \mob}
        \end{subfigure}
        \begin{subfigure}[b]{0.196\linewidth}
        \includegraphics[width=\linewidth]{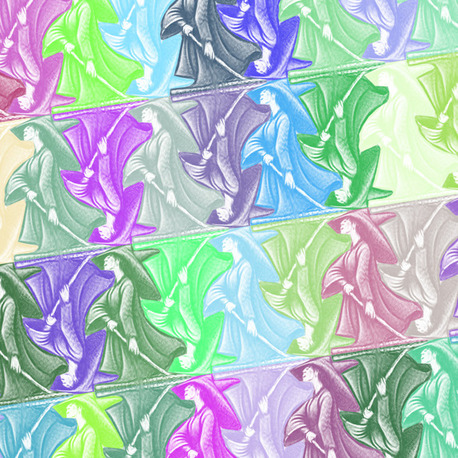}
        \caption{\theprompt{a witch} \mob}
        \end{subfigure}
        \begin{subfigure}[b]{0.196\linewidth}
        \includegraphics[width=\linewidth]{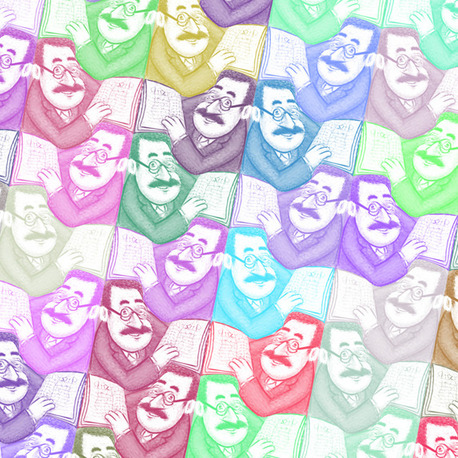}
        \caption{\theprompt{a math professor} \mob}
        \end{subfigure}
        \begin{subfigure}[b]{0.196\linewidth}
        \includegraphics[width=\linewidth]{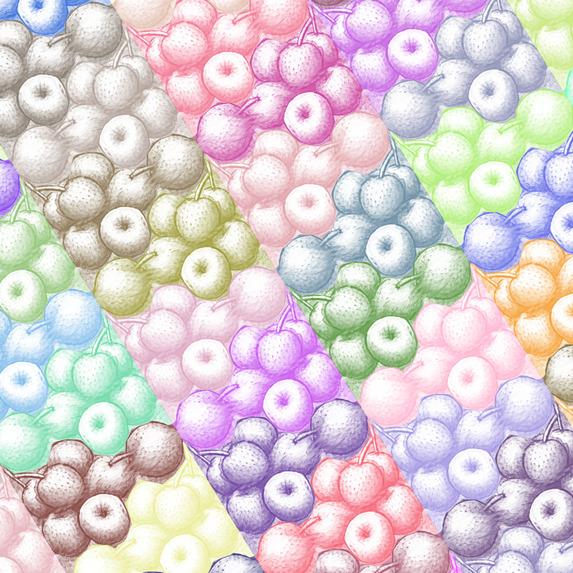}
        \caption{\theprompt{Fruits} \mob}
        \end{subfigure}
        \begin{subfigure}[b]{0.196\linewidth}
        \includegraphics[width=\linewidth]{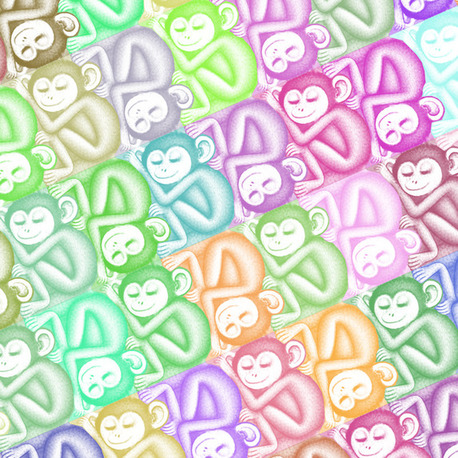}
        \caption{\theprompt{a monkey} \mob}
        \end{subfigure}
\\
        \begin{subfigure}[b]{0.196\linewidth}
        \includegraphics[width=\linewidth]{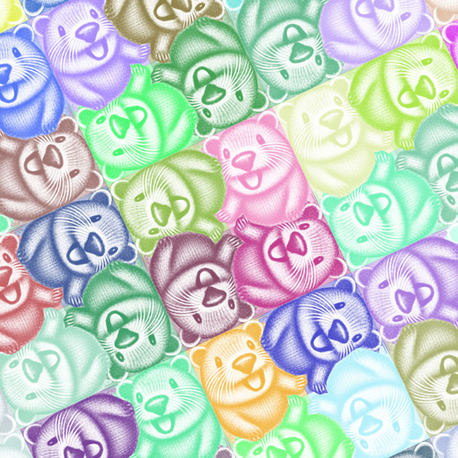}
        \caption{\theprompt{a beaver} \orbIhyb}
        \end{subfigure}
        \begin{subfigure}[b]{0.196\linewidth}
        \includegraphics[width=\linewidth]{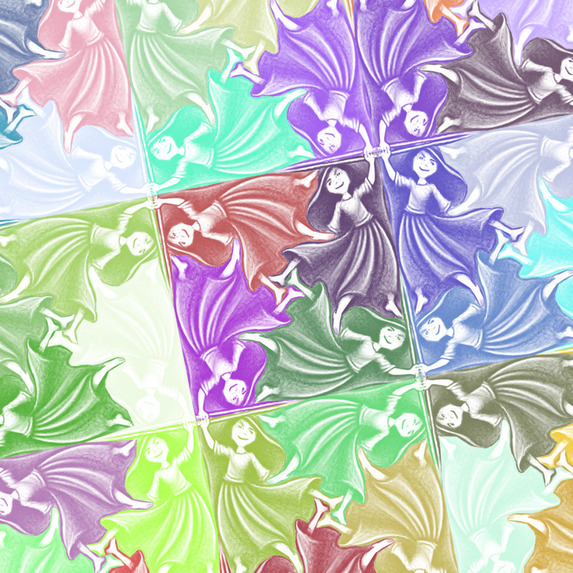}
        \caption{\theprompt{a witch} \orbIhyb}
        \end{subfigure}
        \begin{subfigure}[b]{0.196\linewidth}
        \includegraphics[width=\linewidth]{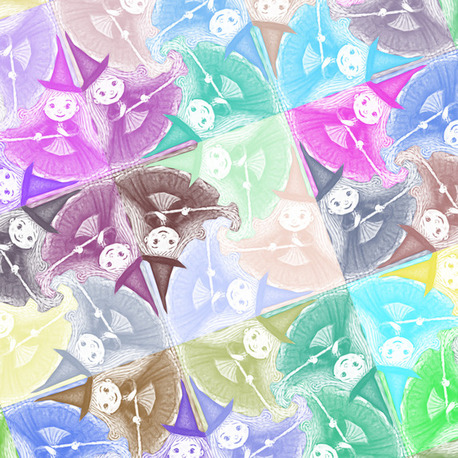}
        \caption{\theprompt{a witch} \orbIhyb}
        \end{subfigure}
        \begin{subfigure}[b]{0.196\linewidth}
        \includegraphics[width=\linewidth]{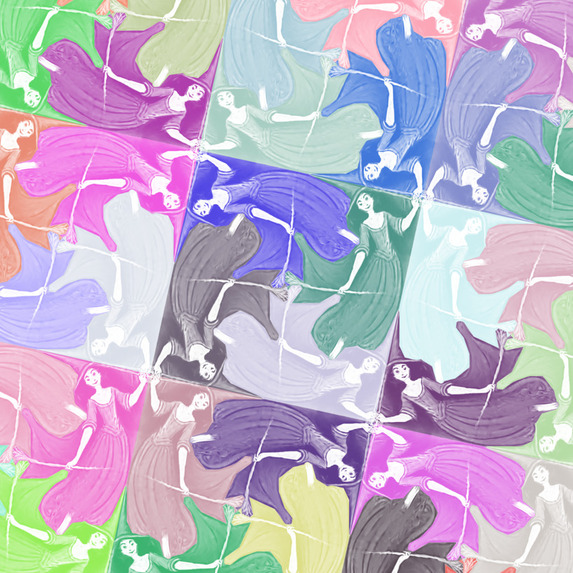}
        \caption{\theprompt{a witch} \orbIhyb}
        \end{subfigure}
        \begin{subfigure}[b]{0.196\linewidth}
        \includegraphics[width=\linewidth]{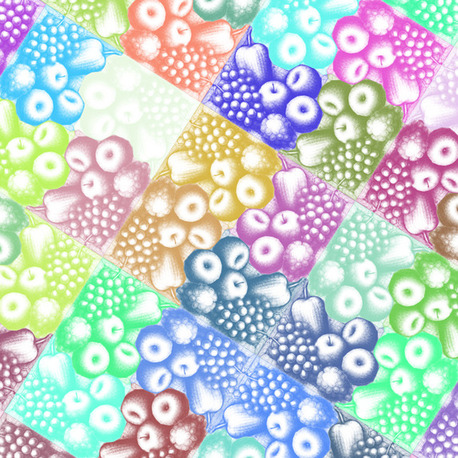}
        \caption{\theprompt{Fruits} \orbIhyb}
        \end{subfigure}
\\
        \begin{subfigure}[b]{0.196\linewidth}
        \includegraphics[width=\linewidth]{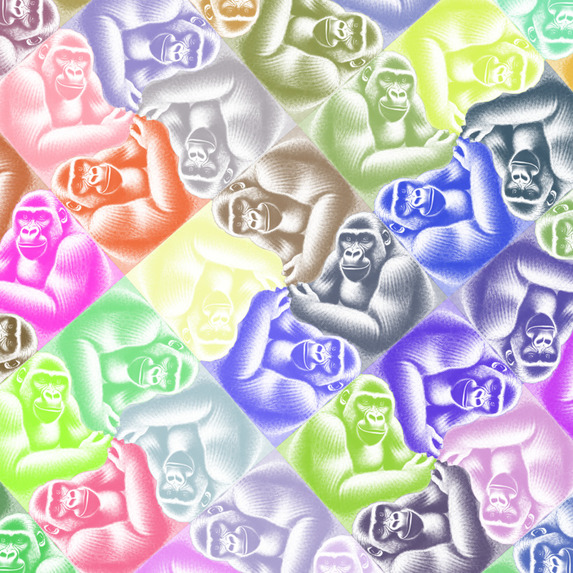}
        \caption{\theprompt{a funky gorilla} \orbIhyb}
        \end{subfigure}
        \begin{subfigure}[b]{0.196\linewidth}
        \includegraphics[width=\linewidth]{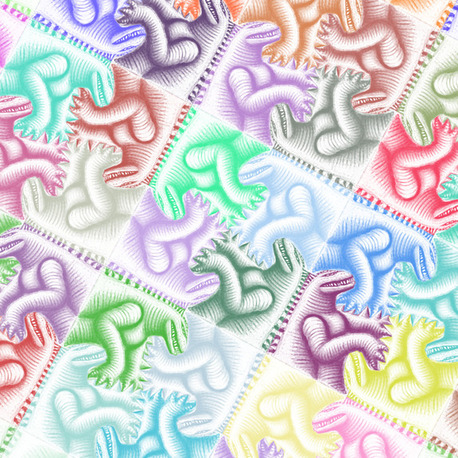}
        \caption{\theprompt{a Godzilla} \orbIhyb}
        \end{subfigure}
        \begin{subfigure}[b]{0.196\linewidth}
        \includegraphics[width=\linewidth]{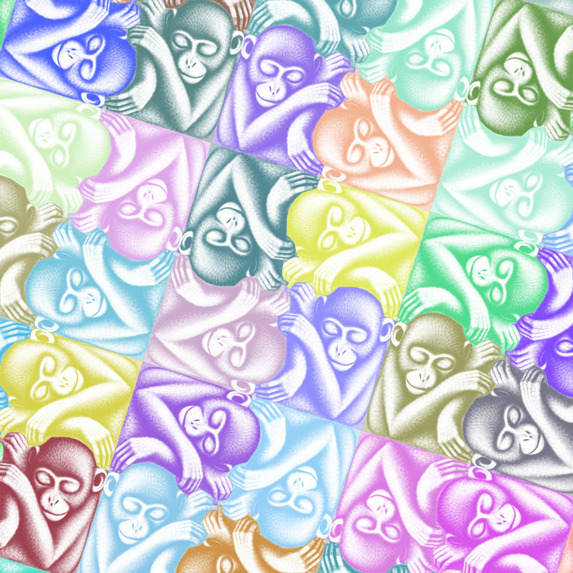}
        \caption{\theprompt{a monkey} \orbIhyb}
        \end{subfigure}
        \begin{subfigure}[b]{0.196\linewidth}
        \includegraphics[width=\linewidth]{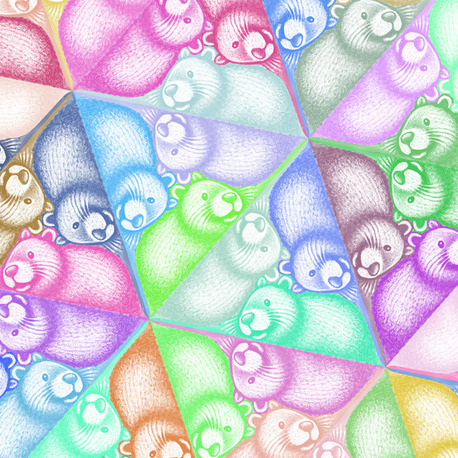}
        \caption{\theprompt{a beaver} \orbIIhyb}
        \end{subfigure}
        \begin{subfigure}[b]{0.196\linewidth}
        \includegraphics[width=\linewidth]{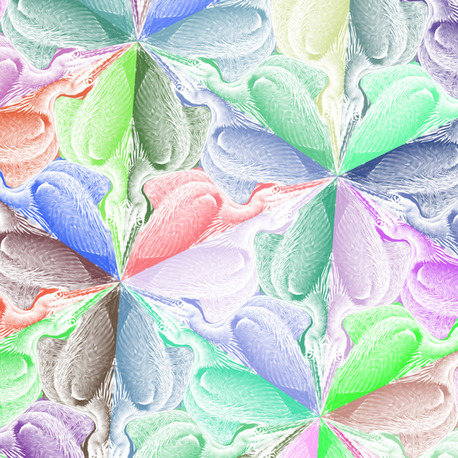}
        \caption{\theprompt{a heron} \orbIIhyb}
        \end{subfigure}
\\
\captionsetup{justification=justified}
    \caption{\textbf{Tiling menagerie part V.} Examples of tilings produced by our method for different prompts and symmetries. Our method produces appealing, plausible results which contain solely the desired object, and cover the plane without overlaps.}
    \label{fig:gallery_1}
\end{figure*}

\begin{figure*}
    \centering
    \captionsetup{justification=centering}
        \begin{subfigure}[b]{0.196\linewidth}
        \includegraphics[width=\linewidth]{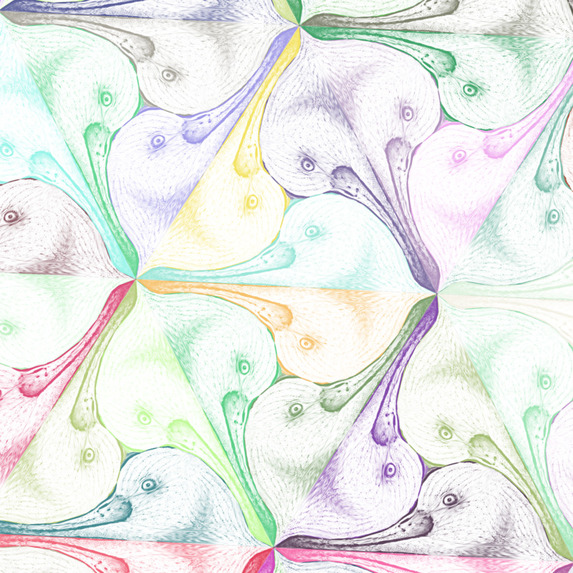}
        \caption{\theprompt{a flamingo} \orbIIhyb}
        \end{subfigure}
        \begin{subfigure}[b]{0.196\linewidth}
        \includegraphics[width=\linewidth]{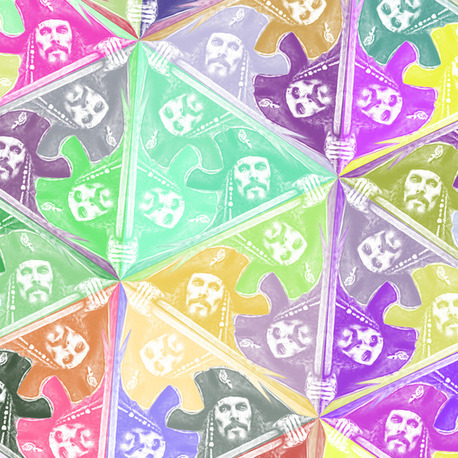}
        \caption{\theprompt{a pirate} \orbIIhyb}
        \end{subfigure}
        \begin{subfigure}[b]{0.196\linewidth}
        \includegraphics[width=\linewidth]{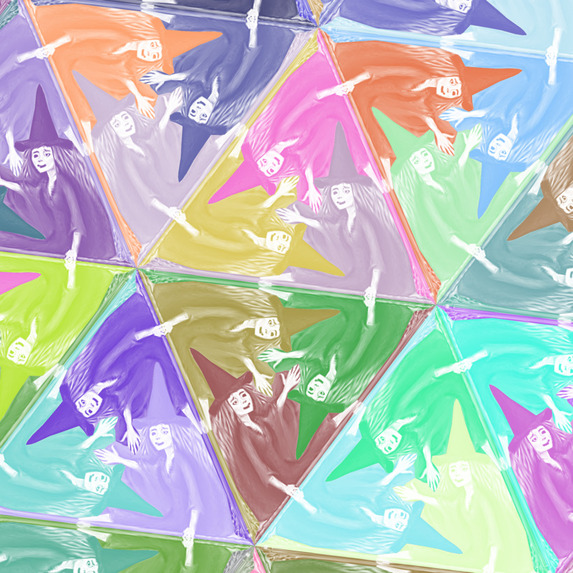}
        \caption{\theprompt{a witch} \orbIIhyb}
        \end{subfigure}
        \begin{subfigure}[b]{0.196\linewidth}
        \includegraphics[width=\linewidth]{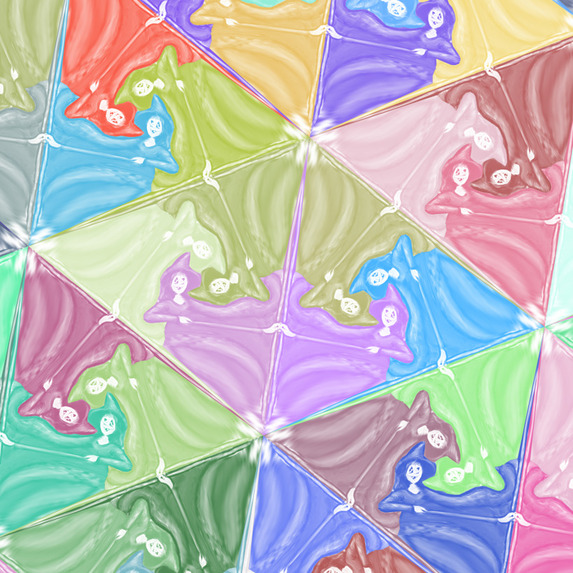}
        \caption{\theprompt{a witch} \orbIIhyb}
        \end{subfigure}
        \begin{subfigure}[b]{0.196\linewidth}
        \includegraphics[width=\linewidth]{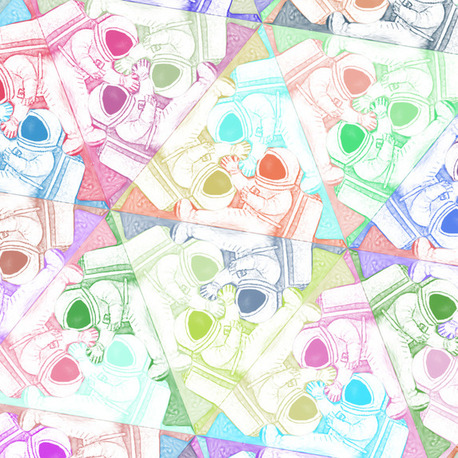}
        \caption{\theprompt{an astronaut} \orbIIhyb}
        \end{subfigure}
\\
        \begin{subfigure}[b]{0.196\linewidth}
        \includegraphics[width=\linewidth]{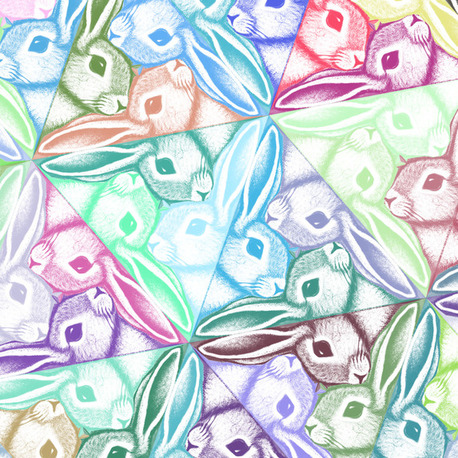}
        \caption{\theprompt{a hippie rabbit} \orbIIhyb}
        \end{subfigure}
        \begin{subfigure}[b]{0.196\linewidth}
        \includegraphics[width=\linewidth]{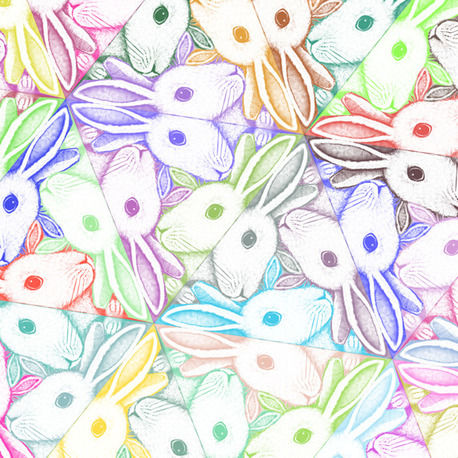}
        \caption{\theprompt{a hippie rabbit} \orbIIhyb}
        \end{subfigure}
        \begin{subfigure}[b]{0.196\linewidth}
        \includegraphics[width=\linewidth]{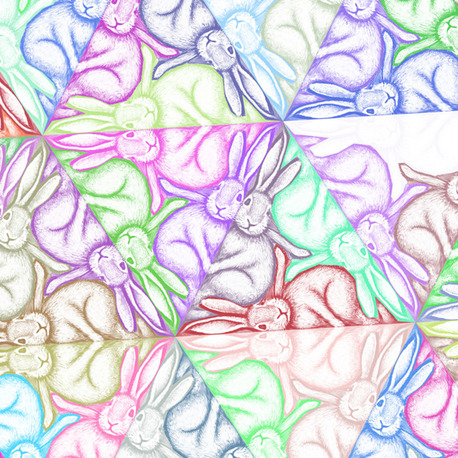}
        \caption{\theprompt{a hippie rabbit} \orbIIhyb}
        \end{subfigure}
        \begin{subfigure}[b]{0.196\linewidth}
        \includegraphics[width=\linewidth]{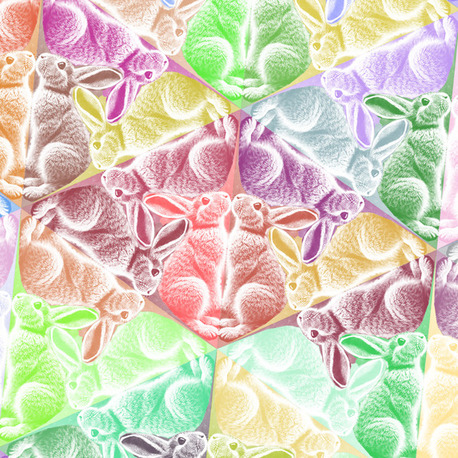}
        \caption{\theprompt{a hippie rabbit} \orbIIhyb}
        \end{subfigure}
        \begin{subfigure}[b]{0.196\linewidth}
        \includegraphics[width=\linewidth]{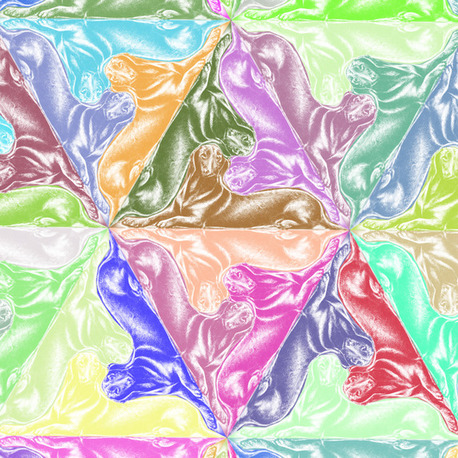}
        \caption{\theprompt{a Dachshund} \orbIIhyb}
        \end{subfigure}
\\
        \begin{subfigure}[b]{0.196\linewidth}
        \includegraphics[width=\linewidth]{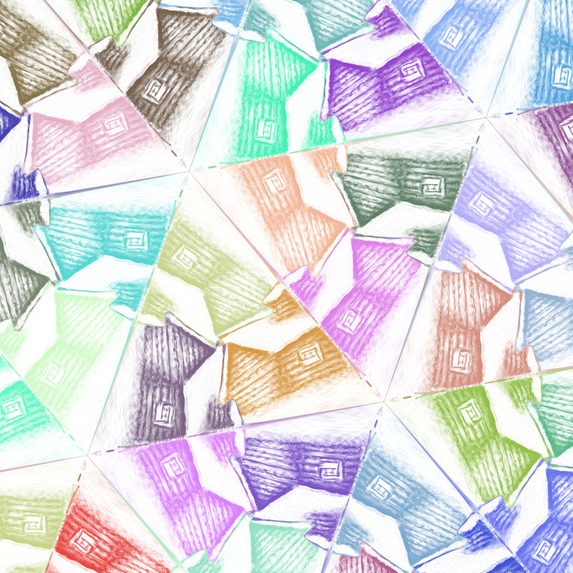}
        \caption{\theprompt{a snowy cabin} \orbIIhyb}
        \end{subfigure}
        \begin{subfigure}[b]{0.196\linewidth}
        \includegraphics[width=\linewidth]{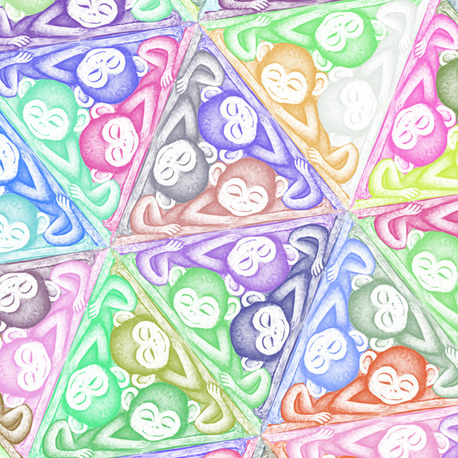}
        \caption{\theprompt{a monkey} \orbIIhyb}
        \end{subfigure}
        \begin{subfigure}[b]{0.196\linewidth}
        \includegraphics[width=\linewidth]{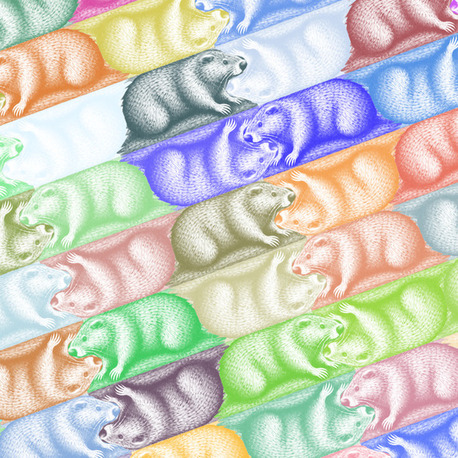}
        \caption{\theprompt{a beaver} \orbIVhyb}
        \end{subfigure}
        \begin{subfigure}[b]{0.196\linewidth}
        \includegraphics[width=\linewidth]{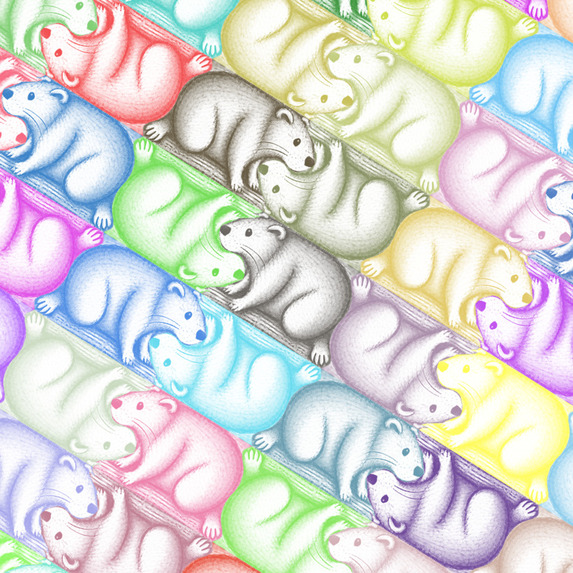}
        \caption{\theprompt{a beaver} \orbIVhyb}
        \end{subfigure}
        \begin{subfigure}[b]{0.196\linewidth}
        \includegraphics[width=\linewidth]{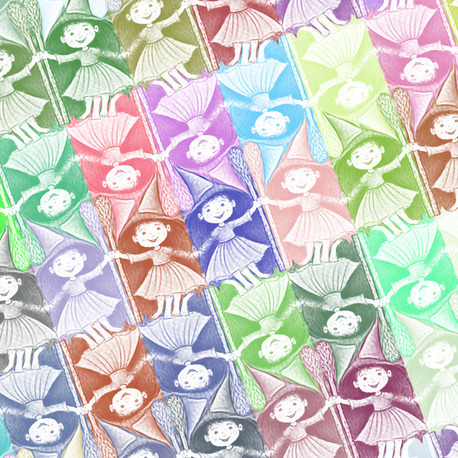}
        \caption{\theprompt{a witch} \orbIVhyb}
        \end{subfigure}
\\
        \begin{subfigure}[b]{0.196\linewidth}
        \includegraphics[width=\linewidth]{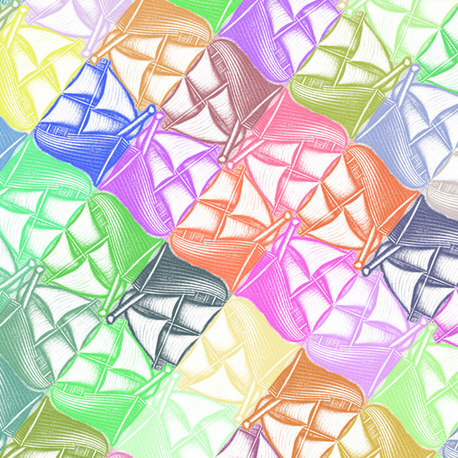}
        \caption{\theprompt{a pirate ship} \orbIVhyb}
        \end{subfigure}
        \begin{subfigure}[b]{0.196\linewidth}
        \includegraphics[width=\linewidth]{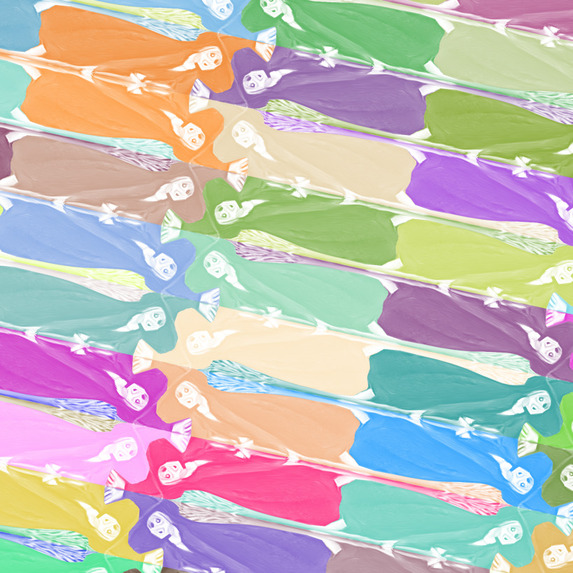}
        \caption{\theprompt{a witch} \orbIVhyb}
        \end{subfigure}
        \begin{subfigure}[b]{0.196\linewidth}
        \includegraphics[width=\linewidth]{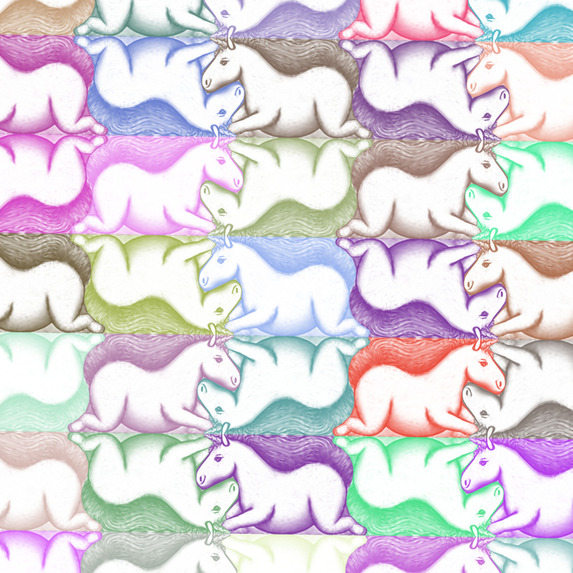}
        \caption{\theprompt{a fat unicorn} \orbIVhyb}
        \end{subfigure}
        \begin{subfigure}[b]{0.196\linewidth}
        \includegraphics[width=\linewidth]{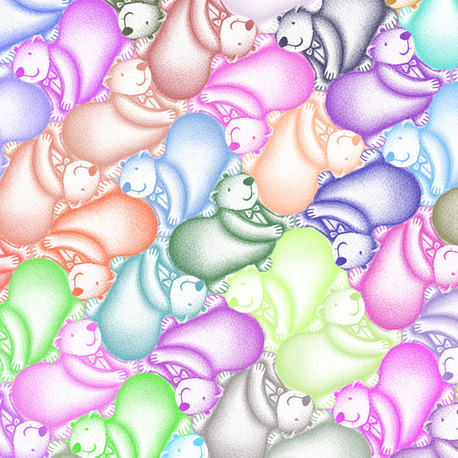}
        \caption{\theprompt{a beaver} \projective}
        \end{subfigure}
        \begin{subfigure}[b]{0.196\linewidth}
        \includegraphics[width=\linewidth]{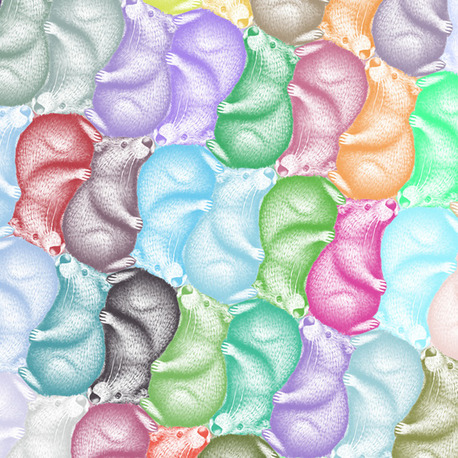}
        \caption{\theprompt{a beaver} \projective}
        \end{subfigure}
\\
        \begin{subfigure}[b]{0.196\linewidth}
        \includegraphics[width=\linewidth]{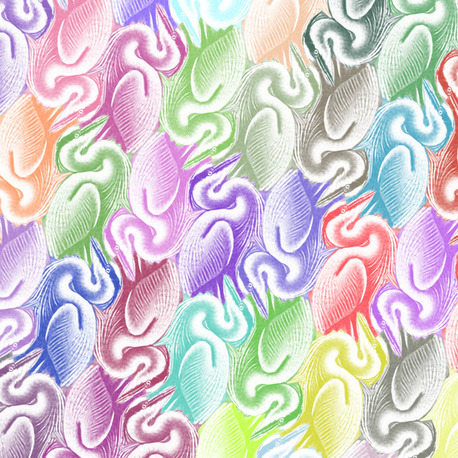}
        \caption{\theprompt{a heron} \projective}
        \end{subfigure}
        \begin{subfigure}[b]{0.196\linewidth}
        \includegraphics[width=\linewidth]{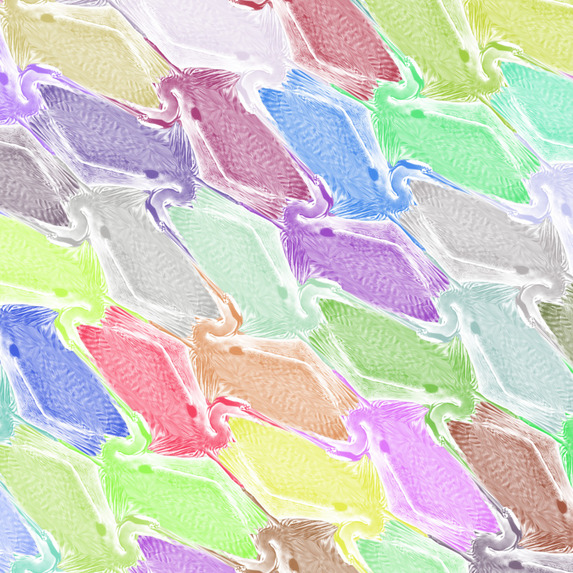}
        \caption{\theprompt{a heron} \projective}
        \end{subfigure}
        \begin{subfigure}[b]{0.196\linewidth}
        \includegraphics[width=\linewidth]{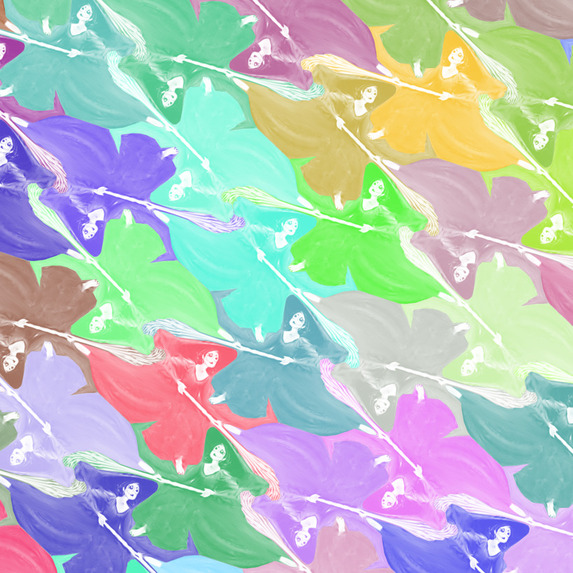}
        \caption{\theprompt{a witch} \projective}
        \end{subfigure}
        \begin{subfigure}[b]{0.196\linewidth}
        \includegraphics[width=\linewidth]{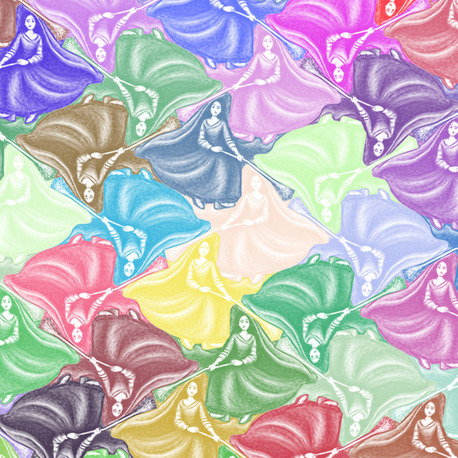}
        \caption{\theprompt{a witch} \projective}
        \end{subfigure}
        \begin{subfigure}[b]{0.196\linewidth}
        \includegraphics[width=\linewidth]{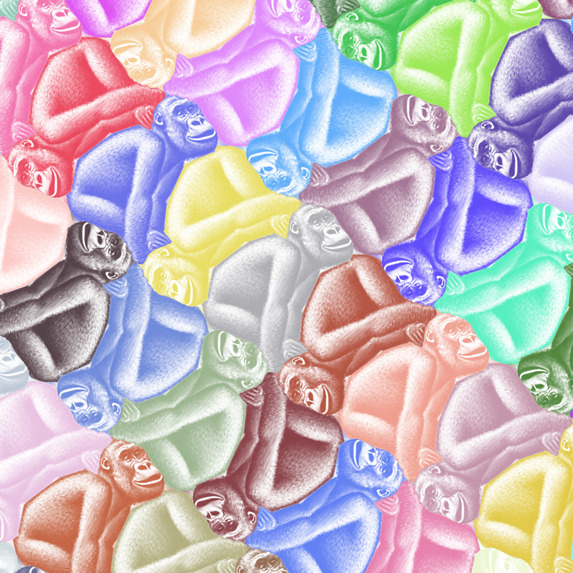}
        \caption{\theprompt{a funky gorilla} \projective}
        \end{subfigure}
\\
\captionsetup{justification=justified}
    \caption{\textbf{Tiling menagerie part VI.} Examples of tilings produced by our method for different prompts and symmetries. Our method produces appealing, plausible results which contain solely the desired object, and cover the plane without overlaps.}
    \label{fig:gallery_2}
\end{figure*}

\begin{figure*}
    \centering
    \captionsetup{justification=centering}
        \begin{subfigure}[b]{0.196\linewidth}
        \includegraphics[width=\linewidth]{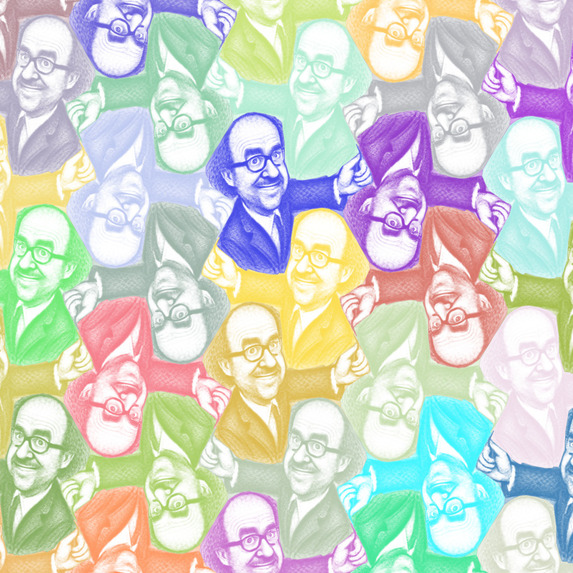}
        \caption{\theprompt{a math professor} \projective}
        \end{subfigure}
        \begin{subfigure}[b]{0.196\linewidth}
        \includegraphics[width=\linewidth]{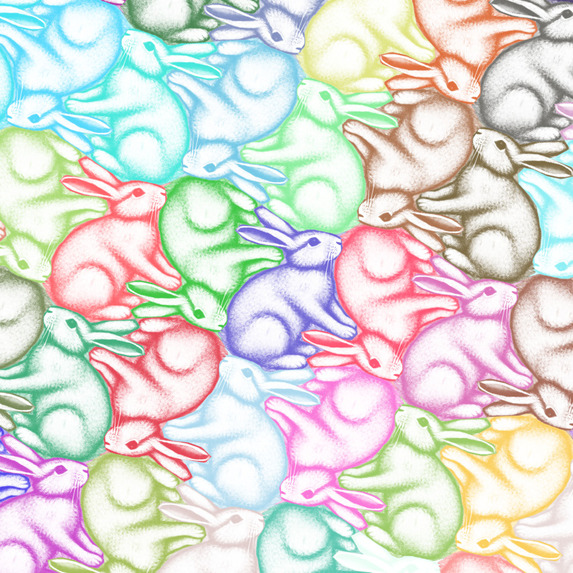}
        \caption{\theprompt{a hippie rabbit} \projective}
        \end{subfigure}
        \begin{subfigure}[b]{0.196\linewidth}
        \includegraphics[width=\linewidth]{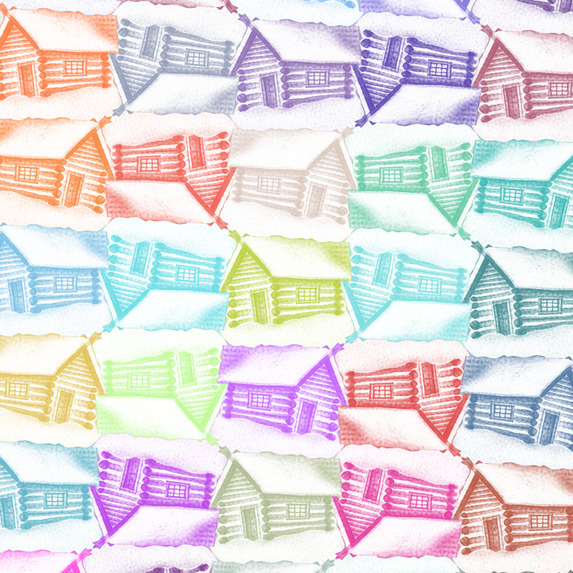}
        \caption{\theprompt{a snowy cabin} \projective}
        \end{subfigure}
        \begin{subfigure}[b]{0.196\linewidth}
        \includegraphics[width=\linewidth]{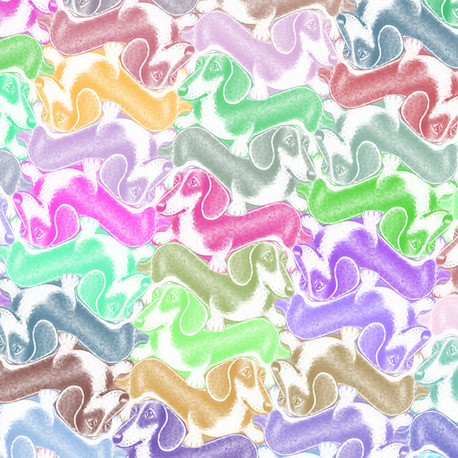}
        \caption{\theprompt{a Dachshund} \projective}
        \end{subfigure}
        \begin{subfigure}[b]{0.196\linewidth}
        \includegraphics[width=\linewidth]{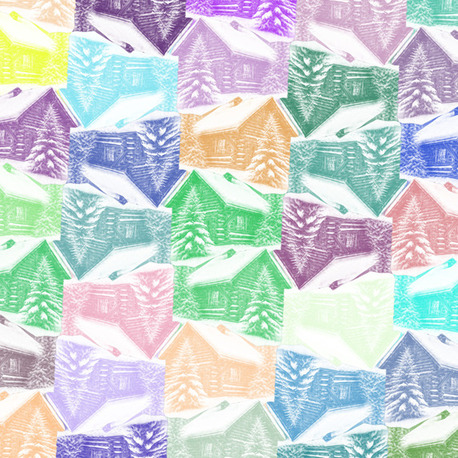}
        \caption{\theprompt{a snowy cabin} \projective}
        \end{subfigure}
\\
        \begin{subfigure}[b]{0.196\linewidth}
        \includegraphics[width=\linewidth]{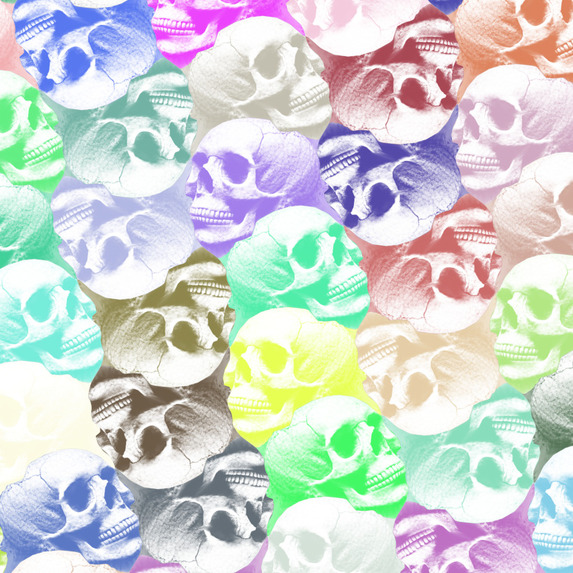}
        \caption{\theprompt{a skull} \projective}
        \end{subfigure}
        \begin{subfigure}[b]{0.196\linewidth}
        \includegraphics[width=\linewidth]{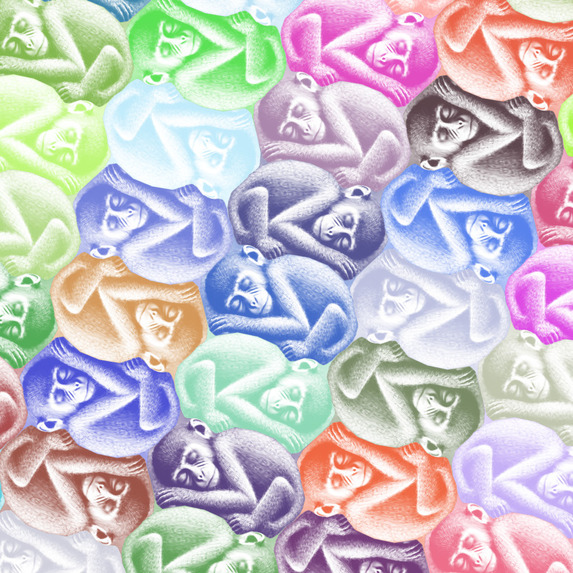}
        \caption{\theprompt{a monkey} \projective}
        \end{subfigure}
        \begin{subfigure}[b]{0.196\linewidth}
        \includegraphics[width=\linewidth]{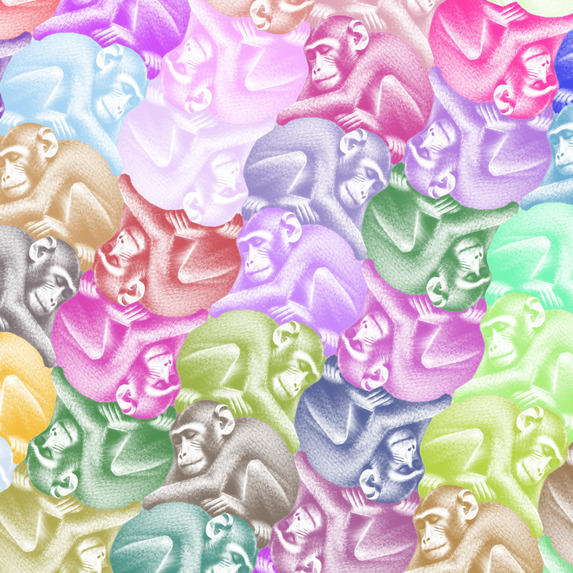}
        \caption{\theprompt{a monkey} \projective}
        \end{subfigure}
        \begin{subfigure}[b]{0.196\linewidth}
        \includegraphics[width=\linewidth]{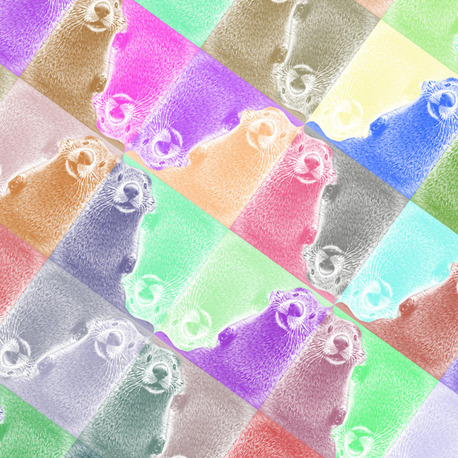}
        \caption{\theprompt{a beaver} \orbRhyb}
        \end{subfigure}
        \begin{subfigure}[b]{0.196\linewidth}
        \includegraphics[width=\linewidth]{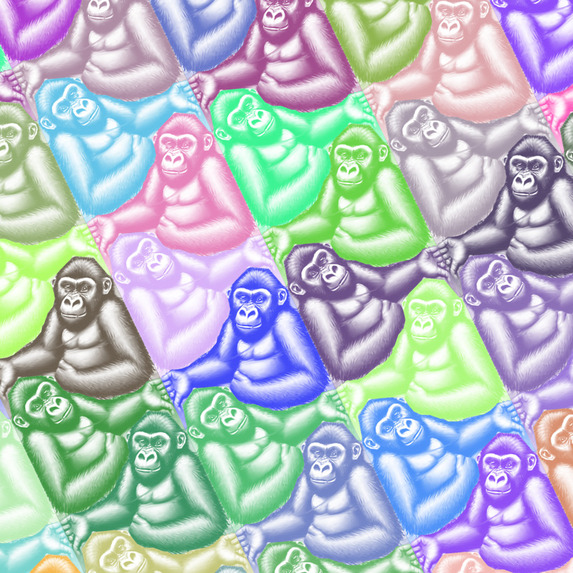}
        \caption{\theprompt{a funky gorilla} \mob}
        \end{subfigure}
\\
        \begin{subfigure}[b]{0.196\linewidth}
        \includegraphics[width=\linewidth]{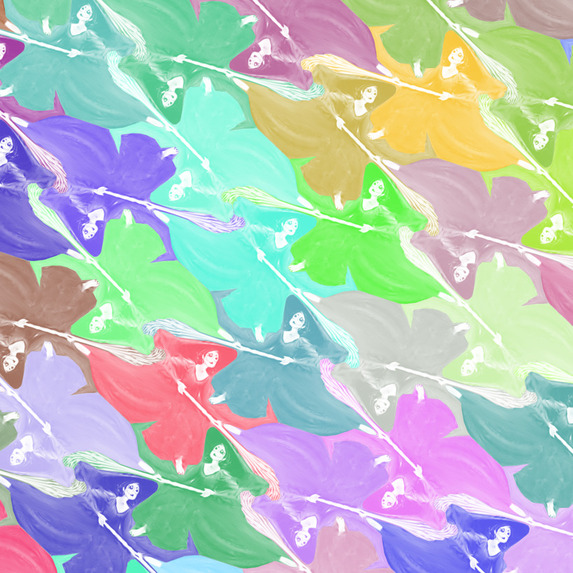}
        \caption{\theprompt{a witch} \projective}
        \end{subfigure}
        \begin{subfigure}[b]{0.196\linewidth}
        \includegraphics[width=\linewidth]{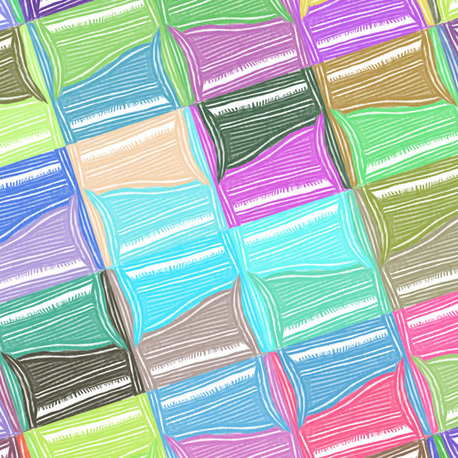}
        \caption{\theprompt{a piano} \orbRhyb}
        \end{subfigure}
        \begin{subfigure}[b]{0.196\linewidth}
        \includegraphics[width=\linewidth]{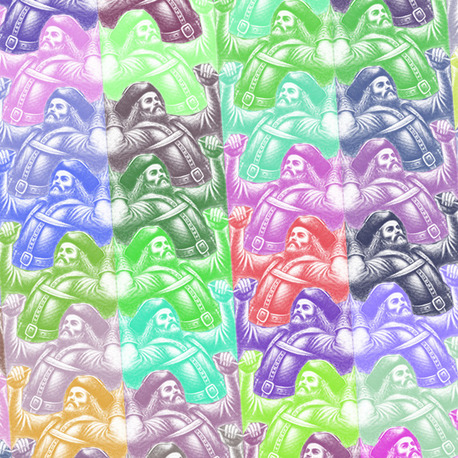}
        \caption{\theprompt{a pirate} \cylinder}
        \end{subfigure}
        \begin{subfigure}[b]{0.196\linewidth}
        \includegraphics[width=\linewidth]{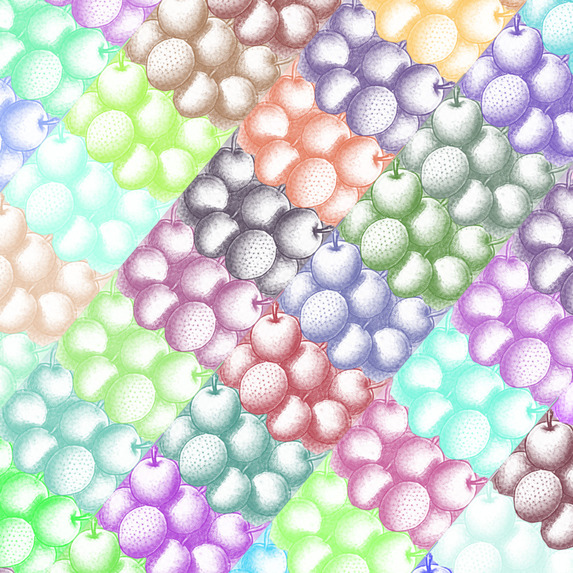}
        \caption{\theprompt{Fruits} \cylinder}
        \end{subfigure}
        \begin{subfigure}[b]{0.196\linewidth}
        \includegraphics[width=\linewidth]{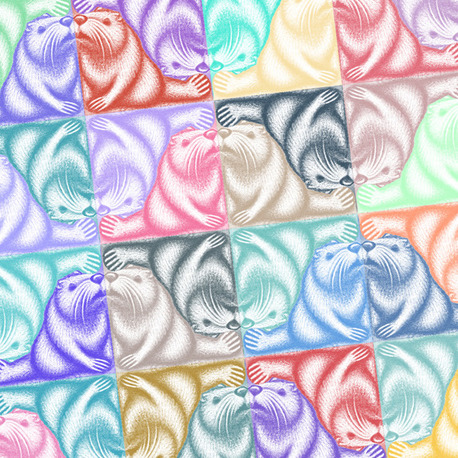}
        \caption{\theprompt{a beaver} \orbRhyb}
        \end{subfigure}
\\
        \begin{subfigure}[b]{0.196\linewidth}
        \includegraphics[width=\linewidth]{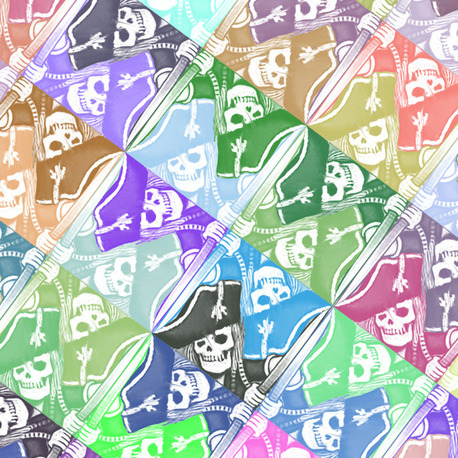}
        \caption{\theprompt{a pirate} \orbRhyb}
        \end{subfigure}
        \begin{subfigure}[b]{0.196\linewidth}
        \includegraphics[width=\linewidth]{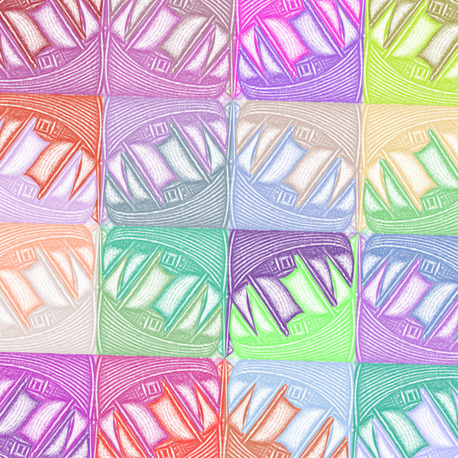}
        \caption{\theprompt{a pirate ship} \orbRhyb}
        \end{subfigure}
        \begin{subfigure}[b]{0.196\linewidth}
        \includegraphics[width=\linewidth]{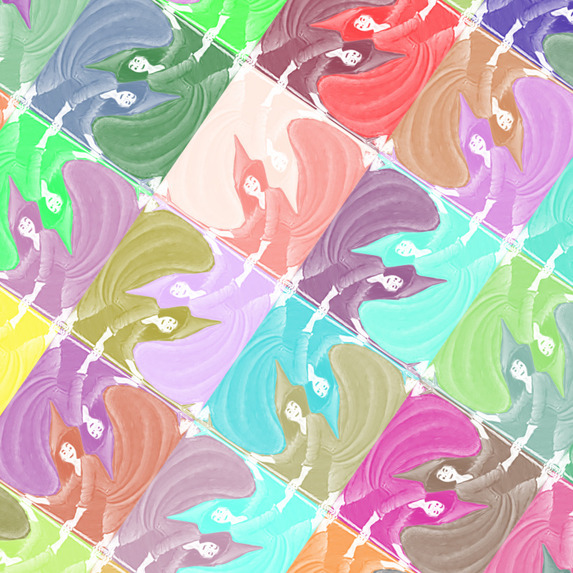}
        \caption{\theprompt{a witch} \orbRhyb}
        \end{subfigure}
        \begin{subfigure}[b]{0.196\linewidth}
        \includegraphics[width=\linewidth]{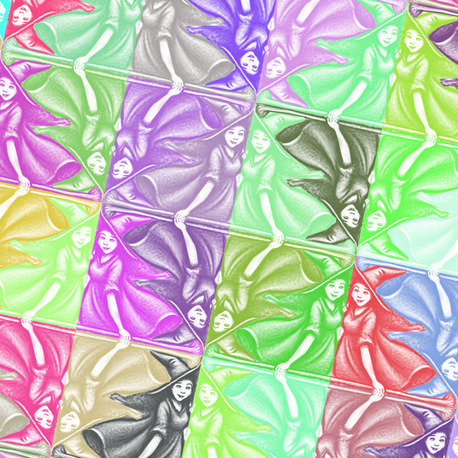}
        \caption{\theprompt{a witch} \orbRhyb}
        \end{subfigure}
        \begin{subfigure}[b]{0.196\linewidth}
        \includegraphics[width=\linewidth]{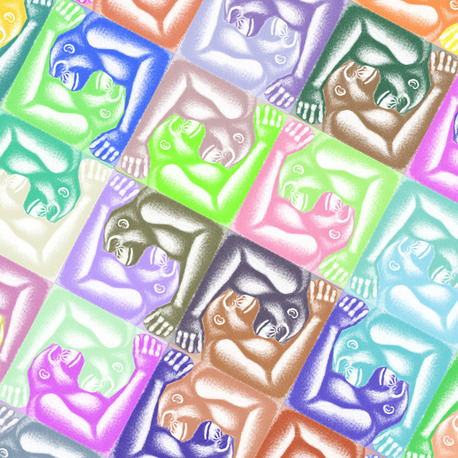}
        \caption{\theprompt{a funky gorilla} \orbRhyb}
        \end{subfigure}
\\
        \begin{subfigure}[b]{0.196\linewidth}
        \includegraphics[width=\linewidth]{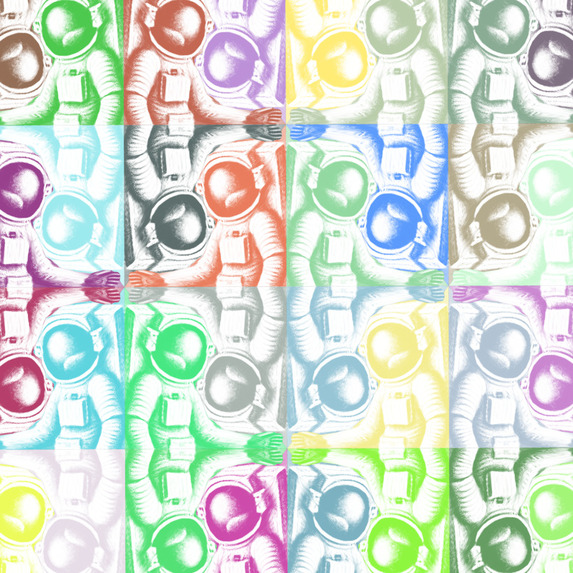}
        \caption{\theprompt{an astronaut} \orbRhyb}
        \end{subfigure}
        \begin{subfigure}[b]{0.196\linewidth}
        \includegraphics[width=\linewidth]{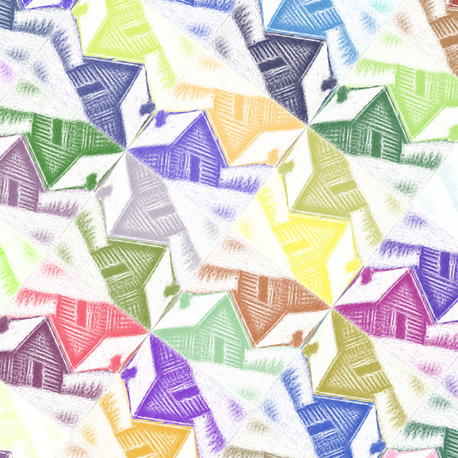}
        \caption{\theprompt{a snowy cabin} \orbRhyb}
        \end{subfigure}
        \begin{subfigure}[b]{0.196\linewidth}
        \includegraphics[width=\linewidth]{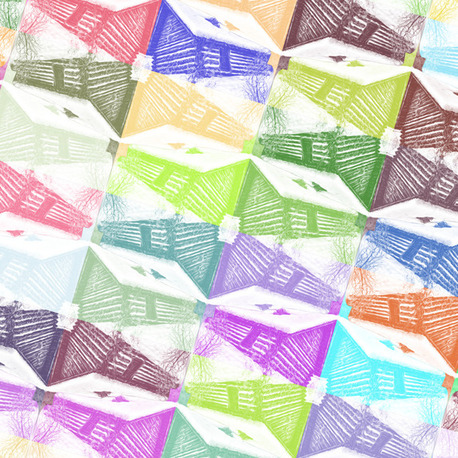}
        \caption{\theprompt{a snowy cabin} \orbRhyb}
        \end{subfigure}
        \begin{subfigure}[b]{0.196\linewidth}
        \includegraphics[width=\linewidth]{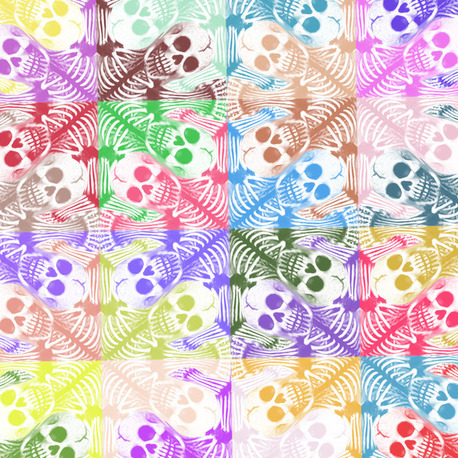}
        \caption{\theprompt{a Skeletton} \orbRhyb}
        \end{subfigure}
        \begin{subfigure}[b]{0.196\linewidth}
        \includegraphics[width=\linewidth]{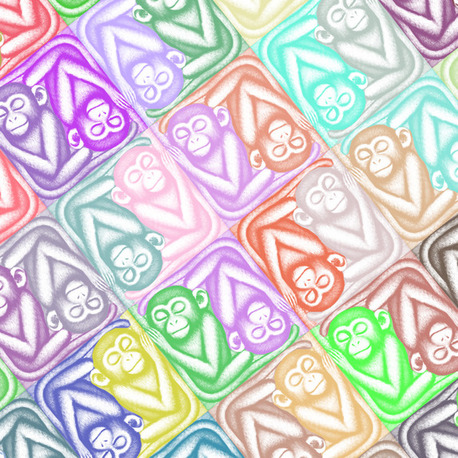}
        \caption{\theprompt{a monkey} \orbRhyb}
        \end{subfigure}
\\
\captionsetup{justification=justified}
    \caption{\textbf{Tiling menagerie part VII.} Examples of tilings produced by our method for different prompts and symmetries. Our method produces appealing, plausible results which contain solely the desired object, and cover the plane without overlaps.}
    \label{fig:gallery_3}
\end{figure*}



\end{document}